\documentclass[journal]{IEEEtran}

\ifCLASSINFOpdf
\usepackage[pdftex]{graphicx}
\else
\fi
\usepackage{url}            
\usepackage{booktabs}       
\usepackage{amsfonts}       

\usepackage{amsmath}
\usepackage{amsthm}
\usepackage{amssymb}

\usepackage{dsfont}

\usepackage{cleveref}
\crefformat{footnote}{#2\footnotemark[#1]#3}

\usepackage{algorithm}
\usepackage[noend]{algorithmic}
\usepackage[shortlabels]{enumitem}

\usepackage{xcolor}

\newcommand{\Sdom}{\mathcal{S}}
\newcommand{\Adom}{\mathcal{A}}
\newcommand{\Zprior}{{z}_0}
\newcommand{\Sprior}{{s}_0}
\newcommand{\zHist}[1]{\mathbf{z}_{#1}}
\newcommand{\sHist}[1]{\mathbf{s}_{#1}}

\newcommand{\sNew}[1]{{s}_{#1}}
\newcommand{\zNew}[1]{{z}_{#1}}
\newcommand{\xNew}[1]{\mathbf{x}_{#1}}

\newcommand{\argmax}{\operatornamewithlimits{argmax}}
\newcommand{\argmin}{\operatornamewithlimits{argmin}}

\newtheorem{definition}{Definition}
\newtheorem{theorem}{Theorem}
\newtheorem{lemma}{Lemma}

\newtheorem{corollary}{Corollary}

\theoremstyle{remark}
\newtheorem{remark}{Remark}

\begin{document}
	%
	\title{Nonmyopic Gaussian Process Optimization with Macro-Actions}
	%
	%
	%

	
	\author{Dmitrii Kharkovskii, Chun Kai Ling, Bryan Kian Hsiang Low
		\thanks{D. Kharkovskii and B. K. H. Low are with the Department of Computer Science,
		National University of Singapore, Republic of Singapore (email:
		dmitrkha@comp.nus.edu.sg, lowkh@comp.nus.edu.sg).}
	\thanks{C. K. Ling is with the Department of Computer Science, Carnegie Mellon University, Pittsburgh, PA $15213$, USA
	(e-mail: chunkail@cs.cmu.edu).}
	}

	\maketitle
	
	\begin{abstract}
		This paper presents a multi-staged approach to nonmyopic adaptive \emph{Gaussian process optimization} (GPO)  for \emph{Bayesian optimization} (BO) of unknown, highly complex objective functions 
		that, in contrast to existing nonmyopic adaptive BO algorithms, exploits the notion of macro-actions for scaling up to a further lookahead to match up to a larger available budget.
		To achieve this, we generalize GP upper confidence bound to a new acquisition function defined w.r.t. a nonmyopic adaptive macro-action policy, which is intractable to be optimized exactly due to an uncountable set of candidate outputs.
		The contribution of our work here is thus to derive a nonmyopic adaptive \emph{$\epsilon$-Bayes-optimal macro-action GPO} ($\epsilon$-Macro-GPO) policy. 
		To perform nonmyopic adaptive BO in real time, we then propose an asymptotically optimal anytime variant of our $\epsilon$-Macro-GPO policy with a performance guarantee.
		We empirically evaluate the performance of our $\epsilon$-Macro-GPO policy and its anytime variant in BO with synthetic and real-world datasets.
	\end{abstract}
	

	%
	\IEEEpeerreviewmaketitle
	
	\section{Introduction}
	\label{sec:intro}
	Recent advances in \emph{Bayesian optimization} (BO) have delivered a promising suite of tools for optimizing an unknown (possibly noisy, non-convex, with no closed-form expression/derivative) objective function with a finite budget of function evaluations, as demonstrated in a wide range of applications like automated machine learning, robotics, sensor networks, environmental monitoring, among others \cite{shahriari16}. 
	Conventionally, a BO algorithm relies on some choice of acquisition function (e.g., 
	improvement-based such as 
	probability of improvement or \emph{expected improvement} (EI) over currently found maximum, information-based  \cite{hennig12,lobato14,villemonteix09}, or \emph{upper confidence bound} (UCB) \cite{srinivas10}) as a heuristic to guide its search for the global maximum. 
	To do this, the BO algorithm exploits the chosen acquisition function
	to repeatedly select an input for evaluating the unknown objective function that trades off between observing a likely maximum based on a GP belief of the unknown objective function (exploitation) vs.~improving the GP belief (exploration) until the budget is expended.
	
	Unfortunately, such a conventional BO algorithm is greedy/myopic and hence performs suboptimally with respect to the given finite budget\footnote{\label{bravo}Acquisition functions like EI~\cite{bull11,vazquez10} and UCB~\cite{srinivas10} 
		offer theoretical guarantees for the convergence rate of their BO algorithms  
		(i.e., in the limit) via regret bounds. 
		In practice, since the budget is limited, such bounds are suboptimal as they cannot be specified to be arbitrarily small.
	}. 
	To be nonmyopic, its policy to select the next input has to additionally account for its subsequent selections of inputs for evaluating the unknown objective function\footnote{Fig.~\ref{fig:illustration} shows how a nonmyopic BO algorithm can outperform a myopic one.}. 
	Perhaps surprisingly, this can be partially achieved by batch BO algorithms capable of \emph{jointly}\footnote{\label{boo}In contrast, a \emph{greedy} batch BO algorithm~\cite{azimi10,contal13,desautels14,gonzalez16b}  selects the inputs of a batch one at a time myopically.} optimizing a batch of inputs~\cite{chevalier13,daxberger17,shah15,wu2016} because 
	their selection of each input has to account for that of all other inputs of the batch\footnote{Batch BO is traditionally considered when resources are available to evaluate the objective function in parallel.
		We suggest a further possibility of using batch BO for non-myopic selections of inputs here.}.
	However, since the batch size is typically set to be much smaller than the given budget, they have to repeatedly select the next batch greedily.
	Furthermore, unlike the conventional BO algorithm described above, their selection of each input is independent of the outputs observed from evaluating the objective function at the other selected inputs of the batch, thus sacrificing some degree of adaptivity.
	Hence, they also perform suboptimally with respect to the given budget.

	Some nonmyopic adaptive BO algorithms~\cite{lam17,lam16,ling16,marchant14,Osborne09} have been developed to combine the best of both worlds. But, they have been empirically demonstrated to be effective and tractable for at most a lookahead of $5$  observations which is usually much less than the size of the available budget in practice, thus causing them to behave myopically in this case.
	To increase the lookahead,
	the work of~\cite{gonzalez16a} has proposed a two-staged approach that utilizes a \emph{greedy} batch BO algorithm\cref{boo} in its second stage to efficiently but myopically optimize all but the first input afforded by the budget.
	Note that the above works on nonmyopic adaptive BO do not provide theoretical performance guarantees except for that of~\cite{ling16}.
	The challenge therefore remains in devising a multi-staged approach to nonmyopic adaptive BO that can empirically scale well to a further lookahead (and hence match up to a larger budget) and still be amenable to a theoretical analysis of its performance,
	which is the focus of our work here.
	
	\begin{figure*}
		\centering
		\begin{tabular}{cccc}
				\hspace{-2.5mm}\includegraphics[width=0.245 \textwidth]{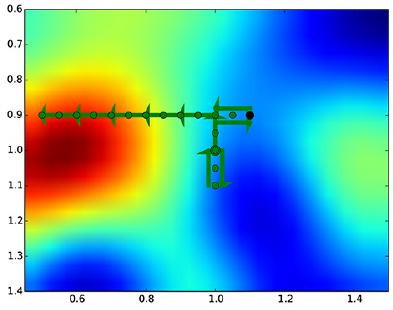} & \hspace{-4mm}
				\includegraphics[width=0.245 \textwidth]{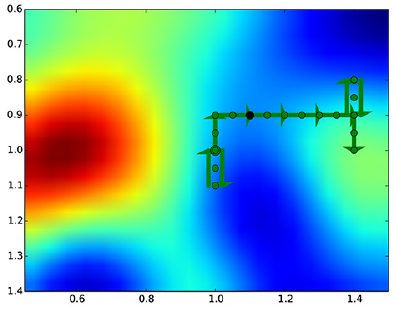} &
				\hspace{-4mm} \includegraphics[width=0.245 \textwidth]{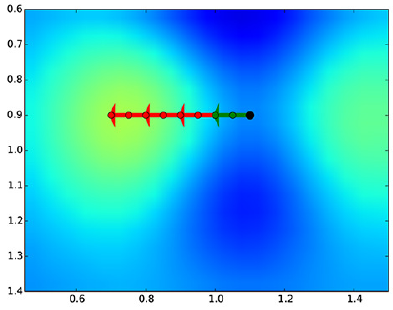} &
				\hspace{-3mm}\includegraphics[width=0.245 \textwidth]{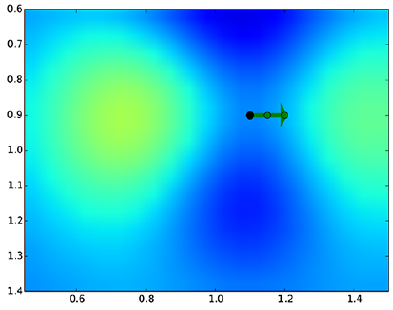} \\
				\hspace{-2.5mm}{\small (a)}  &\hspace{-4mm}{\small (b)}  & \hspace{-4mm} {\small (c)} 
				& \hspace{-3mm} {\small (d)}  
			\end{tabular}
			\caption{Illustrating the behaviors of (a) our nonmyopic $\epsilon$-Macro-GPO policy with a lookahead of $8$ observations ($H=4$, $N=1$) vs.~(b) greedy/myopic DB-GP-UCB~\cite{daxberger17} with macro-action length $\kappa=2$ and budget of $20$ observations 
				in controlling an AUV to gather observations for finding a hotspot (i.e., global maximum) in a simulated plankton density phenomenon (Section~\ref{expt}). 
				Prior observations are at the AUV's initial starting input location (blue circle) and buoy's location $(0,0)$ (not shown here).
				Up till $t=5$, both $\epsilon$-Macro-GPO and DB-GP-UCB produce the same trajectories to reach the input location denoted by a black circle. Figs.~c \& d plot maps of GP posterior mean~\eqref{gp-posteriors} over the phenomenon at stage $t=5$. 
				(c) Since $\epsilon$-Macro-GPO is able to look ahead and plan its macro-actions in the later planning stages 
				(red arrows) to reach the region containing the global maximum, it moves the UAV left towards the global maximum. (d) On the other hand, DB-GP-UCB moves the UAV right towards the local maximum. So, by utilizing lookahead, our nonmyopic $\epsilon$-Macro-GPO policy can outperform the myopic DB-GP-UCB.}
		\label{fig:illustration}
	\end{figure*}

	To address this challenge, we exploit the notion of macro-actions (i.e., each denoting a sequence of primitive actions executed in full without considering any observation taken after performing each primitive action in the sequence)
	inherent to the structure
	of several real-world task environments/applications such as environmental sensing and monitoring, mobile sensor networks, and robotics. Some examples are given below and described in detail in Section~\ref{expt}: 
	\begin{enumerate}[(a)] 
		\item  In monitoring of algal bloom in the coastal ocean, 
		an \emph{autonomous underwater vehicle} (AUV) is deployed on board a research vessel in search for a hotspot of peak phytoplankton abundance and tasked to take dives from the vessel to gather ``Gulper'' water samples for on-deck testing that can be cast as macro-actions~\cite{pennington16};
		
		\item In servicing the mobility demands within an urban city, an autonomous robotic vehicle in a mobility-on-demand system
		cruises along different road trajectories abstracted as macro-actions to find a hotspot of highest mobility demand to pick up a user~\cite{TASE15};
		
		\item In monitoring of the indoor environmental quality of an office environment~\cite{choi12}, a mobile robot mounted with a weather board is tasked to find a hotspot of peak temperature by exploring different stretches of corridors that can be naturally abstracted into macro-actions. This setting is visually illustrated in  Fig.~\ref{fig:robot-example};
		
		\item In monitoring of algal bloom in the coastal ocean, an underwater glider is tasked to find a hotspot of peak chlorophyll fluorescence by optimizing its search trajectory tractably over simple ellipses of varying sizes~\cite{leonard07} that constitute different macro-actions.
	\end{enumerate}

	Macro-actions have in fact been well-studied and used by the planning community
	to scale up algorithms for planning under uncertainty to a further lookahead~\cite{he10,he11,lim11}, which is realized from a much reduced space of possible sequences of primitive actions (i.e., macro-actions) induced by the structure of the task environment/application.Macro-actions are also studied in reinforcement learning community but named  as options instead~\cite{barto03,konidaris07,stolle02}.
	
	The use of macro-actions in the context of nonmyopic adaptive BO 
	poses an interesting research question:
	\emph{How can an acquisition function be defined with respect to a nonmyopic adaptive macro-action\footnote{In BO, each macro-action denotes a sequence of inputs for evaluating the unknown objective function.} policy and optimized tractably to yield such a policy with a provable performance guarantee for a given finite budget?}

	The main technical difficulty in answering this question stems from the need to account for the correlation of outputs to be observed from evaluating the unknown objective function at inputs found within a macro-action and between different macro-actions (Section~\ref{main-section}). Such a correlation structure is the chief ingredient to be exploited for selecting informative observations to find the global maximum.

This paper presents a principled multi-staged Bayesian sequential decision problem framework for nonmyopic adaptive \emph{GP optimization} (GPO) (Section~\ref{main-section}) that, in particular, exploits macro-actions inherent to the structure of several real-world task environments/applications for scaling up to a further lookahead (as compared to the existing nonmyopic adaptive BO algorithms discussed above~\cite{lam16,lam17,ling16,marchant14,Osborne09}) 
to match up to a larger available budget.
To achieve this, we first generalize GP-UCB~\cite{srinivas10} 
to a new acquisition function defined with respect to a nonmyopic adaptive macro-action policy, which,  unfortunately, is intractable to be optimized exactly due to an uncountable set of candidate outputs.
The key novel contribution of our work here is to show that it is in fact possible to solve for a nonmyopic adaptive \emph{$\epsilon$-Bayes-optimal macro-action GPO} ($\epsilon$-Macro-GPO) policy given an arbitrarily user-specified loss bound $\epsilon$ 
via stochastic sampling in each planning stage which requires only a polynomial number of samples in the length of macro-actions\footnote{\label{ling}In contrast, though the nonmyopic adaptive BO algorithm of~\cite{ling16} based on deterministic sampling can be naively generalized to exploit macro-actions, it requires an exponential number of samples per planning stage, as detailed in Remark~\ref{ref:remark-chun}.}.
To perform nonmyopic adaptive BO in real time, we then propose an asymptotically optimal anytime variant of our $\epsilon$-Macro-GPO policy with a performance guarantee. We empirically evaluate the performance of our nonmyopic adaptive 
$\epsilon$-Macro-GPO policy and its anytime variant in BO with synthetic and real-world datasets (Section~\ref{expt}). 

\section{Modeling Spatially Varying Phenomena with Gaussian Processes} 
\label{gppfram}
To simplify exposition of our work here, we will assume the task environment to be a spatially varying phenomenon (e.g., 
indoor environmental quality of an office environment, 
plankton bloom in the ocean, mobility demand within an urban city, as described in Section~\ref{sec:intro}). 
A mobile sensing agent utilizes our proposed nonmyopic adaptive $\epsilon$-Macro-GPO policy or its anytime variant to select and gather observations from the task environment for finding the global maximum.
\subsection{Notations and Preliminaries}
Let $\Sdom$ be the domain of a spatially varying phenomenon corresponding to a set of input locations.
In every stage $t > 0$, the agent executes one of the available macro-actions of length $\kappa$ 
at its current input location 
by deterministically moving through a sequence of $\kappa$ input locations,
denoted by a vector  
$\sNew{t} \in\Adom(\sNew{t-1})$, and observes the corresponding output measurements 
$\zNew{t} \in \mathbb{R}^\kappa$,  
where   
$\Adom(\sNew{t-1})\subseteq \Sdom^\kappa$
denotes a finite set of available macro-actions 
at the agent's current input location. Note that $\Adom(\sNew{t-1})$ depends on the agent's current input location which corresponds to the last component of macro-action $\sNew{t-1}$ executed in the previous stage $t-1$. These notations are visually illustrated in Fig.~\ref{fig:robot-example} and its caption b.  
The state of the agent at its initial starting input location 
is represented by prior observations/data $d_0\triangleq\langle \Sprior , \Zprior\rangle$ available before planning where 
$\Sprior$ and $\Zprior$ denote, respectively, vectors comprising input locations visited and corresponding output measurements observed by the agent prior to planning. 
The agent's initial starting input location
is the last component of $\Sprior$.
In stage $t>0$, the state of the agent 
is represented by observations/data $d_t \triangleq \langle \sHist{t}, \zHist{t} \rangle$ where 
$\sHist{t}\triangleq \Sprior \oplus\ldots\oplus \sNew{t}$ and
$\zHist{t}\triangleq \Zprior \oplus\ldots\oplus \zNew{t}$
denote, respectively, vectors comprising input locations visited and corresponding output measurements observed by the agent up till 
stage $t$
and `$\oplus$' denotes vector concatenation. 
\begin{figure}[t]
	\centering
	\includegraphics[width=0.45\textwidth]{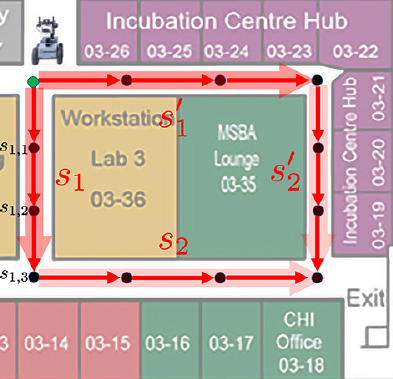} 
	\caption{Example of monitoring indoor environmental quality of an office environment~\cite{choi12}: 
		(a) A mobile robot mounted with a weather board is tasked to find a hotspot of peak temperature by exploring different stretches of corridors that can be naturally abstracted into macro-actions. 
		(b) In stage $t=1$, the robot is at its initial starting input location (green dot). It can decide to execute macro-action \textcolor{red}{$\sNew{1}$} (translucent red arrow), 
		which is a sequence of $\kappa=3$ primitive actions (opaque red arrows) moving it through a sequence of $\kappa=3$ input locations (black dots) to arrive at input location $s_{1, 3}$. 
		So, $\textcolor{red}{\sNew{1}}\triangleq (s_{1,1},s_{1,2}, s_{1,3})$.
		(c) To derive a myopic Macro-GPO or $\epsilon$-Macro-GPO policy with $H=1$, the last stages of Bellman equations in \eqref{eq:OptimalValFunDef}-\eqref{eq_4_8} require macro-actions \textcolor{red}{$\sNew{1}$} and \textcolor{red}{$\sNew{1}'$} as inputs. To derive a nonmyopic one with $H=2$, they require macro-action sequences $\textcolor{red}{\sNew{1}}\oplus \textcolor{red}{\sNew{2}}$ and $\textcolor{red}{\sNew{1}'}\oplus \textcolor{red}{\sNew{2}'}$ as inputs instead.}
	\label{fig:robot-example} 
\end{figure}
\subsection{Gaussian Process (GP)}
\label{subsection-gp}
The spatially varying phenomenon is modeled as a realization of a GP:
Each input location $s \in \Sdom$ is associated with an output measurement $y_s$.
Let $y_{\Sdom} \triangleq \{y_{s}\}_{s \in \Sdom}$ denote a GP, that is, every finite subset of $y_{\Sdom}$ has a multivariate Gaussian distribution.
Then, the GP is fully specified by its \emph{prior} mean $\mu_s \triangleq \mathbb{E}[y_s]$ (we assume w.l.o.g.~that $\mu_s=0$  for all $s \in \Sdom$) and covariance 
$\sigma_{ss'} \triangleq \text{cov}[y_s,y_{s'}]$ for all $s, s'\in \Sdom$,
the latter of which characterizes the spatial correlation structure of the phenomenon.  For example, $\sigma_{ss'}$ can be defined by the commonly-used squared exponential covariance function
$\sigma_{ss'} \triangleq \sigma_{y}^2\exp\{-0.5(s-s')^{\top}\Gamma^{-2}(s-s')\}$ where $\sigma_{y}^2$ is the signal variance controlling the intensity of  output measurements and $\Gamma$ is a diagonal matrix with length-scale components $\ell_1$ and $\ell_2$ controlling the spatial correlation or ``similarity'' between output measurements in the respective east-west and north-south directions of the $2$D phenomenon.

All output measurements observed by the agent are corrupted by an additive noise $\varepsilon$, i.e., $z_{i, j} \triangleq y_{s_{i, j}}+\varepsilon$ for 
stage $i=0,\ldots, t$  and $j = 1, \ldots \kappa$ where $s_{i, j}$ is the $j$-th input location of macro-action $\sNew{i}$ at stage $i$, $z_{i, j}$ is the corresponding output measurement 
and $\varepsilon\sim\mathcal{N}(0,\sigma_n^2)$ with  the noise variance  $\sigma_n^2$.
Supposing the agent has gathered observations  $d_t = \langle \sHist{t}, \zHist{t} \rangle$ from
stages $0$ to $t$, 
the GP model can exploit these observations $d_t$ to perform probabilistic regression by 
providing a Gaussian posterior belief $p(\zNew{t+1} | \sNew{t+1}, d_t ) = \mathcal{N}(\mu_{\sNew{t+1} | d_t}, \Sigma_{\sNew{t+1} | \sHist{t}})$ 
of noisy output measurements for any $\kappa$ input locations $\sNew{t+1} \subset \Sdom$
with the following \emph{posterior} mean vector and covariance matrix, respectively~\cite{gpml}:

\begin{equation}
\begin{array}{rcl}
\mu_{\sNew{t+1} | d_t} &\hspace{-2.4mm}\triangleq&\hspace{-2.4mm}\displaystyle   \Sigma_{\sNew{t+1} \sHist{t}}   \Sigma_{\sHist{t} \sHist{t}}^{-1}   \zHist{t}^\top\ , \vspace{0.5mm}\\ 
\Sigma_{\sNew{t+1} | \sHist{t}} &\hspace{-2.4mm}\triangleq&\hspace{-2.4mm}\displaystyle \Sigma_{\sNew{t+1} \sNew{t+1}}\hspace{-0.5mm} - \Sigma_{\sNew{t+1} \sHist{t}}  \Sigma_{\sHist{t} \sHist{t}}^{-1}  \Sigma_{\sHist{t} \sNew{t+1}}
\end{array}\hspace{-2.5mm}
\label{gp-posteriors}
\end{equation}
where
$\Sigma_{\sNew{t+1} \sHist{t}}$ is a matrix with covariance components $\sigma_{ss'}$ for every input location $s$ of $\sNew{t+1}$ and $s'$ of $\sHist{t}$, $\Sigma_{\sHist{t} \sNew{t+1}}$ is the transpose of $\Sigma_{\sNew{t+1} \sHist{t}}$, and  $\Sigma_{\sHist{t} \sHist{t}}$ $(\Sigma_{\sNew{t+1} \sNew{t+1}})$ is a matrix with covariance components $\sigma_{ss'}+\sigma^2_n\delta_{ss'}$ for every  pair of input locations $s, s'$ of $\sHist{t}$ $(\sNew{t+1})$ and $\delta_{ss'}$ is a Kronecker delta of value $1$ if $s=s'$, and $0$ otherwise. A key property of the GP model is that, different from $\mu_{\sNew{t+1} | d_t}$, $\Sigma_{\sNew{t+1} | \sHist{t}}$ is independent of the output measurements $\zHist{t}$.
\begin{figure*}
	{\begin{tabular}{ccc}
			\includegraphics[height=6.0cm]{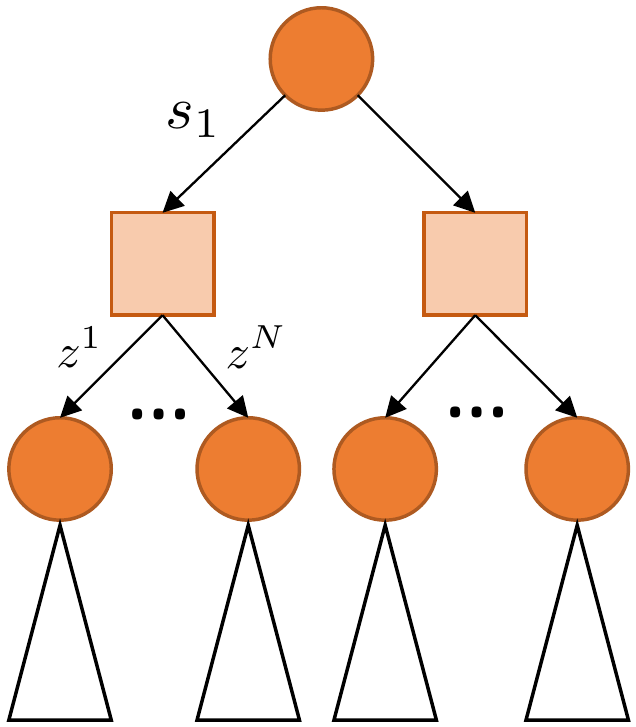} &  \hspace{3mm}
			\includegraphics[height=6.0cm]{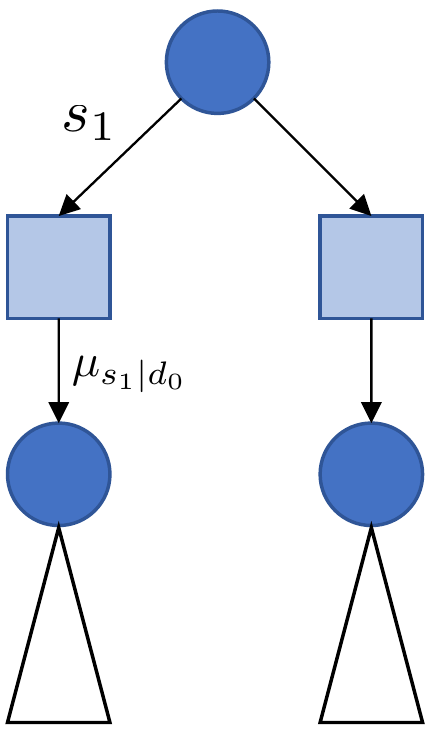} &  \hspace{3mm}
			\includegraphics[height=6.0cm]{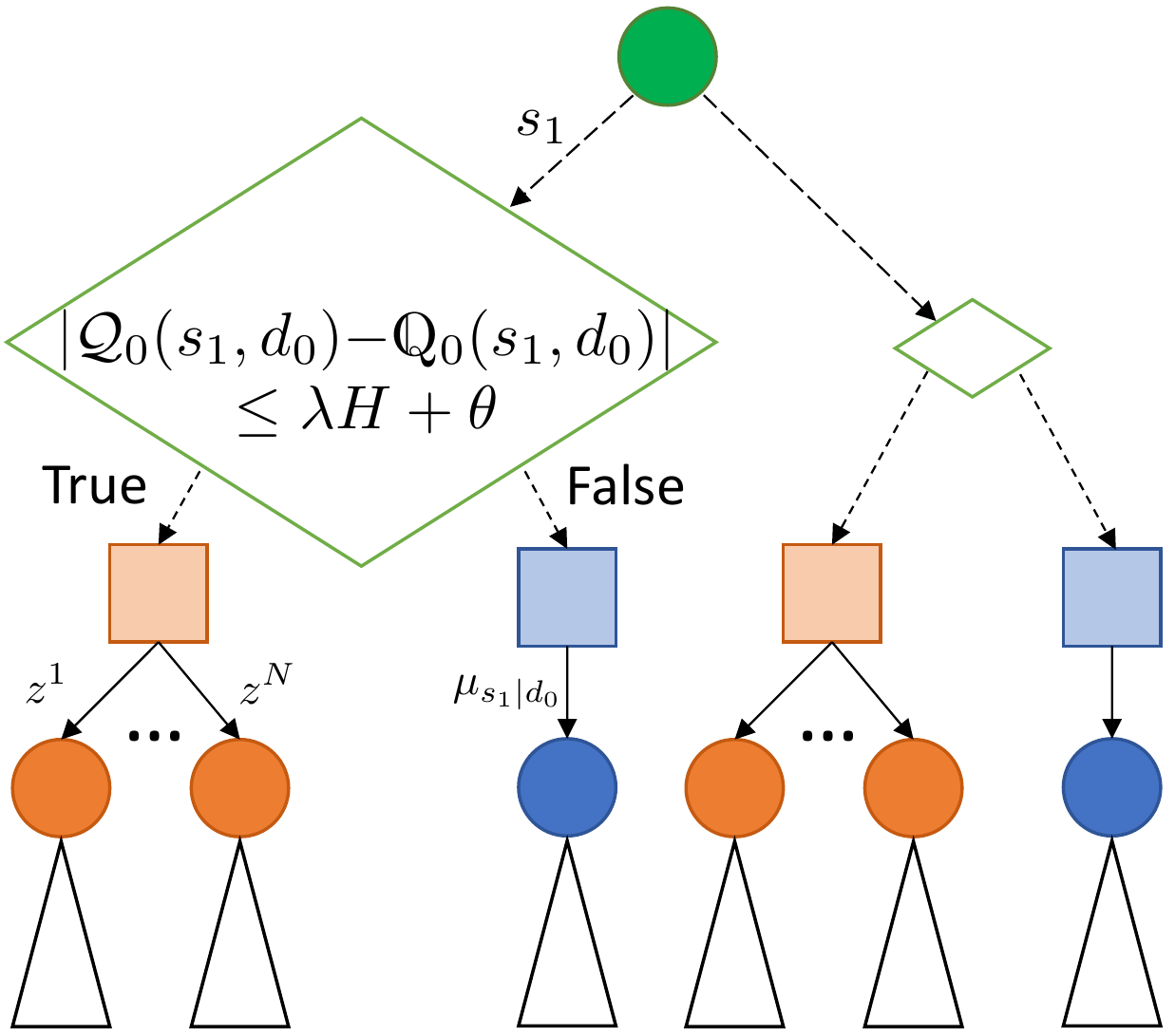}\\
			{(a)} &{(b)} &  {(c)}
	\end{tabular}}
	\caption{Visual illustrations of policies induced by (a) stochastic sampling~\eqref{approx-policy},  (b) most likely observations~\eqref{ml-policy}, and (c) our $\epsilon$-Macro-GPO policy $\pi^{\epsilon}$~\eqref{eq_4_8}.  Circles denote nodes $d_t$ and  squares denote nodes $\langle \sHist{t+1}, \zHist{t} \rangle$.}
	\label{fig:policies}
\end{figure*}
\section{$\epsilon$-Bayes-Optimal Macro-GPO}
\label{main-section}
\subsection{Problem Formulation} 
To cast nonmyopic adaptive \emph{macro-action GP optimization} (Macro-GPO) as a Bayesian sequential decision problem, we define a nonmyopic adaptive macro-action policy $\pi$ to sequentially decide in each stage $t$ 
the next 
macro-action $\pi(d_t) \in \Adom(\sNew{t})$ to be executed for gathering $\kappa$ new observations based on the current observations $d_t$ over a finite planning horizon of $H$ stages (i.e., a  lookahead of $\kappa H$ observations). 
The goal of the agent is to plan/decide its macro-actions to visit input locations $\sHist{H}\triangleq \sNew{1} \oplus\ldots\oplus \sNew{H}$ with the maximum total corresponding output measurements 
\begin{equation*}
	\mathbf{1}^{\top} \zHist{H} = \sum_{t=1}^{H}\mathbf{1}^{\top} \zNew{t} = \sum_{t=1}^{H} \sum_{i=1}^{\kappa} z_{t, i} 
\end{equation*}
or, equivalently, minimum cumulative regret where $\zHist{H} \triangleq \zNew{1} \oplus\ldots\oplus \zNew{H}$ and $\zNew{t} \triangleq (z_{t, 1},\ldots,z_{t, \kappa})$.
However, since only the prior observations/data $d_0$ are known, the Macro-GPO problem involves finding a 
nonmyopic adaptive 
macro-action policy $\pi$ to select input locations 
$\sHist{H}$ 
to be visited by the agent with the
maximum \emph{expected} total corresponding  output measurements 
$\mathbb{E}_{ \zHist{H}| d_0, \pi} [ \mathbf{1}^{\top} \zHist{H}]$ instead.

Supposing the size of the available budget in a real-world task environment exceeds the 
lookahead of $\kappa H$ observations, 
it can afford a \emph{stronger exploration behavior} by including 
an additional weighted exploration term 
$\beta\  \mathbb{I}[y_{\Sdom} ; \zHist{H} | d_0,\pi]$; 
its effect on BO performance is empirically investigated in Section~\ref{expt}.
The conditional mutual information 
$\mathbb{I}   [y_{\Sdom} ; \zHist{H}  | d_0,\pi]$ here
can be interpreted as the information gain on the phenomenon over the entire domain $\Sdom$ (i.e., equivalent to $y_{\Sdom}$) from gathering observations $\langle \sHist{H} , \zHist{H}  \rangle$ selected according to the 
nonmyopic adaptive 
macro-action policy $\pi$ given the prior data $d_0$. Then, the acquisition function w.r.t. a nonmyopic adaptive macro-action policy $\pi$ when starting in $d_0$ and following $\pi$ thereafter can be defined as
\begin{equation}
	V_0^{\pi}(d_0) \triangleq \mathbb{E}_{\zHist{H}| d_0, \pi } [  \mathbf{1}^{\top} \zHist{H}]  + \beta\  \mathbb{I}[y_{\Sdom} ; \zHist{H} | d_0, \pi]\ .
\label{eq:policy_value} 
\end{equation}
Applying the chain rule for mutual information and a few other information-theoretic
results to~\eqref{eq:policy_value} yields the following $H$-stage Bellman equations (Appendix~\ref{general-policy-proof}): 
\begin{equation}
	\hspace{-1.9mm}
	\begin{array}{rcl}
		V_t^\pi(d_t) &\hspace{-2.4mm}\triangleq &\hspace{-2.4mm} Q_t^\pi( \pi(d_{t}), d_t)\ , \\
		Q_t^\pi(\sNew{t+1}, d_t)&\hspace{-2.4mm} \triangleq  &\hspace{-2.4mm} R(\sNew{t+1}, d_t)\ + \\ 
		&&\hspace{-2.4mm}\mathbb{E}_{\zNew{t+1} | \sNew{t+1}, d_{t}} [ V_{t+1}^\pi ( \langle \sHist{t+1},\zHist{t}\hspace{-0.5mm}\oplus\hspace{-0.5mm}\zNew{t+1} \rangle)]
	\end{array}\hspace{-4.7mm}
	\label{eq:general-policy} 
\end{equation}
for stages $t =0,\ldots,H-1$ where $V_H^{\pi}(d_H) \triangleq 0$ and 
\begin{equation}
	\hspace{-1.9mm}
	\begin{array}{c}
		R(\sNew{t+1}, d_t) \triangleq \mathbf{1}^{\top}\mu_{\sNew{t+1} | d_t} +  0.5\beta \log | I + \sigma_n^{-2}  \Sigma_{\sNew{t+1} | \sHist{t}} | \ .
	\end{array}
	\label{eq:reward-def}
\end{equation}
To solve the Macro-GPO problem, Bayes-optimality\footnote{Bayes-optimality is previously studied in discrete \emph{Bayesian reinforcement learning} (BRL)~\cite{Poupart2006} but its assumed discrete-valued output measurements and Markov property do not hold in Macro-GPO. Continuous BRLs~\cite{Ross09,Ross08} assume a known parametric observation function, the reward function to be independent of output measurements and previous input locations, and/or, when using GP, the most likely observations during planning with no performance guarantee.} is exploited to select input locations to be visited by the agent that maximize the expected total corresponding  output measurements (and, if the budget can afford, the additional weighted exploration term representing the information gain on the phenomenon) with respect to all possible induced sequences of future GP posterior beliefs $p(\zNew{t+1} | \sNew{t+1}, d_t )$ for $t = 0,\ldots, H-1$.
Formally, this involves choosing a nonmyopic adaptive macro-action policy $\pi$ to maximize $V^{\pi}_0 (d_0)$, which we call the Bayes-optimal Macro-GPO policy $\pi^{*}$. That
is, 
\begin{equation*}
	V_{0}^{*}(d_0) \triangleq V_{0}^{\pi^*}(d_0) = \max_{\pi} {V_{0}^{\pi}(d_0)}. 
\end{equation*} 
Plugging $\pi^*$ into $V_{t}^{\pi}(d_t)$ and $Q_{t}^{\pi}(\sNew{t+1}, d_t)$~\eqref{eq:general-policy} gives

\begin{equation} 
	\hspace{-1.9mm}
	\begin{array}{rcl}
		\displaystyle V_t^*(d_t)&\hspace{-2.4mm}  \triangleq & \hspace{-2.4mm}\max_{\sNew{t+1} \in \Adom(\sNew{t})} \hspace{-0.5mm} Q_t^*(\sNew{t+1}, d_t)\ ,  \\
		Q_t^* (\sNew{t+1}, d_t) &\hspace{-2.4mm} \triangleq & \hspace{-2.4mm} R(\sNew{t+1}, d_t)\ + \\
		& &\hspace{-2.4mm}\mathbb{E}_{\zNew{t+1} | \sNew{t+1}, d_{t}}[ V_{t+1}^* (\langle \sHist{t+1},\zHist{t}\hspace{-0.5mm}\oplus\hspace{-0.5mm}\zNew{t+1} \rangle)]
	\end{array}\hspace{-4.4mm}
\label{eq:OptimalValFunDef}  
\end{equation}
for stages $t=0,\ldots, H-1$ where $V_H^{*}(d_H) \triangleq 0$.\footnote{\label{footnote-h-diff}To understand the effect of $H$ on how much macro-action sequence information are required as inputs to the Bellman equations in \eqref{eq:OptimalValFunDef}-\eqref{eq_4_8}, 
	refer to Fig.~\ref{fig:robot-example} and its caption c for a visual illustration.} 
When the lookahead of  $\kappa H$  observations matches up to the available budget, 
the Bayes-optimal Macro-GPO policy $\pi^*$ can naturally trade off between exploration vs. exploitation without needing the additional weighted exploration term in~\eqref{eq:policy_value} or~\eqref{eq:reward-def} (i.e., $\beta=0$): Its selected macro-action $\pi^*(d_t) = \argmax_{\sNew{t+1} \in \Adom(\sNew{t})} \, Q^{*}_t (\sNew{t+1}, d_{t}) $ in each stage $t$ has to trade off between exploiting the current GP posterior belief $p(\zNew{t+1} | \pi^*(d_t), d_t )$ to maximize the expected total corresponding output measurements $R(\pi^*(d_t), d_t)= \mathbf{1}^{\top} \mu_{\pi^*(d_t) | d_t}$ vs.~improving the GP posterior belief of the phenomenon (i.e., exploration) so as to maximize the expected total output measurements $\mathbb{E}_{\zNew{t+1}| \pi^*(d_t), d_t}[V^*_{t+1}(\langle \sHist{t}\oplus\pi^*(d_t),\zHist{t}\oplus\zNew{t+1} \rangle)]$  in the later stages.
\begin{figure*}
	\hspace{-2.3mm}
	\begin{tabular}{cc}
		\hspace{-0mm}\includegraphics[height=1.9cm]{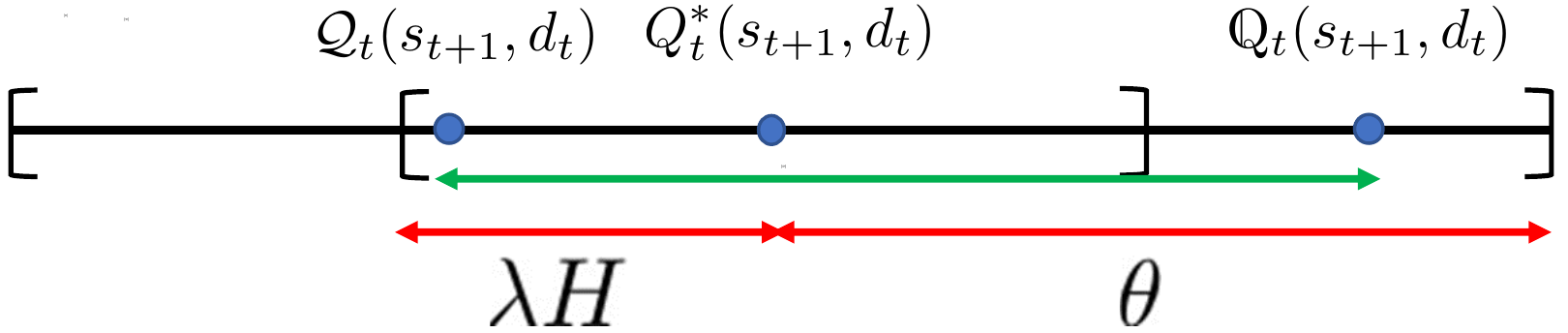}\hspace{-0mm} &
		\includegraphics[height=1.9cm]{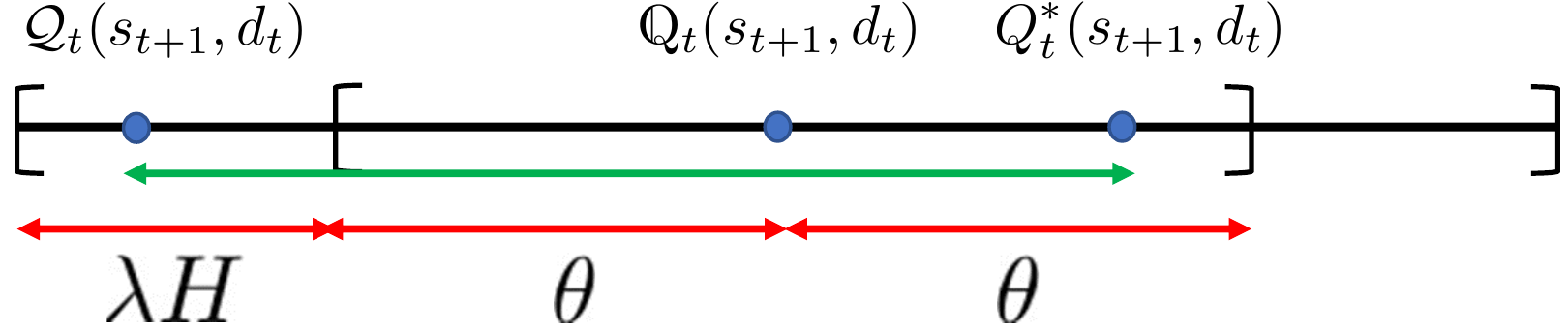}\\
		\hspace{-0mm} {(a)}  \hspace{-0mm} & {(b)}
	\end{tabular}\vspace{-0mm}
	\caption{Implications of specifying conditional policy~\eqref{eq_4_8}: 
		(a) When $| \mathcal{Q}_t (\sNew{t+1}, d_t) - {Q}^*_t (\sNew{t+1}, d_t) | \le \lambda H$, $|\mathcal{Q}_t (\sNew{t+1}, d_t) - \mathds{Q}_t(\sNew{t+1}, d_t)|$ (green) is at most 
		$\lambda H  + \theta$ (red). 
		(b) When $| \mathcal{Q}_t (\sNew{t+1}, d_t) - {Q}^*_t (\sNew{t+1}, d_t) | > \lambda H$ and $|\mathcal{Q}_t (\sNew{t+1}, d_t) - \mathds{Q}_t(\sNew{t+1}, d_t)| \le \lambda H  + \theta$, ${Q}^{\epsilon}_t (\sNew{t+1}, d_t)=\mathcal{Q}_t (\sNew{t+1}, d_t)$ due to~\eqref{eq_4_8} and 
		$| {Q}^{\epsilon}_t (\sNew{t+1}, d_t) - {Q}^*_t (\sNew{t+1}, d_t) |$ (green)  is at most $\lambda H +2\theta$ (red). 
		We do not show other cases (e.g., when 
		both $\mathcal{Q}_t (\sNew{t+1}, d_t)$ and $\mathds{Q}_t (\sNew{t+1}, d_t)$ are larger than  $Q_t^* (\sNew{t+1}, d_t)$ in (a) or 
		$| \mathcal{Q}_t (\sNew{t+1}, d_t) - {Q}^*_t (\sNew{t+1}, d_t) | > \lambda H$ and 
		$|\mathcal{Q}_t (\sNew{t+1}, d_t) - \mathds{Q}_t(\sNew{t+1}, d_t)| > \lambda H  + \theta$ in (b), ${Q}^{\epsilon}_t (\sNew{t+1}, d_t)=\mathds{Q}_t (\sNew{t+1}, d_t)$ due to~\eqref{eq_4_8})			
		which are all covered by our rigorous analysis in the text below~\eqref{eq_4_8}.}	
	\label{fig:conditions}
\end{figure*}

When the available budget is larger than the lookahead of  $\kappa H$  observations,
it can afford a \emph{stronger exploration behavior} by setting a positive weight $\beta>0$ on the exploration term $0.5 \log | I + \sigma_n^{-2}  \Sigma_{\pi^*(d_t) | \sHist{t}} |$ in~\eqref{eq:reward-def};
its effect on BO performance is empirically investigated in Section~\ref{expt}.
This exploration term can be interpreted as the information gain $\mathbb{I}[y_{\Sdom};\zNew{t+1}|d_t,\pi^*(d_t)]$ on the phenomenon (Appendix~\ref{general-policy-proof}) from executing the macro-action $\pi^*(d_t)$ to gather $\kappa$ new observations. As such, the macro-action $\pi^*(d_t)$ can gain more information on the phenomenon (larger exploration term) by gathering observations with higher uncertainty (larger individual posterior variance) but lower  correlation (smaller magnitude of posterior covariance) between them.
\subsection{$\epsilon$-Bayes-Optimal Macro-GPO ($\epsilon$-Macro-GPO)} 
In general, the 
Bayes-optimal Macro-GPO policy $\pi^\ast$ cannot be derived exactly because the expectation term in~\eqref{eq:OptimalValFunDef} (and hence $Q^{*}_t$ and $V^{*}_t$) often cannot be evaluated in closed form due to an uncountable set of candidate output measurements. 

To resolve this issue, we will exploit the following result on the Lipschitz continuity of $R(\sNew{t+1}, d_t)$~\eqref{eq:reward-def} and consequently of $V_t^*(d_t)$~\eqref{eq:OptimalValFunDef} in the realized output measurements $\zHist{t}$ (see Appendices~\ref{lemma:reward-proof} and~\ref{lip-optimal-value} for their respective proofs) to tractably derive a nonmyopic adaptive $\epsilon$-Macro-GPO policy $\pi^{\epsilon}$ whose expected performance loss is theoretically guaranteed to be not more than an arbitrarily user-specified loss bound $\epsilon$:
\begin{lemma} 
	\label{lemma:reward}
	Let $\alpha(\sHist{t+1}) \triangleq \lVert\Sigma_{\sNew{t+1}\sHist{t}}\Sigma_{\sHist{t}\sHist{t}}^{-1}\rVert_F$ 
	and $d_{t}'\triangleq\langle \sHist{t}, \zHist{t}' \rangle$. 
	Then,
	$$| R(\sNew{t+1}, d_t) - R(\sNew{t+1}, d'_{t}) | \le \sqrt{\kappa}\ \alpha(\sHist{t+1}) \lVert\zHist{t} - \zHist{t}'\rVert.
	$$
\end{lemma}
Preliminary to the design and construction of our proposed nonmyopic adaptive $\epsilon$-Macro-GPO policy $\pi^{\epsilon}$  is the approximation of the expectation term in~\eqref{eq:OptimalValFunDef} for each candidate macro-action $\sNew{t+1}$ in every stage using \emph{stochastic} sampling of $N$ i.i.d. multivariate Gaussian vectors $\zNew{}^1,\ldots, \zNew{}^N$ from the GP posterior belief $p(\zNew{t+1} | \sNew{t+1}, d_t )$~\eqref{gp-posteriors}, as illustrated in Fig.~\ref{fig:policies}a:
\begin{equation}
	\hspace{-1.9mm}
	\begin{array}{rcl}
		\displaystyle
		\mathcal{V}_t (d_t) & \hspace{-2.4mm}\triangleq &  \hspace{-2.4mm}\max_{\sNew{t+1} \in \Adom(\sNew{t})} \mathcal{Q}_t (\sNew{t+1}, d_t)\ , \\
		\mathcal{Q}_t (\sNew{t+1}, d_t) & \hspace{-2.4mm}\triangleq & \hspace{-2.4mm} \displaystyle R(\sNew{t+1}, d_t)\hspace{-0.5mm}+\hspace{-0.5mm}\frac{1}{N}\hspace{-0.5mm} \sum_{\ell = 1}^{N} 
		\mathcal{V}_{t+1} (\langle \sHist{t+1}, \zHist{t}\hspace{-0.5mm}\oplus\hspace{-0.5mm} \zNew{}^\ell \rangle)
	\end{array}
	\label{approx-policy} 
\end{equation}
for stages $t =0, \ldots, H-1$ where $\mathcal{V}_H (d_H) \triangleq 0$.\cref{footnote-h-diff}
We prove in Appendix~\ref{sec:th:1_new_proof} that 
$\mathcal{Q}_t (\sNew{t+1}, d_t)$~\eqref{approx-policy} can approximate ${Q}^*_t (\sNew{t+1}, d_t)$~\eqref{eq:OptimalValFunDef} arbitrarily closely for all 
$\sNew{t+1}$ 
with a high probability of at least $1 - \delta$ requiring only a polynomial number $N$ of samples in the macro-action length $\kappa$~\eqref{eq:expected_samples} per planning stage:
\begin{theorem}
	\label{th:1_new}
	Suppose that the observations $d_{t}$, $H\in\mathbb{Z}^+$, a budget of $\kappa(H-t)$ input locations for $t=0,\ldots, H-1$, $\delta\in(0, 1)$, and $\lambda > 0$ are given.
	Then, the probability of 
	$$| \mathcal{Q}_t (\sNew{t+1}, d_t) - {Q}^*_t (\sNew{t+1}, d_t) | \le \lambda H$$ 
	for all $\sNew{t+1} \in \Adom(\sNew{t})$
	is at least $1 - \delta$ by setting
	\begin{equation}
	\begin{array}{c}
	N =\mathcal{O}((\kappa^{2H}/\lambda^2) \log(\kappa A /(\delta\lambda)))
	\end{array}
	\label{eq:th1-samples}
	\end{equation}
	where $A$ is the largest number of candidate macro-actions available in any input location.
\end{theorem}
\begin{remark}
	\label{rem:3}
	Since 
	$|	\mathcal{V}_{t} (d_{t} ) -  V^{*}_{t} (d_{t} )|   \le   \max_{\sNew{t+1} \in \Adom(\sNew{t})} |  \mathcal{Q}_{t}(\sNew{t+1}, d_t)  -  Q_{t}^{*}(\sNew{t+1},d_{t}) |$, 
	it immediately follows from Theorem~\ref{th:1_new} that the probability of $|\mathcal{V}_{t} (d_{t} ) -  V^{*}_{t} (d_{t} )| \le \lambda H$ is 
	at least $1- \delta$.
\end{remark}
\begin{remark}
	It can be observed from Theorem~\ref{th:1_new} that the number $N$~\eqref{eq:th1-samples} of stochastic samples increases\cref{joker} with (a) a tighter bound $\lambda$ on the error $|\mathcal{Q}_t (\sNew{t+1}, d_t) - {Q}^*_t (\sNew{t+1}, d_t) |$ due to stochastic sampling, 
	(b) a higher probability $1- \delta$ of $\mathcal{Q}_t$~\eqref{approx-policy} approximating ${Q}^*_t$~\eqref{eq:OptimalValFunDef} closely, 
	(c) a larger number $A$ of candidate macro-actions, and (d) a greater macro-action length $\kappa$.
\end{remark}
Deriving the above probabilistic bound usually requires using a concentration inequality involving independent Gaussian random variables. However, 
the components of the multivariate Gaussian random vector $\zNew{t+1}$ in~\eqref{eq:OptimalValFunDef} are \emph{correlated} output measurements corresponding to the $\kappa$ input locations found within the candidate macro-action $\sNew{t+1}$.
To resolve this complication, we exploit a change of variables trick (i.e., to make the components independent) and  the Lipschitz continuity of $R(\sNew{t+1}, d_t)$ (Lemma~\ref{lemma:reward}) for enabling the use of the Tsirelson-Ibragimov-Sudakov inequality~\cite{bouch} to prove the probabilistic bound in  Theorem~\ref{th:1_new}, as shown in Appendix~\ref{sec:th:1_new_proof}.

Theorem~\ref{th:1_new}, however, only entails probabilistic bounds on how far $\mathcal{V}_t (d_t)$~\eqref{approx-policy} is from ${V}^*_t (d_t)$~\eqref{eq:OptimalValFunDef} (see Remark~\ref{rem:3}) and on the resulting policy loss.
We will prove a stronger non-trivial result: In the unlikely event (with an arbitrarily small probability of at most $\delta$) that 
$\mathcal{Q}_t (\sNew{t+1}, d_t)$~\eqref{approx-policy} is unboundedly far from ${Q}^*_t (\sNew{t+1}, d_t)$~\eqref{eq:OptimalValFunDef} for some 
$\sNew{t+1}$, 
we instead rely on the $\kappa$ most likely observations\footnote{Though the nonmyopic BO algorithm of~\cite{marchant14} assumes the most likely observations during planning, it does not consider macro-actions nor give a performance guarantee.} $\mu_{\sNew{t+1} | d_t}$ for approximating the expectation term in~\eqref{eq:OptimalValFunDef}  (see Fig.~\ref{fig:policies}b):
\begin{equation}
\hspace{-1.9mm}
	\begin{array}{rcl}
	\displaystyle\mathds{V}_t (d_t) &\hspace{-2.4mm}\triangleq &\hspace{-2.4mm}\max_{\sNew{t+1} \in \Adom(\sNew{t})} \mathds{Q}_t (\sNew{t+1}, d_t)\ ,\vspace{0.5mm}\\
	\mathds{Q}_t (\sNew{t+1}, d_t) &\hspace{-2.4mm}\triangleq  &\hspace{-2.4mm} R(\sNew{t+1}, d_t)\hspace{-0.5mm}+ \hspace{-0.5mm}\mathds{V}_{t+1} (\langle \sHist{t+1}, \zHist{t}\hspace{-0.5mm}\oplus\hspace{-0.5mm}\mu_{\sNew{t+1} | d_t} \rangle)
	\end{array}	
\label{ml-policy}
\end{equation}	
for stages $t =0, \ldots, H-1$ where $\mathds{V}_H (d_H)  \triangleq 0$.\cref{footnote-h-diff} 
Unlike $\mathcal{Q}_t (\sNew{t+1}, d_t)$~\eqref{approx-policy}, the approximation quality of $\mathds{Q}_t (\sNew{t+1}, d_t)$~\eqref{ml-policy} can be \emph{deterministically} bounded but cannot be user-specified to be arbitrarily good, as shown in Theorem~\ref{th-mle_bound} below (see Appendix~\ref{sec:th-mle_bound_proof} for the proof).
To ease understanding, we visually illustrate in Fig.~\ref{fig:policies} how the policies induced by stochastic sampling~\eqref{approx-policy} vs.~most likely observations~\eqref{ml-policy} differ and are used to design our $\epsilon$-Macro-GPO policy $\pi^{\epsilon}$~\eqref{eq_4_8}.
\begin{theorem}
	\label{th-mle_bound}
	Suppose that the observations $d_{t}$, $H\in\mathbb{Z}^+$, and a budget of $\kappa(H-t)$ input locations for $t=0,\ldots, H-1$ are given. Then, 
	$$| \mathds{Q}_t (\sNew{t+1}, d_t) - {Q}^*_t (\sNew{t+1}, d_t) | \le \theta$$
	for all $\sNew{t+1} \in \Adom(\sNew{t})$ where $\theta \triangleq \mathcal{O}( \kappa^{H + 1/2})$. 
\end{theorem}
\begin{remark}
	\label{ref:remark-chun}
	$\mathds{V}_t$~\eqref{ml-policy} can be potentially generalized to resemble $\mathcal{V}_t$~\eqref{approx-policy} by approximating the expectation term in \eqref{eq:OptimalValFunDef} 
	for each candidate macro-action $\sNew{t+1}$ in every stage 
	via \emph{deterministic} sampling from the GP posterior belief $p(\zNew{t+1} | \sNew{t+1}, d_t ) = \mathcal{N}(\mu_{\sNew{t+1} | d_t}, \Sigma_{\sNew{t+1} | \sHist{t}})$~\eqref{gp-posteriors} over the $\kappa$-dimensional output measurement space of $\zNew{t+1}$.
	To do this, the nonmyopic adaptive BO algorithm of \cite{ling16} can be extended to handle macro-actions by uniformly partitioning and sampling the $\kappa$-dimensional space of $\zNew{t+1}$ but would consequently incur an \emph{exponential} number of samples (in $\kappa$) per planning stage.
	In contrast, our $\epsilon$-Macro-GPO policy $\pi^\epsilon$ only requires a polynomial number (in $\kappa$) of samples per planning stage, as shown in Theorem~\ref{th:expected}.
\end{remark}
The key question remains: Under what condition(s) should our $\epsilon$-Macro-BO policy $\pi^\epsilon$ decide to follow that induced by stochastic sampling~\eqref{approx-policy} and, 
if so, what is the required number $N$ of samples in~\eqref{approx-policy} 
such that its \emph{expected} performance loss can be deterministically guaranteed to be within an arbitrarily user-specified bound $\epsilon$?
Ideally, this can be decided if we can directly assess whether $\mathcal{Q}_t (\sNew{t+1}, d_t)$~\eqref{approx-policy} approximates ${Q}^*_t (\sNew{t+1}, d_t)$~\eqref{eq:OptimalValFunDef} closely (i.e., $| \mathcal{Q}_t (\sNew{t+1}, d_t) - {Q}^*_t (\sNew{t+1}, d_t) | \le \lambda H$) for all $\sNew{t+1} \in \Adom(\sNew{t})$, which unfortunately is not possible since ${Q}^*_t (\sNew{t+1}, d_t)$ cannot be tractably evaluated, as explained previously. 
To overcome this technical difficulty, we propose a nonmyopic adaptive $\epsilon$-Macro-BO policy $\pi^\epsilon$ that decides to strictly follow that induced by stochastic sampling~\eqref{approx-policy} 
{only} if $\mathcal{Q}_t (\sNew{t+1}, d_t)$~\eqref{approx-policy} is boundedly close to $\mathds{Q}_t(\sNew{t+1}, d_t)$~\eqref{ml-policy} for all $\sNew{t+1} \in \Adom(\sNew{t})$:
\begin{equation}
\hspace{-1.9mm}
\begin{array}{rcl}
\displaystyle
\pi^{\epsilon}(d_t) &\hspace{-2.4mm} \triangleq & \hspace{-2.4mm} \argmax_{\sNew{t+1} \in \Adom(\sNew{t})} Q^{\epsilon}_t (\sNew{t+1}, d_{t})\ , \\
Q^{\epsilon}_t(\sNew{t+1}, d_t) &\hspace{-2.4mm} \triangleq  & \hspace{-2.4mm}
\begin{cases}
\mathcal{Q}_t(\sNew{t+1}, d_t) &\hspace{-3.7mm}  
\begin{array}{l}
\text{if } |\mathcal{Q}_t\hspace{-0.3mm} (\sNew{t+1}, d_t)\hspace{-0.8mm} -\hspace{-0.6mm} \mathds{Q}_t\hspace{-0.3mm}(\sNew{t+1}, d_t)|\\ 
\quad \le \lambda H  +\theta\ , 
\end{array}
\\
\mathds{Q}_t\hspace{-0.3mm} (\sNew{t+1}, d_t)     & \hspace{-2mm}\text{otherwise};
\end{cases}
\end{array}  		
\label{eq_4_8}
\end{equation}
for stages $t =0, \ldots, H-1$.\cref{footnote-h-diff} 
Like the Macro-GPO policy $\pi^*$, $\pi^\epsilon$ can also naturally trade off between exploration vs.~exploitation, by the same reasoning as earlier. Unlike the deterministic policy $\pi^*$, $\pi^\epsilon$ is stochastic due to its use of stochastic sampling in $\mathcal{Q}_t$~\eqref{approx-policy}.

Of noteworthy interest and discussion are the implications of the tractable choice of the if condition in~\eqref{eq_4_8}
for theoretically guaranteeing the performance of our $\epsilon$-Macro-BO policy $\pi^\epsilon$, which we illustrate in Fig.~\ref{fig:conditions}:\vspace{0.5mm}

\noindent
{\bf I.} In the likely event (with a high probability of at least $1-\delta$) that $| \mathcal{Q}_t (\sNew{t+1}, d_t) - {Q}^*_t (\sNew{t+1}, d_t) | \le \lambda H$ for all $\sNew{t+1} \in \Adom(\sNew{t})$ (Theorem~\ref{th:1_new}), 
\begin{equation*}
\begin{array}{l}
\displaystyle |\mathcal{Q}_t (\sNew{t+1}, d_t) -\mathds{Q}_t(\sNew{t+1}, d_t)|\\
\displaystyle \le | \mathcal{Q}_t (\sNew{t+1}, d_t) - {Q}^*_t (\sNew{t+1}, d_t) |  \\
\displaystyle \quad + |{Q}^*_t (\sNew{t+1}, d_t) - \mathds{Q}_t \hspace{-0.3mm}(\sNew{t+1}, d_t)| \\
\displaystyle \leq \lambda H  + \theta
\end{array}
\end{equation*} 
for all $\sNew{t+1} \in \Adom(\sNew{t})$ such that the first inequality is due to triangle inequality and the second inequality is due to Theorems~\ref{th:1_new} and~\ref{th-mle_bound}. 
Consequently, according to~\eqref{eq_4_8}, $Q^{\epsilon}_t (\sNew{t+1}, d_{t}) = \mathcal{Q}_t (\sNew{t+1}, d_t)$ for all $\sNew{t+1} \in \Adom(\sNew{t})$ and $\pi^\epsilon(d_{t})$ thus selects the same macro-action as the policy induced by stochastic sampling~\eqref{approx-policy}.\vspace{0.5mm}

\noindent
{\bf II.}  In the unlikely event (with an arbitrarily small probability of at most $\delta$) that 
$\mathcal{Q}_t (\sNew{t+1}, d_t)$~\eqref{approx-policy} is unboundedly far from ${Q}^*_t (\sNew{t+1}, d_t)$~\eqref{eq:OptimalValFunDef}  
(i.e., $| \mathcal{Q}_t (\sNew{t+1}, d_t) - {Q}^*_t (\sNew{t+1}, d_t) | > \lambda H$) for some $\sNew{t+1} \in \Adom(\sNew{t})$, 
$\pi^\epsilon(d_{t})$~\eqref{eq_4_8} guarantees that, for any selected macro-action $\sNew{t+1} \in \Adom(\sNew{t})$,
$$
\hspace{-1.9mm}
\begin{array}{l}
| {Q}^{\epsilon}_t (\sNew{t+1}, d_t) - {Q}^*_t (\sNew{t+1}, d_t) |\\
=\hspace{-1mm}
\begin{cases}
|\mathcal{Q}_t \hspace{-0.3mm}(\sNew{t+1}, d_t)\hspace{-0.8mm} -\hspace{-0.6mm} {Q}^*_t \hspace{-0.3mm}(\sNew{t+1}, d_t) | &\hspace{-1.05mm}  
\begin{array}{l}
\text{if } |\mathcal{Q}_t\hspace{-0.3mm} (\sNew{t+1}, d_t)\hspace{-0.8mm} -\hspace{-0.6mm} \mathds{Q}_t\hspace{-0.3mm}(\sNew{t+1}, d_t)| \\
\quad\le \lambda H  + \theta , 
\end{array}
\\
|\mathds{Q}_t \hspace{-0.3mm}(\sNew{t+1}, d_t)\hspace{-0.8mm} -\hspace{-0.6mm} {Q}^*_t \hspace{-0.3mm}(\sNew{t+1}, d_t) |     & \hspace{0.65mm}\text{otherwise};
\end{cases}
\\
\le\hspace{-1mm}
\begin{cases}
\hspace{-1.78mm}
\begin{array}{l}
|\mathcal{Q}_t \hspace{-0.3mm}(\sNew{t+1}, d_t)\hspace{-0.8mm} -\hspace{-0.6mm} \mathds{Q}_t \hspace{-0.3mm}(\sNew{t+1}, d_t) | \\
+ |\mathds{Q}_t \hspace{-0.3mm}(\sNew{t+1}, d_t)\hspace{-0.8mm} -\hspace{-0.6mm} {Q}^*_t \hspace{-0.3mm}(\sNew{t+1}, d_t) |
\end{array}
&\hspace{-5.5mm}  
\begin{array}{l}
\text{if } |\mathcal{Q}_t\hspace{-0.3mm} (\sNew{t+1}, d_t)\hspace{-0.8mm} -\hspace{-0.6mm} \mathds{Q}_t\hspace{-0.3mm}(\sNew{t+1}, d_t)| \\
\quad\le \lambda H  + \theta , 
\end{array}
\\
\theta     & \hspace{-3.8mm}\text{otherwise};
\end{cases}
\\
\le \lambda H +2\theta\ ,\quad\text{by triangle inequality and Theorem~\ref{th-mle_bound}.}
\end{array}
$$
The above two implications of our tractable choice of the if condition in~\eqref{eq_4_8} are central to establishing our main result deterministically  	
bounding the \emph{expected} performance loss of $\pi^\epsilon$  relative to that of Bayes-optimal Macro-BO policy $\pi^*$, that is, policy $\pi^\epsilon$ is $\epsilon$-Bayes-optimal.
\begin{figure*}
	\centering
	{\begin{tabular}{cccc}
			\hspace{-4.5mm}\includegraphics[width=0.24 \textwidth]{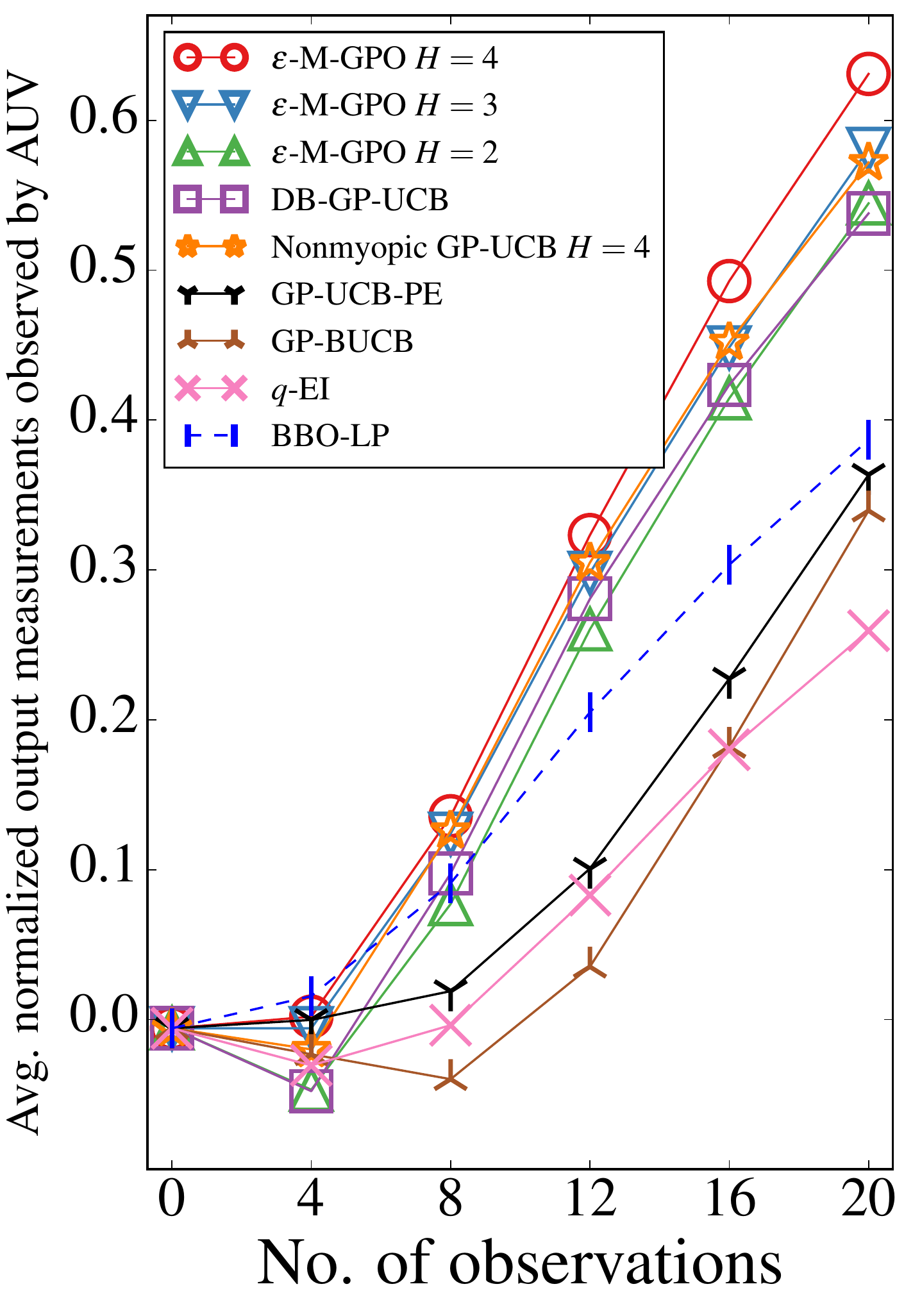} & \hspace{-4mm}
			\includegraphics[width=0.24 \textwidth]{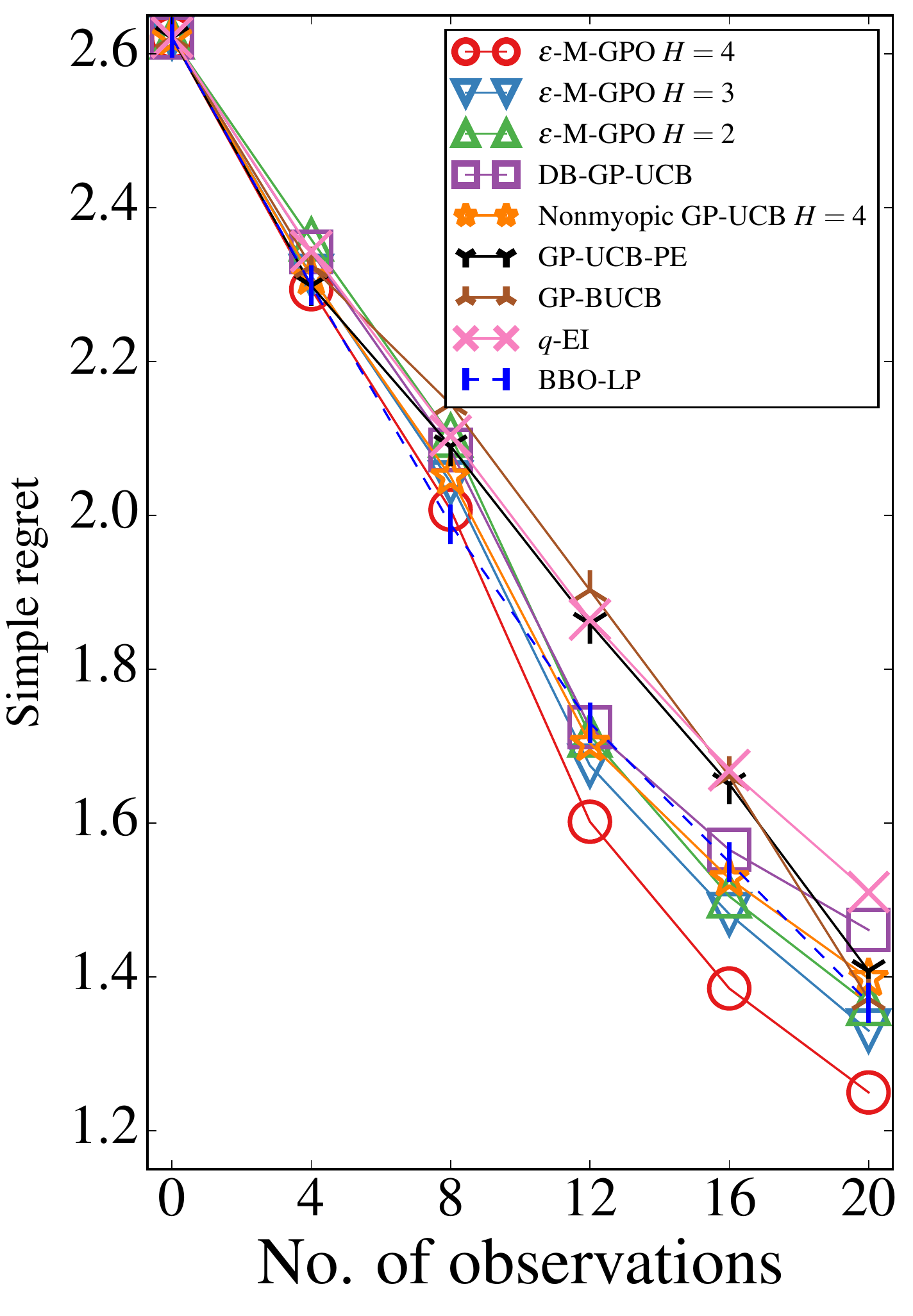} &
			\hspace{-4mm} \includegraphics[width=0.24 \textwidth]{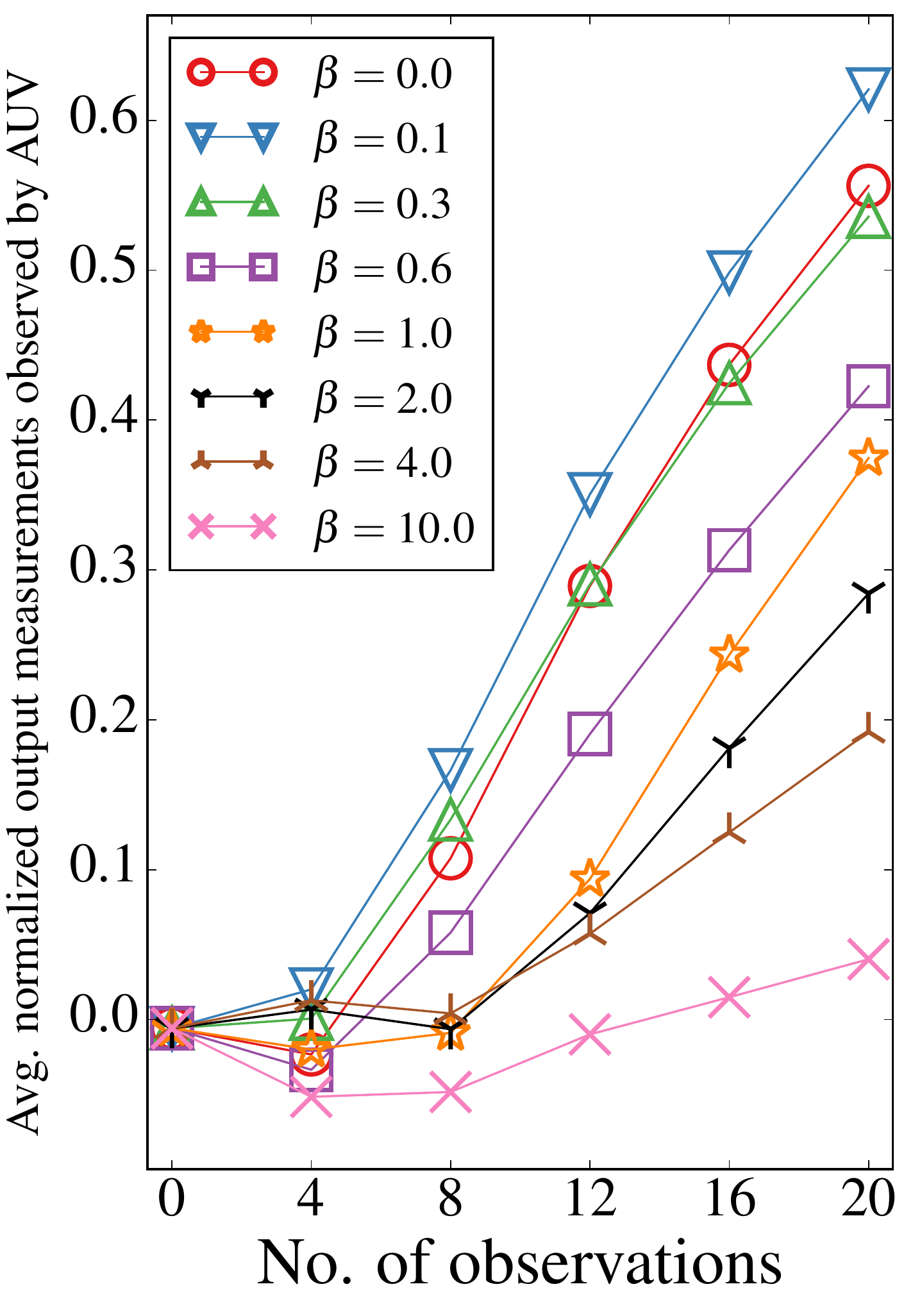} &
			\hspace{-3mm}\includegraphics[width=0.24 \textwidth]{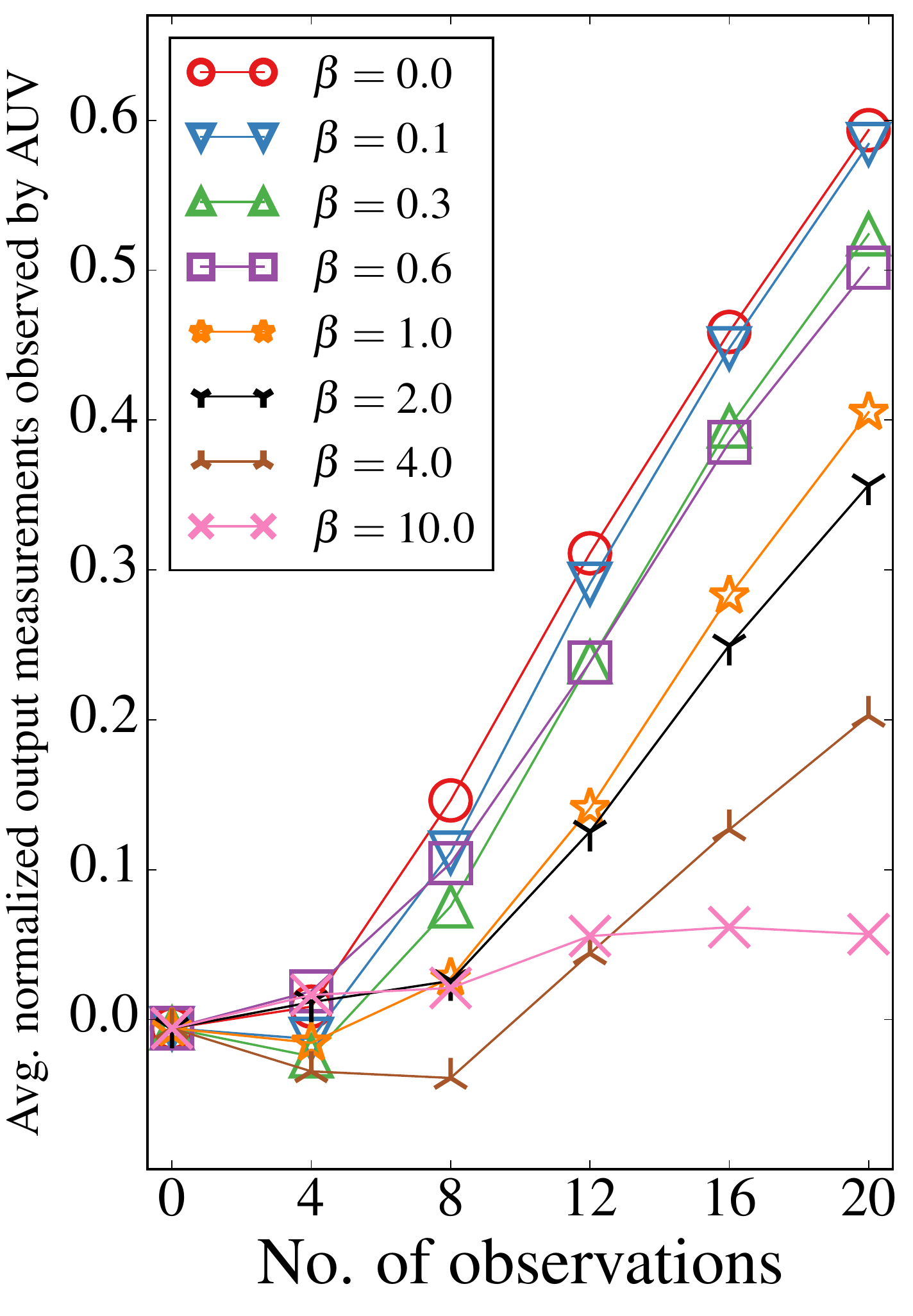}  \\
			\hspace{-4.5mm}{\small (a)}  &\hspace{-4mm}{\small (b)}  & \hspace{-4mm} {\small (c)} 
			& \hspace{-3mm} {\small (d)} \\
	\end{tabular}}
	\caption{Graphs of  (a) average normalized\cref{footnote-rewards} output measurements observed by AUV, (b) simple regrets achieved by tested BO algorithms,  average normalized output measurements achieved by $\epsilon$-Macro-GPO ($\epsilon$-M-GPO in the graphs) with  (c) $H=2$  and  (d) $H = 3$ and varying exploration weights $\beta$  vs. no. of observations for simulated plankton density phenomena. 
		Standard errors are given in Tables~\ref{table:simulated-var-methods} and~\ref{table::simulated-var-beta} in Appendix~\ref{plankton}.} 
	\label{fig:simulated_total}
\end{figure*}

To understand the rationale/implications of our choice of 
if condition in~\eqref{eq_4_8}, refer to Fig.~\ref{fig:conditions}.
These implications are central to establishing our main result deterministically bounding the \emph{expected} performance loss of $\pi^\epsilon$  relative to that of $\pi^*$, i.e., $\pi^\epsilon$ is $\epsilon$-Bayes-optimal (see proof in Appendix~\ref{th:expected-proof}):
\begin{theorem}
	\label{th:expected}
	Suppose that the observations $d_{0}$, $H\in\mathbb{Z}^+$, a budget of $\kappa H$ input locations, and a user-specified loss bound $\epsilon > 0$ are given. 		
	Then, $V^*_0(d_0) - \mathbb{E}_{\pi^{\epsilon}}[V^{\pi^\epsilon}_0(d_0)] \le \epsilon$ 
	for policy $\pi^\epsilon$ defined in \eqref{eq_4_8} , 
	by setting $\theta \triangleq \mathcal{O}( \kappa^{H + 1/2})$  according to Theorem~\ref{th-mle_bound}, 
	$\delta = {\epsilon}/(8 \theta H)$, and $\lambda = {\epsilon}/(4 H^2)$ 
	in Theorem~\ref{th:1_new} 
	to yield 
	\begin{equation}
	\begin{array}{c}
	N  = \mathcal{O}( (\kappa^{2H}/{\epsilon^2}) \log(\kappa A /\epsilon))
	\end{array}
	\label{eq:expected_samples}
	\end{equation}
	where $A$ denotes the largest number of candidate macro-actions available at any input location in $\Sdom$.
\end{theorem}
\begin{remark}
	It can be observed from Theorem~\ref{th:expected} that the number $N$ of stochastic samples increases\footnote{\label{joker}In fact, $N$ also increases when a larger $H$ is available and the spatial phenomenon varies with more intensity and less noise (larger $\sigma^2_y/\sigma^2_n$) (Appendix~\ref{sec:th:1_new_proof}).
		These constants are omitted from~\eqref{eq:expected_samples} to ease clutter.} with (a) a tighter user-specified loss bound $\epsilon$, (b) a larger number $A$ of candidate macro-actions at any input location in $\Sdom$,
	and (c) a greater macro-action length $\kappa$.  
\end{remark} 
\subsection{Anytime $\epsilon$-Macro-GPO}
Unlike the Bayes-optimal Macro-GPO
policy $\pi^*$, our  $\epsilon$-Macro-GPO
policy $\pi^{\epsilon}$ can be derived exactly since its incurred time does not depend on the size of the uncountable set of candidate output measurements.
But, deriving $\pi^{\epsilon}$~\eqref{eq_4_8} requires expanding an entire search tree of $\mathcal{O}(N^H)$ nodes to solve the $H$-stage Bellman equations of $\mathcal{V}_t$~\eqref{approx-policy}, which 
 which incurs time with a $\mathcal{O}(N^H)$ term and 
is not always needed to achieve $\epsilon$-Bayes optimality in practice.
To ease this computational burden (e.g., for real-time planning), we propose an asymptotically optimal anytime variant of our $\epsilon$-Macro-GPO policy that can attain good BO performance 
quickly and improve its approximation quality over time, as briefly discussed here and detailed along with the pseudocode in Appendix~\ref{sec:anytime}.

The intuition behind our anytime $\epsilon$-Macro-GPO algorithm is to incrementally expand a search tree by iteratively simulating greedy exploration paths down the partially constructed tree and expanding the sub-trees rooted at nodes with the largest uncertainty of their corresponding values $V^{*}_t(d_t)$ so as to improve their approximation quality. Such an uncertainty at each encountered node $d_t$ is quantified by the gap between its maintained upper and lower heuristic bounds $\overline{V}^{*}_t(d_t)$ and $\underline{V}^{*}_t(d_t)$ that are (a) tightened via backpropagation from the leaves up through node $d_t$ to the root $d_0$ and (b) subsequently used to refine that at its siblings by exploiting the Lipschitz continuity of $V^{*}_t$ (Appendix~\ref{lip-optimal-value}).
%
%
%
Consequently, each iteration of our anytime $\epsilon$-Macro-GPO algorithm only incurs linear time in $N$.
The formulation of our anytime variant resembles that of $\epsilon$-Macro-GPO policy $\pi^{\epsilon}$~\eqref{eq_4_8} except that it utilizes the lower heuristic bound instead of $\mathcal{Q}_t$~\eqref{approx-policy} and a modified if condition to bound its expected performance loss likewise, 
as detailed in Appendix~\ref{sec:anytime}.
\begin{figure*}[t]
	\centering
	\begin{tabular}{cc}
		\hspace{-4mm} \includegraphics[width=0.45 \textwidth]{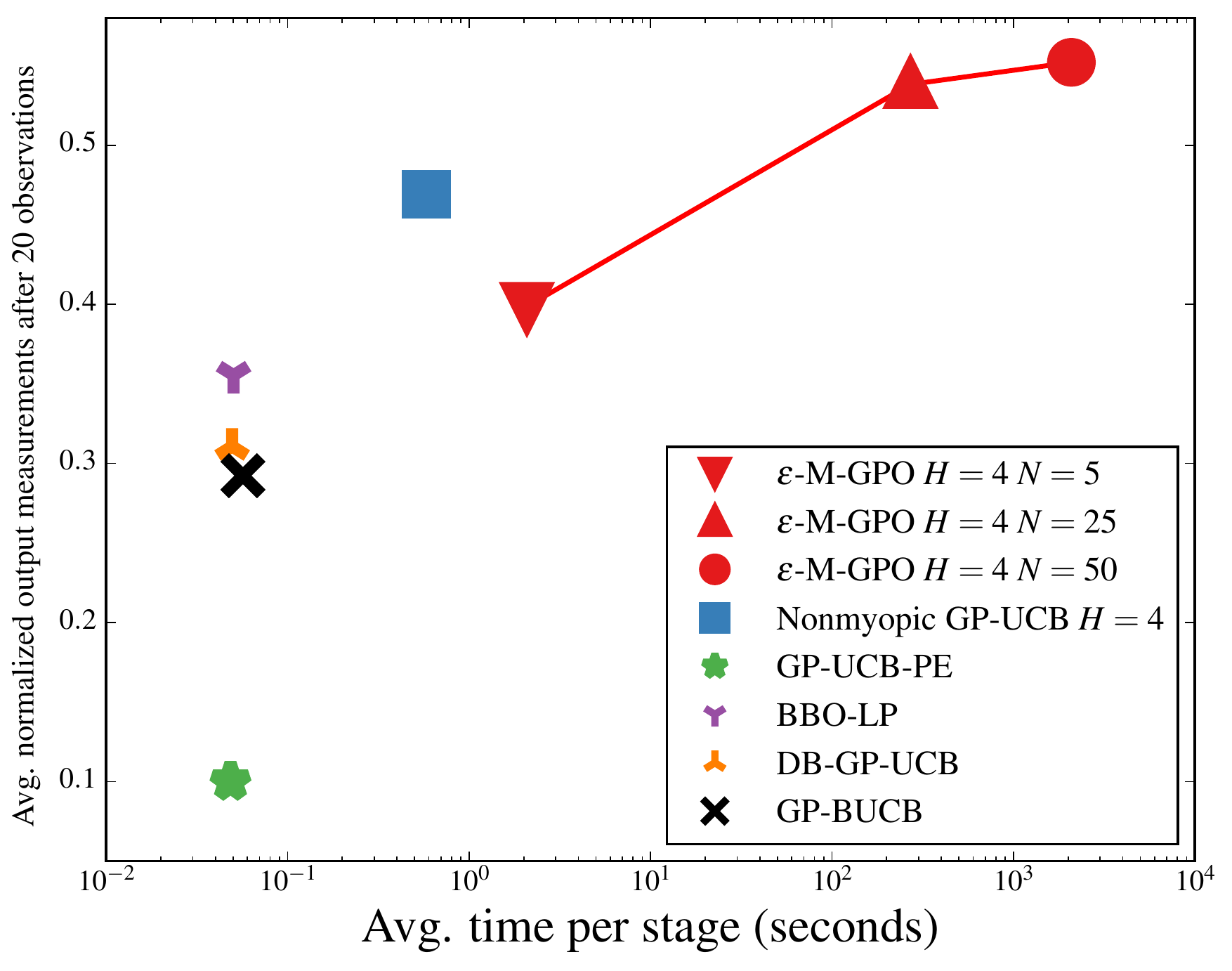}\hspace{-1mm} &
		\hspace{-4mm} \includegraphics[width=0.45 \textwidth]{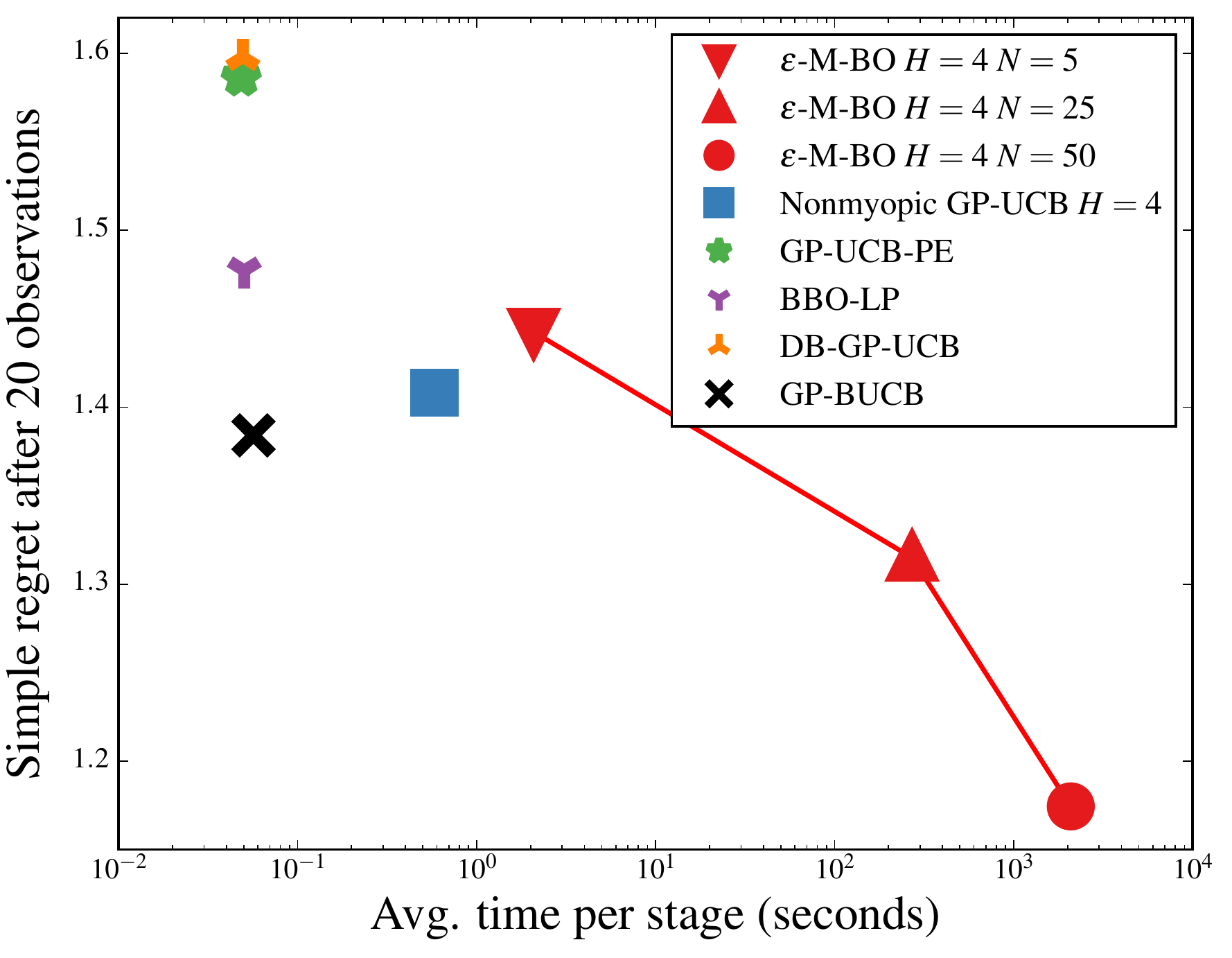}\\
		\hspace{-4mm}{\small (a)}\hspace{-1mm} & {\small (b)}\hspace{-1mm} 
	\end{tabular}
	\caption{Graphs of  (a) average normalized\cref{footnote-rewards} output measurements observed by AUV and (b) simple regrets achieved by tested BO algorithms vs. average time per stage for simulated plankton density phenomena.}
	\label{fig:runtime}
\end{figure*}
\section{Experiments and Discussion}
\label{expt}
This section empirically evaluates the performance of our nonmyopic adaptive $\epsilon$-Macro-GPO policy and its anytime variant for a given finite budget with three datasets featuring simulated plankton density phenomena~\cite{pennington16}, 
a real-world traffic phenomenon~\cite{TASE15}, 
and
a real-world temperature phenomenon over an office environment~\cite{choi12}.
The performances of our $\epsilon$-Macro-GPO policy and its anytime variant are compared with that of state-of-the-art (a) nonmyopic GP-UCB~\cite{marchant14} generalized to handle macro-actions that coincides with our deterministic policy~\eqref{ml-policy} exploiting the most likely observations during planning,
(b) \emph{distributed batch GP-UCB} (DB-GP-UCB)~\cite{daxberger17} that 
casts a macro-action as a batch to be optimized and
is thus equivalent to $\epsilon$-Macro-GPO with $H=1$, 
(c) $q$-EI \cite{chevalier13} that does likewise,
and (d) greedy batch BO algorithms\footnote{\label{footnote1}Unlike DB-GP-UCB and $q$-EI, a greedy batch BO algorithm cannot exploit the full informativeness of any candidate macro-action for its macro-action selection: Since it selects the inputs of a batch one at a time myopically\cref{boo}, its first few selected input locations immediately decide its chosen macro-action and consequently the remaining sequence of input locations found within.}
such as GP-BUCB~\cite{desautels14}, GP-UCB-PE~\cite{contal13}, and BBO-LP~\cite{gonzalez16b} 
whose implementations are detailed in Appendix~\ref{used-algs}.\footnote{It is not obvious to us how GLASSES~\cite{gonzalez16a} and Rollout~\cite{lam16} can be modified to handle macro-actions and  are thus not empirically compared here. However, since Rollout~\cite{lam16} also exploits Bellman equations, it is compared with our $\epsilon$-Macro-GPO by setting macro-action length to $\kappa=1$ (i.e., primitive action).} 

Four performance metrics are used: (a) average 
normalized\footnote{\label{footnote-rewards}To ease interpretation of results, the prior mean is subtracted from each output measurement to normalize it.} 
output measurements observed by the agent (larger average output measurements imply less average/cumulative regret (Section~\ref{main-section})),
(b) simple regret (i.e., difference between global maximum and currently found maximum),   (c) no. of explored nodes in all constructed search trees (more nodes incur more time), and (d) average runtime per stage.
\subsection{Simulated plankton density phenomena}
An \emph{autonomous underwater vehicle} (AUV) is deployed on board of a \emph{research vessel} (RV) in search for a hotspot of peak phytoplankton abundance (i.e., algal bloom) in coastal ocean. 
The AUV and RV are initially positioned near the center of the plankton density (mg/m$^{3}$) phenomenon spatially distributed over a $5$ km by $5$ km region that is discretized into a $50 \times 50$ grid of input locations. 
The phenomenon is modeled as a realization of a GP and simulated using the GP hyperparameters $\mu_s  = 0$, $\ell_1 = \ell_2 = 0.5$~km,  $\sigma_{y}^2 = 1$, and $\sigma_n^2 = 10^{-5}$. 
The AUV is tasked to execute the selected macro-action of a straight dive (due to limited maneuverability) along one of the $4$ cardinal directions from the RV to gather ``Gulper'' water samples/observations over $\kappa=4$ input locations for precise on-deck testing~\cite{pennington16}; given a budget of $20$ observations, this will be repeated for $5$ times (i.e. $5$ stages) 
from the input location that it has previously surfaced.
\begin{table}[h]
	\caption{No. of explored nodes by $\epsilon$-Macro-GPO (when $H=1$, it corresponds to DB-GP-UCB)  for simulated plankton density phenomena.}
	\centering
	\begin{tabular}{cccc}
		\hline
		$H = 1$& $H = 2$ & $H = 3$ & $H = 4$\\
		\hline
		$2.50 \times 10$ & $8.01 \times 10^3$ & $2.40 \times 10^6$ & 	$6.41 \times 10^8$\\
		\hline
	\end{tabular}		
	\label{table:simulated}
	\vspace{0mm}
\end{table}
Figs.~\ref{fig:simulated_total}a and~\ref{fig:simulated_total}b show results of the performances of $\epsilon$-Macro-GPO with $H =2,3,4$ (lookahead of, respectively, $8$, $12$, $16$ observations), $\beta = 0$, and $N=100$,\footnote{\label{crawfish}Specifying the value of $N$ (instead of $\epsilon$) may yield a loose $\epsilon$ based on Theorem~\ref{th:expected}. Nevertheless, the resulting $\epsilon$-Macro-GPO with $H =3,$$4$ empirically outperforms other tested BO algorithms.} and the other tested BO algorithms 
averaged over $250$ independent realizations of the simulated phenomena. 	
It can be observed that as the number of observations increases, the nonmyopic adaptive BO algorithms generally outperform the myopic ones.
In particular, the performance of $\epsilon$-Macro-GPO  improves considerably by increasing $H$: $\epsilon$-Macro-GPO with the furthest lookahead (i.e., $H = 4$) achieves the largest average  normalized output measurements observed by the AUV and smallest simple regret after $20$ observations at the cost of a larger number of explored nodes (see Table~\ref{table:simulated}).
For example, the nonmyopic $\epsilon$-Macro-GPO with $H = 4$ achieves $0.093\sigma_y$ ($0.059\sigma_y$) more average output measurements and $0.211 \sigma_y$ ($0.148\sigma_y$) less simple regret than myopic DB-GP-UCB (nonmyopic GP-UCB with the same horizon $H=4$ but assuming most likely observations during planning), which are expected.

Figs.~\ref{fig:simulated_total}c and~\ref{fig:simulated_total}d show the effect of varying exploration weights 	
$\beta$ on the performance of $\epsilon$-Macro-GPO with $H = 2$ and $H = 3$, respectively. 
It can be observed from
Fig.~\ref{fig:simulated_total}c that when $H = 2$, $\epsilon$-Macro-GPO with $\beta=0.1$ achieves $0.064\sigma_y$ more average normalized output measurements than that with $\beta = 0$ after $20$ observations,
which indicates the need of a slightly stronger exploration behavior. 
Fig.~\ref{fig:simulated_total}d shows that by increasing to a lookahead of $12$ observations (i.e., $H = 3$), $\epsilon$-Macro-GPO no longer needs the additional weighted exploration term in~\eqref{eq:reward-def} (i.e., $\beta=0$) since it can naturally trade off between exploration vs. exploitation, as explained previously (Section~\ref{main-section}).
It can also be observed from Figs.~\ref{fig:simulated_total}c and~\ref{fig:simulated_total}d that $\beta = 10$ greatly hurts its performance due to an overly aggressive exploration.

We also investigate the effect of varying the number $N$ of stochastic samples on the behavior of  $\epsilon$-Macro-GPO. To this end, $\epsilon$-Macro-GPO with a fixed horizon $H$ offers an advantage of being able to trade off its performance for time efficiency by decreasing  $N$. This observation is theoretically validated in Theorem~\ref{th:expected} and empirically illustrated in Fig.~\ref{fig:runtime}.

Figs.~\ref{fig:runtime}a and~\ref{fig:runtime}b show results of the performances of $\epsilon$-Macro-GPO with $H =4$ (lookahead of $16$ observations), $\beta = 0$, and $N=5,25,50$, and the other tested BO algorithms 
averaged over $35$ independent realizations of the simulated plankton density phenomena. 
It can be observed that the performance of $\epsilon$-Macro-GPO  improves considerably by increasing $N$: $\epsilon$-Macro-GPO with the largest number of samples (i.e., $N = 50$) achieves the largest average  normalized output measurements and smallest simple regret after $20$ observations  at the cost of larger average time per iteration.
For example,  $\epsilon$-Macro-GPO with  $N=50$ achieves $0.26\sigma_y$ more average output measurements and $0.21 \sigma_y$ less simple regret than myopic GP-BUCB, but needs $2085.37$ more seconds per iteration.
\begin{figure*}
	\centering
	{\begin{tabular}{cccc}
			\hspace{-1mm}\includegraphics[height=6.1cm]{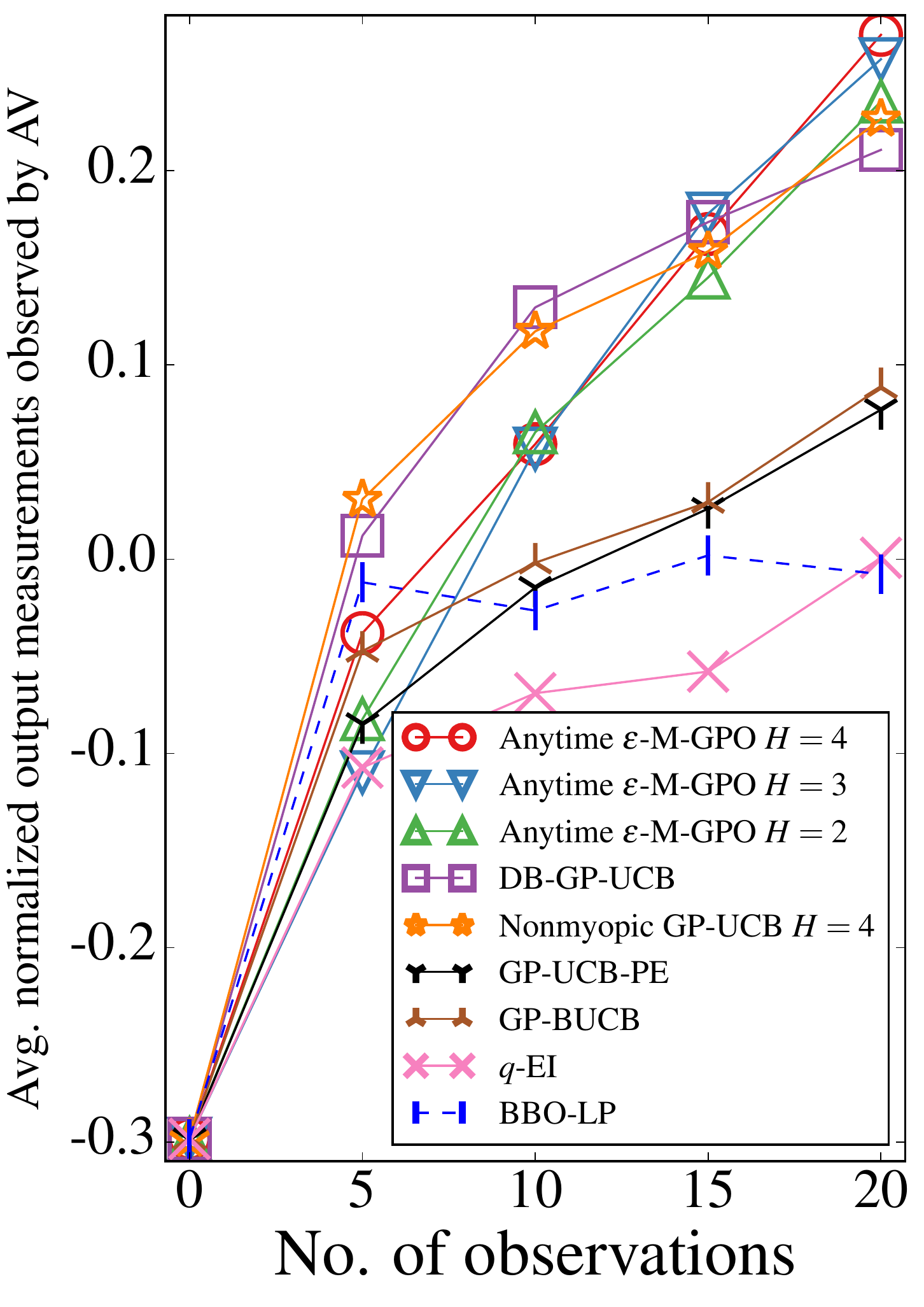}\hspace{-1mm} &
			\includegraphics[height=6.1cm]{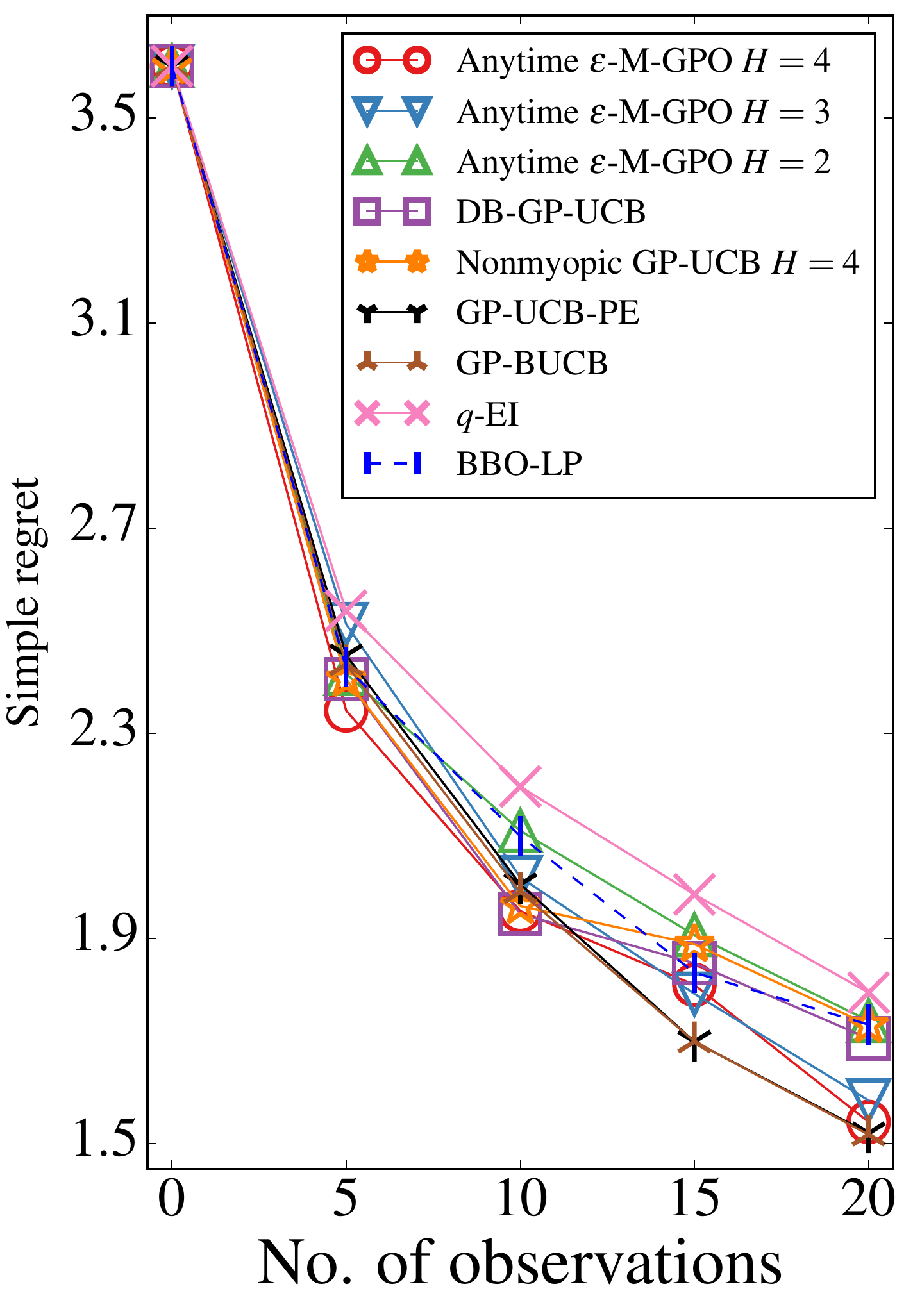}\hspace{-1mm} &
			\includegraphics[height=6.1cm]{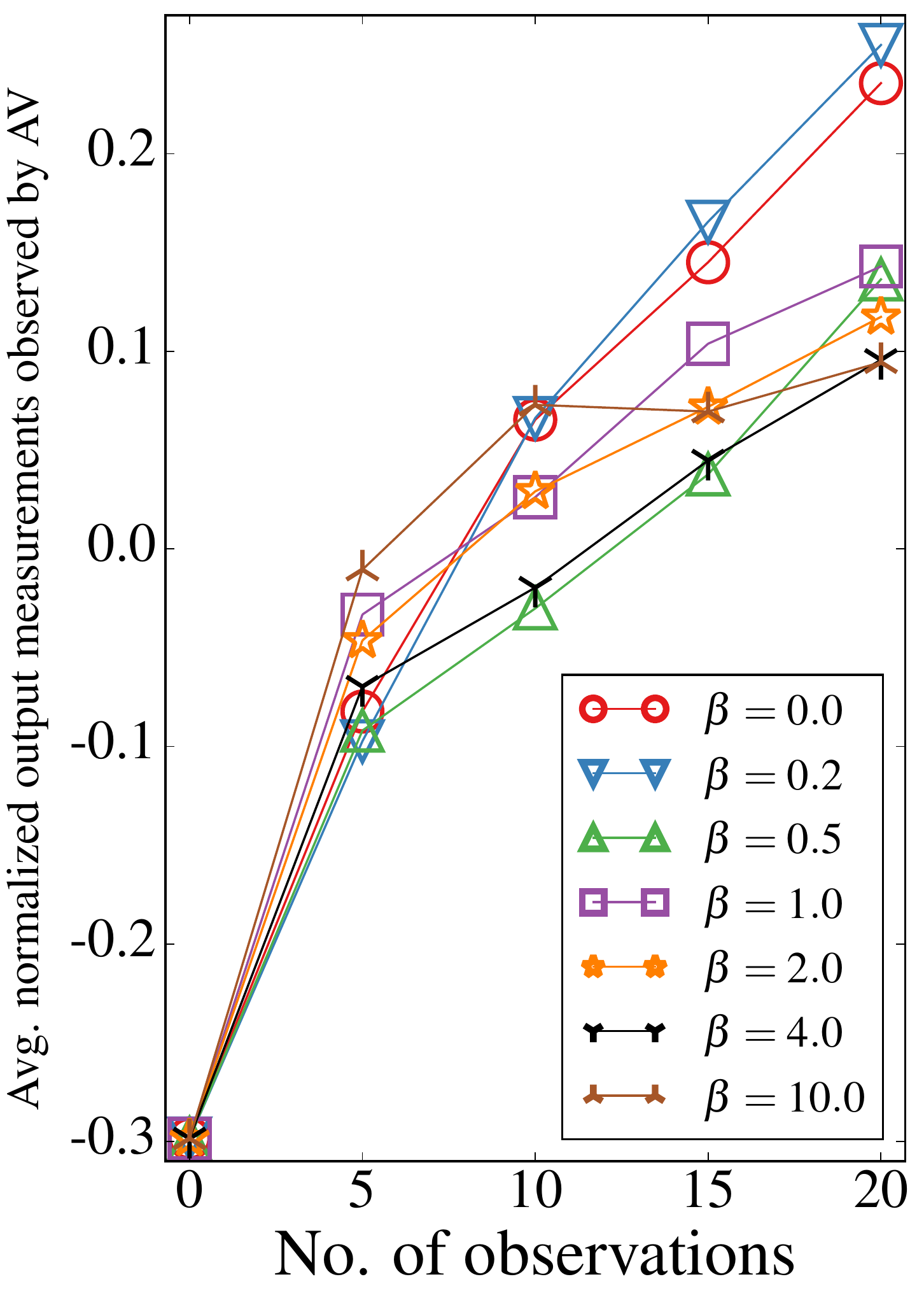}\hspace{-1mm} &
			\includegraphics[height=6.1cm]{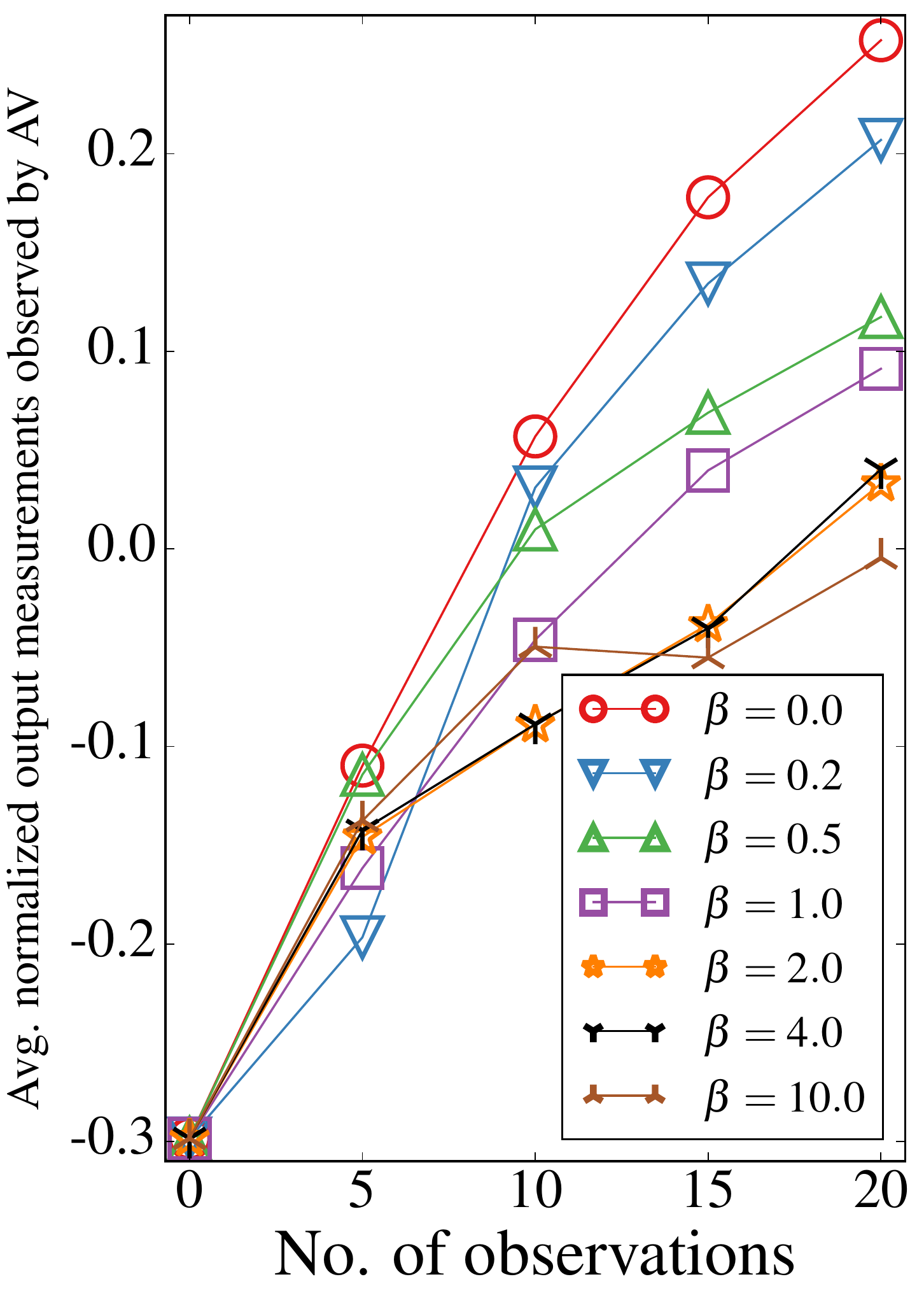}\\
			\hspace{-1mm}{\small (a)}\hspace{-1mm} & {\small (b)}\hspace{-1mm} & {\small (c)}\hspace{-1mm} & {\small (d)}
	\end{tabular}}
	\caption{Graphs of (a) average normalized\cref{footnote-rewards} output measurements observed by the AV and (b) simple regrets achieved by the tested BO algorithms, and  average normalized output measurements achieved by \emph{anytime} $\epsilon$-Macro-GPO with (c) $H=2$ and (d) $H = 3$ and varying exploration weights $\beta$ vs. no. of observations for  real-world traffic phenomenon. The standard errors are given in Tables~\ref{table:road-var-methods} and~\ref{table::road-var-beta} in Appendix~\ref{sec:road-add}.}
	\label{fig:roads_total}
\end{figure*}	
\subsection{Real-world traffic phenomenon} 
To service the mobility demands within the central business district of an urban city, an \emph{autonomous vehicle} (AV) in a mobility-on-demand system cruises along different road trajectories  to find a hotspot of highest mobility demand to pick up a user.
The $29.4$~km by~$11.9$~km service area is gridded into $100 \times 50$ input regions, of which only $2506$ input regions are accessible to the AV via the road network. 
The AV can cruise from input region $s$ to an adjacent input region $s'$ 
using one primitive action 
iff at least one road segment in the road network starts in $s$ and ends in $s'$;
the maximum outdegree from any input region is $8$.  
In any input region, a surrogate demand measurement is obtained by counting the number of pickups\footnote{A distributed gossip-based protocol can be used to aggregate these pickup information from the AVs in the input region that are connected via an ad hoc wireless communication network~\cite{TASE15}. Any AV entering the input region can then access its pickup count by joining its ad hoc network.} from all historic taxi trajectories generated by a major taxi company during $9$:$30$-$10$~p.m.~on August $2$, $2010$~\cite{TASE15}; the resulting mobility demand pattern is visualized in Fig.~\ref{domdom}. 
\begin{figure}[h]
	\centering
	\includegraphics[width=0.49 \textwidth]{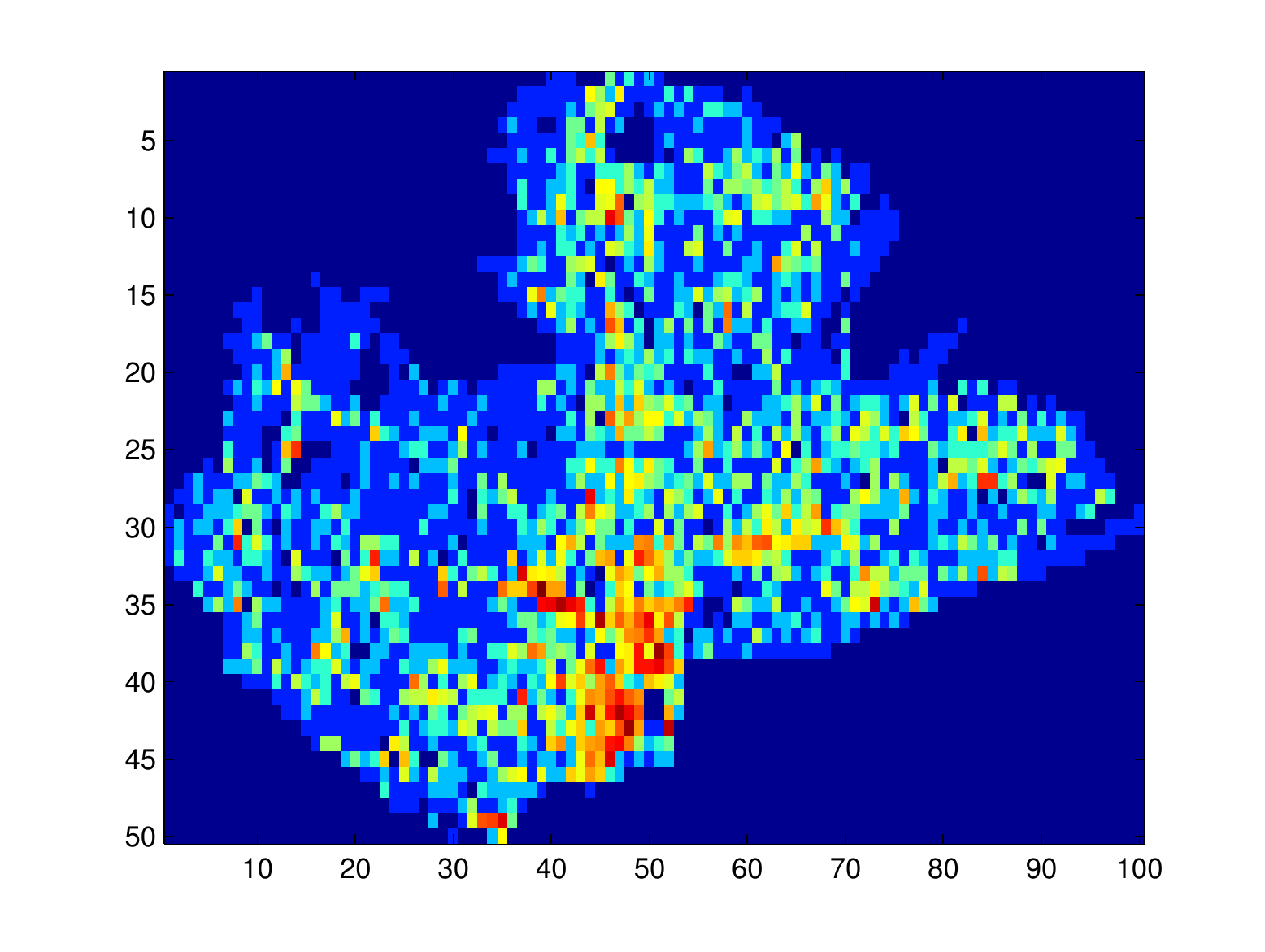}
	\caption{Mobility demand pattern spatially distributed over the central business district of an urban city during $9$:$30$-$10$~p.m.~on August $2$, $2010$: ``Hotter'' regions
		indicate larger numbers of pickups (Image courtesy of~\cite{TASE15}).}
	\label{domdom}
\end{figure}
The original demand measurements are log-transformed to remove skewness and extremity for stabilizing the GP covariance structure and
the GP hyperparameters $\mu_s  =  1.5673$, 
$\ell_1 = 0.1689$~km,  $\ell_2 = 0.1275$~km,  
$\sigma_y^2 = 0.7486$, and $\sigma_n^2 = 0.0111$ are then learned using  maximum likelihood estimation~\cite{gpml}; 
note that the length-scales and signal-to-noise ratio are relatively smaller than that of the simulated plankton density phenomena.
The AV is tasked to execute the selected macro-action of a cruising trajectory along $\kappa = 5$ adjacent input regions to observe their corresponding demand measurements; given a budget of $20$ observations, this will be repeated for $4$ times 
(i.e. $4$ stages) 
from the input region that it has previously cruised to. 
Since every input region $s$ has a large number of available macro-actions (i.e., with an average of $178$ and maximum of $1193$ macro-actions), $20$ of them are randomly\footnote{\label{fret}The BO performance of $\epsilon$-Macro-GPO and its anytime variant can be potentially improved by using macro-action generation algorithms~\cite{he11} instead of random selection.} selected 
to form its representative set of candidate macro-actions.
\begin{table}[h]
	\caption{No. of explored nodes by anytime $\epsilon$-Macro-GPO (when $H=1$, it corresponds to DB-GP-UCB) for the real-world traffic phenomenon (i.e., mobility demand pattern).}
	\centering
	\begin{tabular}{cccc}
		\hline
		$H = 1$& $H = 2$ & $H = 3$ & $H = 4$\\
		\hline
		$8.29 \times 10$ & $9.52 \times 10^4$ & $1.29 \times 10^6$ & 	$1.34 \times 10^7$\\
		\hline
	\end{tabular}		
	\label{table:roads}
\end{table}

Figs.~\ref{fig:roads_total}a and~\ref{fig:roads_total}b show results of the performances of \emph{anytime} $\epsilon$-Macro-GPO with $H =2,3,4$ (a lookahead of, respectively, $10$, $15$, $20$ observations), $\beta = 0$, and $N=300$ 
after running for $1500$ iterations\cref{crawfish}, and the other tested BO algorithms averaged over $35$ random starting input regions of the AV.		
Similar to the results for simulated plankton density phenomena, it can be observed that the performance of anytime $\epsilon$-Macro-GPO  improves considerably by increasing $H$: Anytime $\epsilon$-Macro-GPO with the furthest lookahead (i.e., $H = 4$) achieves the largest average normalized output measurements observed by the AV and among the least simple regret after $20$ observations at the cost of a larger number of explored nodes (see Table~\ref{table:roads}).
For example, the nonmyopic anytime $\epsilon$-Macro-GPO with $H = 4$ achieves $0.069 \sigma_{y}$ ($0.05 \sigma_{y}$) more average output measurements and $0.188 \sigma_y$ ($0.219\sigma_y$) less simple regret than myopic DB-GP-UCB (nonmyopic GP-UCB with $H=4$), which are expected.
Interestingly, GP-BUCB and GP-UCB-PE can achieve simple regret comparable to that of anytime $\epsilon$-Macro-GPO with $H = 4$ even though they perform very poorly in terms of average output measurements.

Figs.~\ref{fig:roads_total}c and~\ref{fig:roads_total}d show the effect of varying exploration weights 	
$\beta$ on the performance of anytime $\epsilon$-Macro-GPO with $H = 2$ and $H = 3$, respectively. 
It can be observed from Fig.~\ref{fig:roads_total}c that when $H = 2$, anytime $\epsilon$-Macro-GPO with $\beta=0.2$ achieves $0.022 \sigma_{y}$ more average normalized output measurements than that with $\beta = 0$ after $20$ observations,
which indicates the need of a slightly stronger exploration behavior. Fig.~\ref{fig:roads_total}d shows that by increasing to a lookahead of $15$  observations(i.e., $H = 3$), anytime $\epsilon$-Macro-GPO no longer needs the additional weighted exploration term in~\eqref{eq:reward-def} (i.e., $\beta=0$) since it can naturally trade off between exploration vs. exploitation, as explained previously (Section~\ref{main-section}).		
It can also be observed from Figs.~\ref{fig:roads_total}c and~\ref{fig:roads_total}d that $\beta \geq 0.5$ hurts its performance due to overly aggressive exploration.	
\begin{figure}
	\centering
	{\begin{tabular}{cc}
				\hspace{-4mm} \includegraphics[width=0.23 \textwidth]{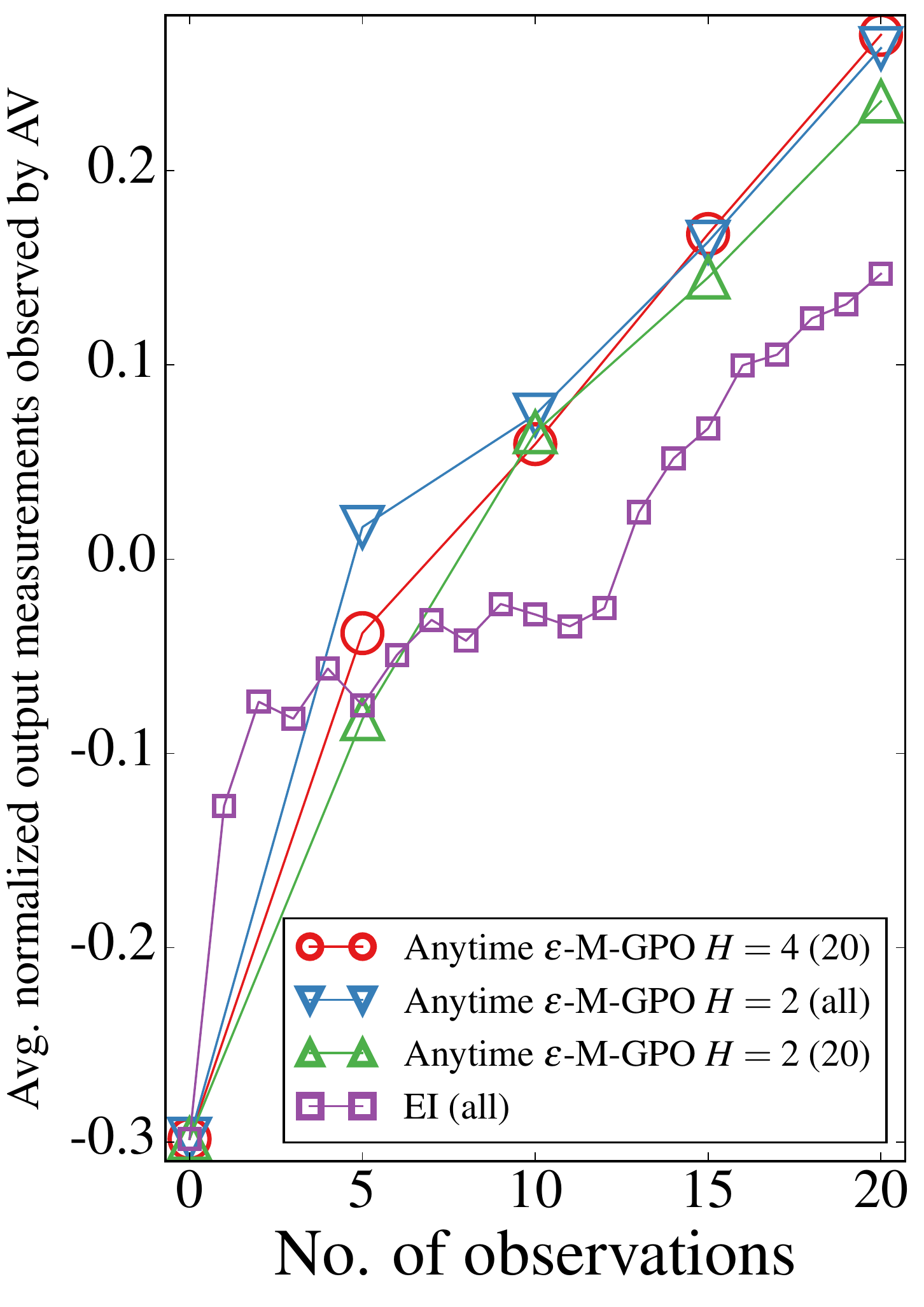}\hspace{-1mm} &
				\hspace{-4mm} \includegraphics[width=0.23 \textwidth]{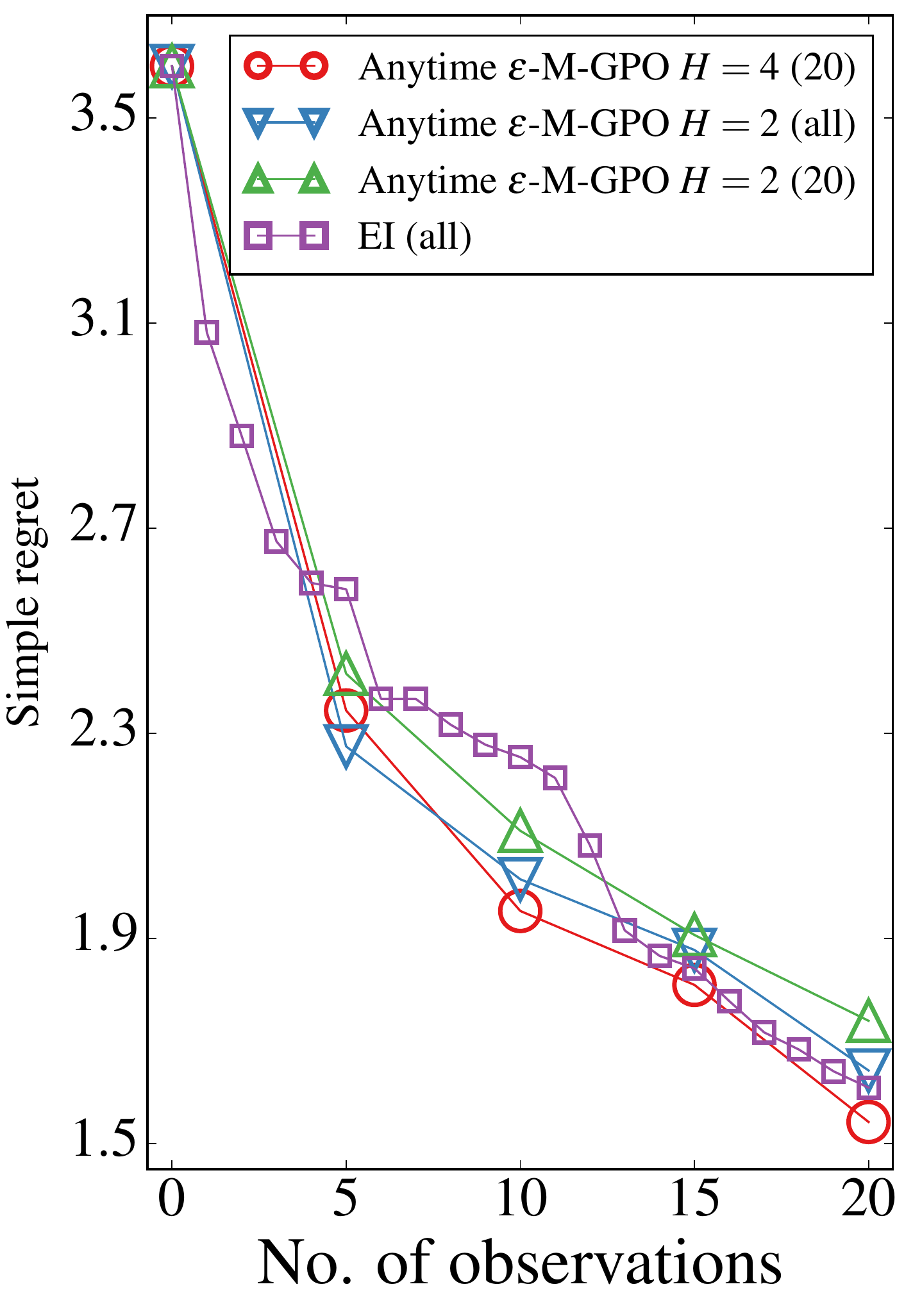}\\
			\hspace{-4mm}{\small (a)}\hspace{-1mm} & {\small (b)}\hspace{-1mm} 
	\end{tabular}}
	\caption{Graphs of  (a) average normalized output measurements observed by the AV and (b) simple regrets achieved by \emph{anytime} $\epsilon$-Macro-GPO with $H=2, 4$ and $20$ randomly selected macro-actions per input region, anytime $\epsilon$-Macro-GPO with $H=2$ and all available macro-actions (the no. of available macro-actions per input region is enclosed in brackets), and EI with all available macro-actions of length $1$ vs. no. of observations for real-world traffic phenomenon. Standard errors are given in Table~\ref{table:road-var-h2} in Appendix~\ref{sec:road-add}.}
	\label{fig:roads_h2}
\end{figure}

Lastly, we investigate the effect of downsampling the number of available macro-actions per input region to $20$ on the performance of anytime $\epsilon$-Macro-GPO.		
To do this, the performances of anytime $\epsilon$-Macro-GPO with $H=2,4$ and $20$ randomly selected macro-actions per input region are compared with that of anytime $\epsilon$-Macro-GPO with $H=2$ and all available macro-actions as well as myopic EI~\cite{shahriari16} with all available macro-actions of length $1$.
It can be observed from Figs.~\ref{fig:roads_h2}a and~\ref{fig:roads_h2}b that when $H=2$, downsampling the number of available macro-actions per input region to $20$ decreases average normalized output measurements by $0.032 \sigma_{y}$ and increases simple regret by $0.112 \sigma_y$ after $20$ observations, 	
but also reduces the number of explored nodes by more than $1$ order of magnitude (see Table~\ref{table:roads_h2}). 
By increasing to a lookahead of $20$ observations, anytime $\epsilon$-Macro-GPO with $H=4$ and $20$ randomly selected macro-actions per input region achieves $0.008\sigma_y$ more average normalized output measurements and $0.116 \sigma_y$ less simple regret than that with $H=2$ and all available macro-actions at the cost of a larger number of explored nodes.
Though EI can access all available macro-actions of length $1$ (i.e, no restriction on action space of AV), it obtains much less average normalized output measurements and more simple regret than anytime $\epsilon$-Macro-GPO with $H=4$ and $20$ randomly selected macro-actions per input region due to its myopia.
\begin{table}[h]
	\caption{No. of explored nodes by anytime $\epsilon$-Macro-GPO (the no. of available macro-actions per input region is enclosed in brackets) for the real-world traffic phenomenon (i.e., mobility demand pattern).}
	\centering
	\begin{tabular}{cccc}
		\hline
		$H = 2\ (20)$ & $H = 2\ (\text{all})$ & $H = 4\ (20)$\\
		\hline
		$0.95 \times 10^5$ & $1.26 \times 10^6$ & $1.34 \times 10^7$\\
		\hline
	\end{tabular}		
	\label{table:roads_h2}
\end{table}
\begin{figure*}[h]
	\centering
	{\begin{tabular}{cccc}
			\hspace{-1mm}\includegraphics[width=0.23\textwidth]{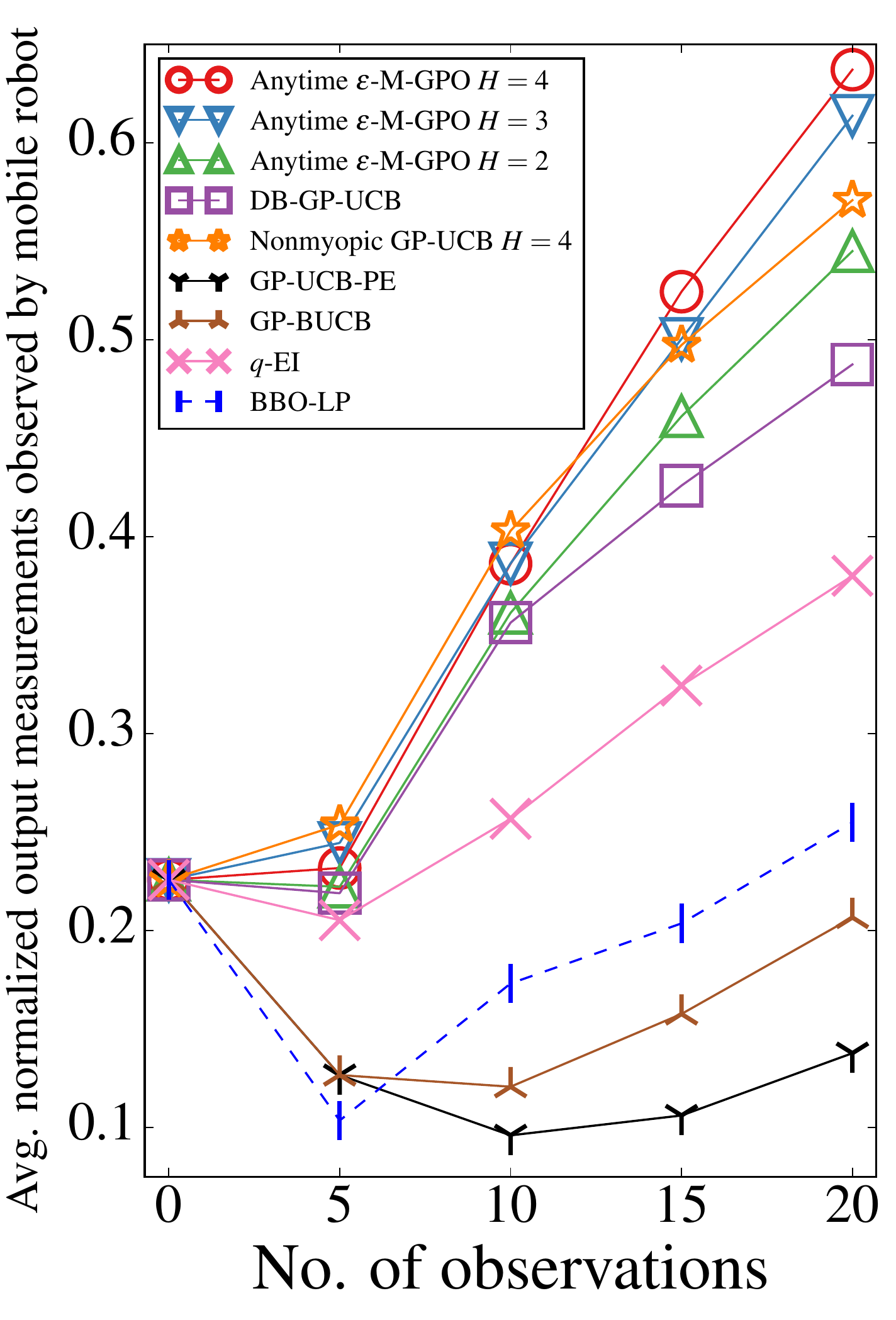} &
			\includegraphics[width=0.23\textwidth]{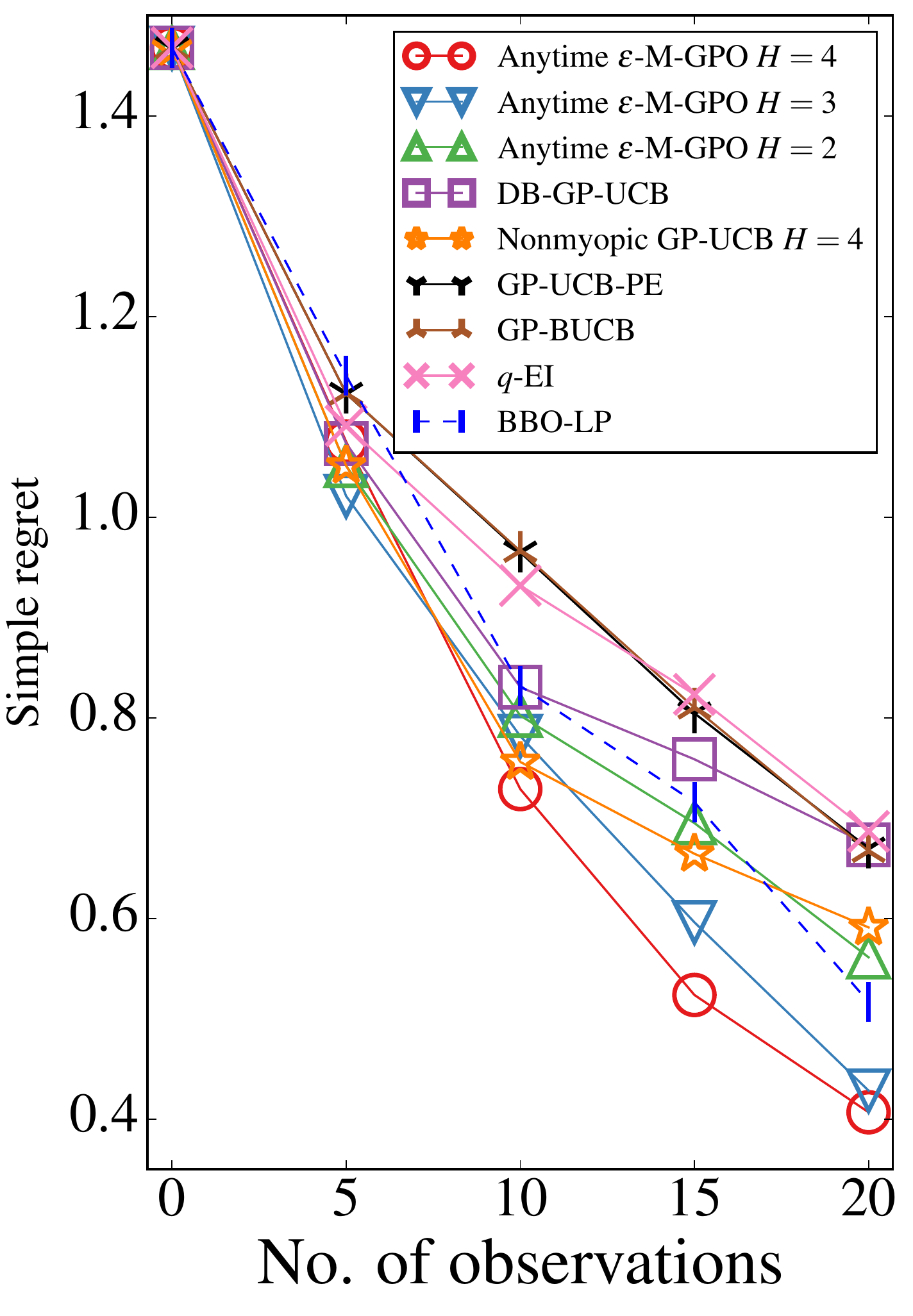} &
			\includegraphics[width=0.23\textwidth]{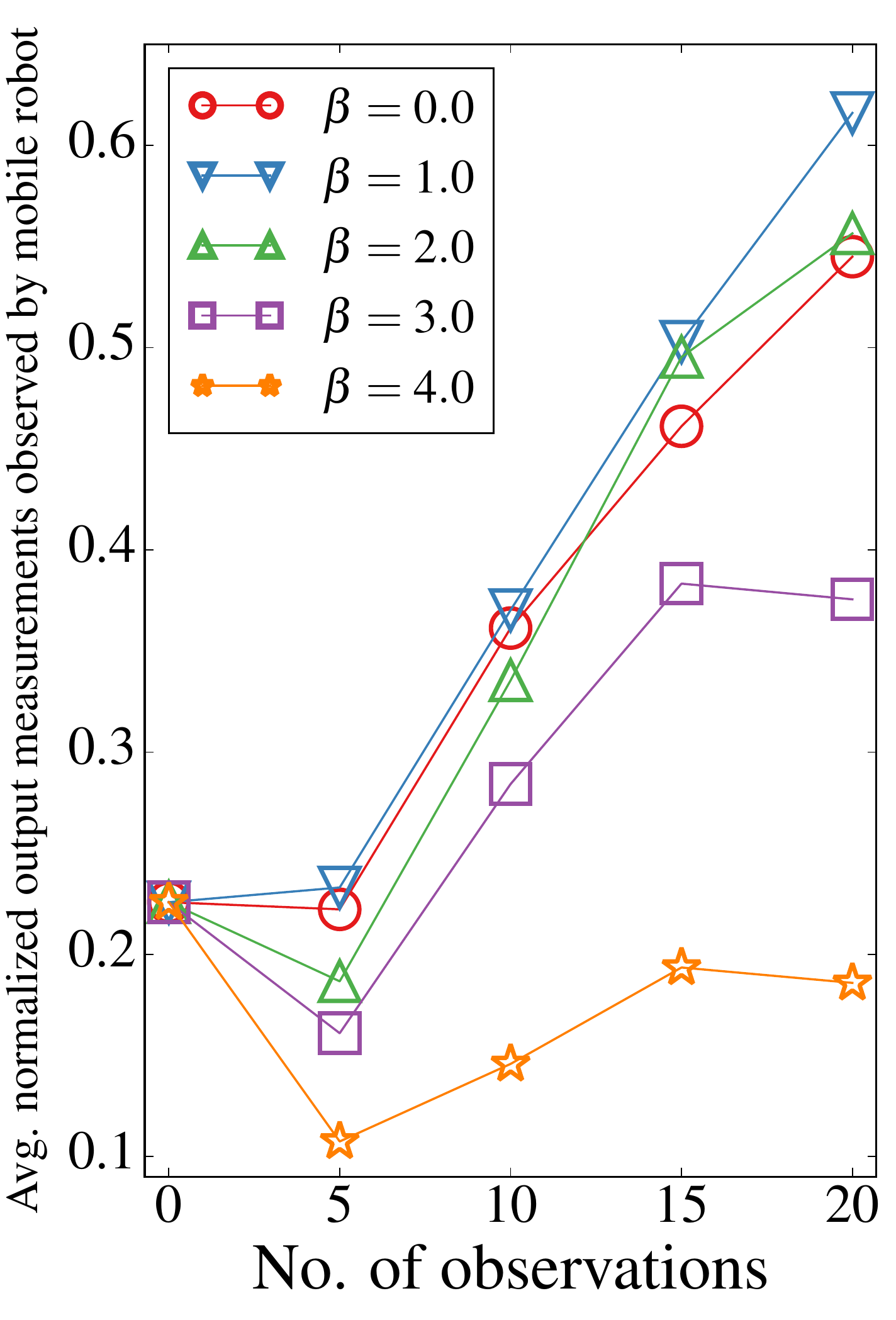} &
			\includegraphics[width=0.23\textwidth]{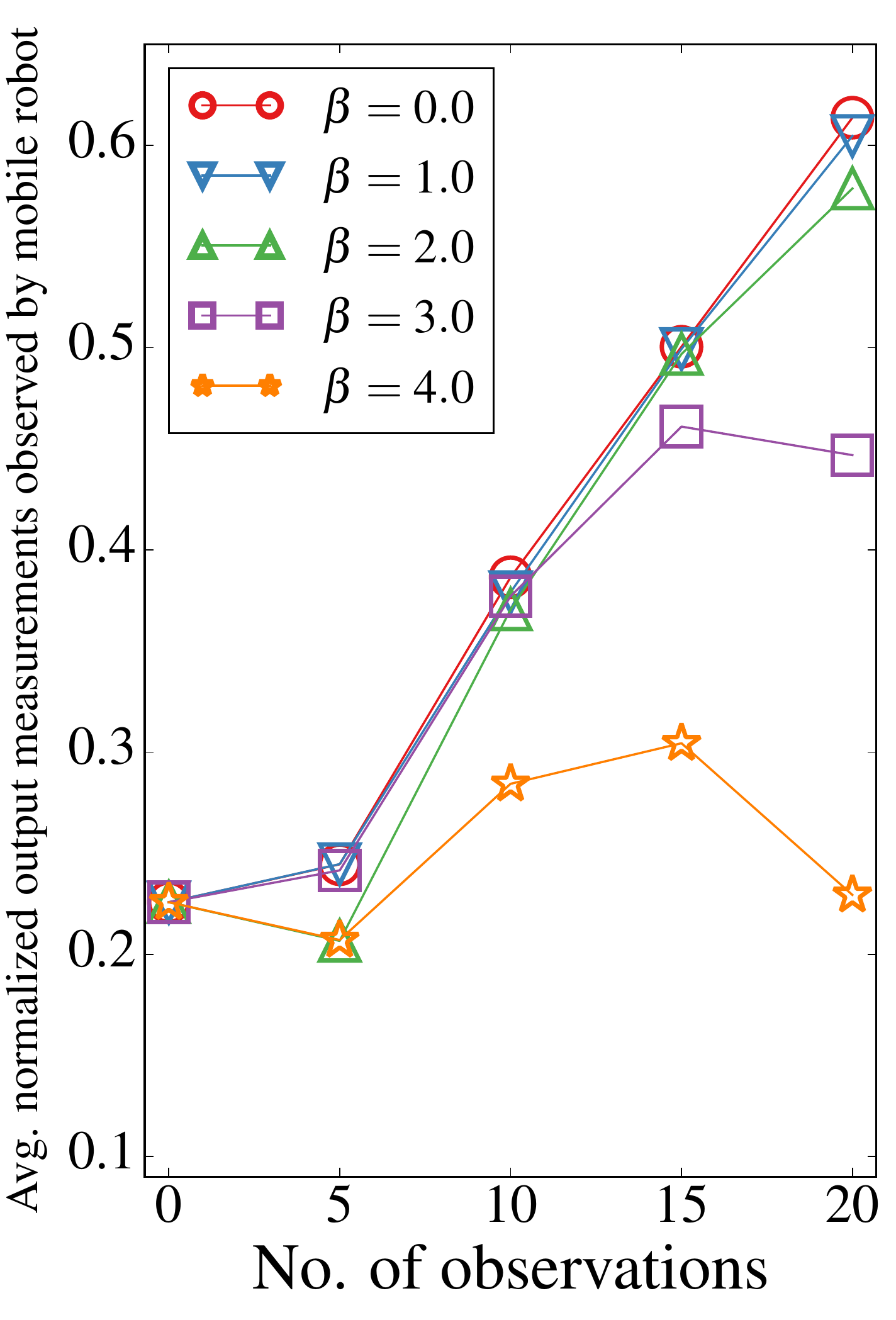}\\
			\hspace{-1mm}{(a)} & {(b)} & {(c)} & {(d)}
	\end{tabular}}
	\caption{Graphs of (a) average normalized\cref{footnote-rewards} output measurements observed by the mobile robot and (b) simple regrets achieved by the tested BO algorithms vs. no. of observations, and  average normalized output measurements achieved by \emph{anytime} $\epsilon$-Macro-GPO with (c) $H=2$ and (d) $H = 3$ and varying exploration weights $\beta$ vs. no. of observations for the real-world temperature phenomenon over the Intel Berkeley Research Lab. The standard errors are given in Tables~\ref{table:robot-var-methods} and~\ref{table::robot-var-beta} in Appendix~\ref{sec:exp-robot}.}
	\label{fig:robot_total}
\end{figure*}	
\subsection{Real-world temperature phenomenon} 
In monitoring of the indoor environmental quality of an office environment~\cite{choi12}, a mobile robot mounted with a weather board is tasked to find a hotspot of peak temperature by exploring different stretches of corridors that can be naturally abstracted into macro-actions. 
The temperature ($^{\circ}$C) phenomenon is spatially distributed over the Intel Berkeley Research Lab (of about $41$~m by $32$~m in size) with $41$ deployed temperature sensors (see Fig.~\ref{fig:dataset})
and modeled as a realization of a GP.
Using the observations/data gathered by the $41$ temperature sensors\footnote{\url{http://db.csail.mit.edu/labdata/labdata.html}},
the GP hyperparameters $\mu_s  =  17.8513$, $\ell_1 = 4.0058$~m, $\ell_2 = 11.3811$~m, $\sigma_y^2 = 0.5964$, and $\sigma_n^2 = 0.0597$ are learned using maximum likelihood estimation~\cite{gpml}.
Then, using these learned hyperparameters and the observations/data gathered by the $41$ temperature sensors, 
we exploit the GP posterior mean~\eqref{gp-posteriors} to predict the temperature measurements at the $104$ input locations shown in Fig.~\ref{fig:dataset}; these predictions together with the data obtained from the $41$ sensors serve as the dataset for the experiment here.
The mobile robot is tasked to execute the selected macro-action of a motion path along a stretch of $\kappa=5$ input locations on one of the corridors in the lab to observe their corresponding temperature measurements;
given a budget of $20$ observations, this will be repeated for $4$ times 
from the input location that it has previously moved to.
Since every input location $s$ has a large number of available macro-actions (i.e., with an average of $27$ and maximum of $114$ macro-actions), $20$ of them are randomly\cref{fret} selected to form its representative set of candidate macro-actions.%
\begin{figure}[h]
	\centering
	\includegraphics[height=5.5cm]{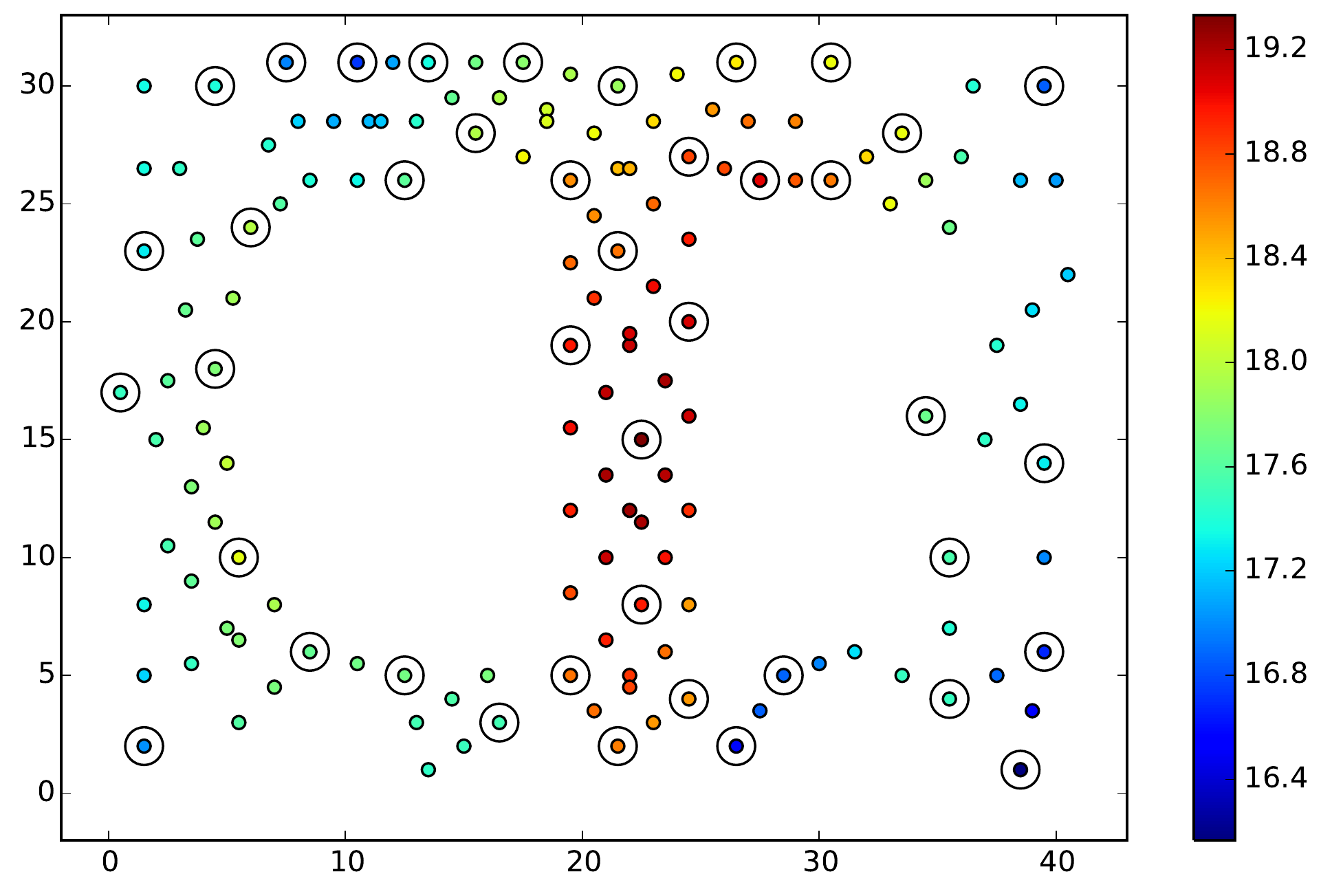}
	\caption{The temperature measurements at the $104$ input locations (not circled) in the Intel Berkeley Research lab are predicted using the GP posterior mean~\eqref{gp-posteriors} based on the data gathered by the $41$ temperature sensors (circled);
		these predictions together with the data obtained from the $41$ sensors serve as the dataset for the experiment here.}
	\label{fig:dataset}
\end{figure}

Figs.~\ref{fig:robot_total}a and~\ref{fig:robot_total}b show results of the performances of \emph{anytime} $\epsilon$-Macro-GPO with $H =2,3,4$ (lookahead of, respectively, $10$, $15$, $20$ observations), $\beta = 0$, and $N=300$ 
after running for $1500$ iterations\cref{crawfish}, and the other tested BO algorithms averaged over $35$ random initial starting input locations of the mobile robot. %
Similar to the results for simulated plankton density phenomena and real-world traffic phenomenon, it can be observed that as the number of observations increases, the nonmyopic adaptive BO algorithms generally outperform the myopic ones.
In particular, the performance of anytime $\epsilon$-Macro-GPO  improves considerably by increasing $H$ such that anytime $\epsilon$-Macro-GPO with the furthest lookahead (i.e., $H = 4$) achieves the largest average normalized output measurements observed by the mobile robot and smallest simple regret after $20$ observations at the cost of a larger number of explored nodes (see Table~\ref{table:robot}).
For example, the nonmyopic anytime $\epsilon$-Macro-GPO with $H = 4$ achieves $0.194 \sigma_{y}$ ($0.086 \sigma_{y}$) more average normalized output measurements and $0.345 \sigma_y$ ($0.239\sigma_y$) less simple regret than the myopic DB-GP-UCB (nonmyopic GP-UCB with the same horizon $H=4$ but assuming most likely observations during planning), which are expected.
\begin{table}[h]
	\caption{No. of explored nodes by anytime $\epsilon$-Macro-GPO (when $H=1$, it corresponds to DB-GP-UCB)  for the real-world temperature phenomenon over the Intel  Berkeley Research Lab.}
	\centering
	\begin{tabular}{cccc}
		\hline
		$H = 1$& $H = 2$ & $H = 3$ & $H = 4$\\
		\hline
		$7.51 \times 10$ & $8.88 \times 10^4$ & $1.13 \times 10^6$ & 	$1.12 \times 10^7$\\
		\hline
	\end{tabular}		
	\label{table:robot}
\end{table}

Figs.~\ref{fig:robot_total}c and~\ref{fig:robot_total}d show the effect of varying exploration weights 	
$\beta$ on the performance of anytime $\epsilon$-Macro-GPO with $H = 2$ and $H = 3$, respectively. It can be observed from Fig.~\ref{fig:robot_total}c that 
when $H = 2$, anytime $\epsilon$-Macro-GPO with $\beta=1$ achieves $0.092 \sigma_{y}$ more average normalized output measurements than that with $\beta = 0$ after $20$ observations,
which indicates the need of a slightly stronger exploration behavior. 
Fig.~\ref{fig:robot_total}d shows that by increasing to a lookahead of $15$ observations (i.e., $H = 3$), anytime $\epsilon$-Macro-GPO no longer needs the additional weighted exploration term in~\eqref{eq:reward-def} (i.e., $\beta=0$) since it can naturally trade off between exploration vs. exploitation, as explained previously (Section~\ref{main-section}).	
It can also be observed from Figs.~\ref{fig:robot_total}c and~\ref{fig:robot_total}d that $\beta\geq 3$ hurts its performance due to overly aggressive exploration.	

Lastly, we investigate the effect of downsampling the number of available macro-actions per input location to $20$ on the performance of anytime $\epsilon$-Macro-GPO. Similar to that for the real-world traffic phenomenon, 	
the performances of anytime $\epsilon$-Macro-GPO with $H=2,4$ and $20$ randomly selected macro-actions per input location are compared with that of anytime $\epsilon$-Macro-GPO with $H=2$ and all available macro-actions as well as myopic EI~\cite{shahriari16} with all available macro-actions of length $1$.	
It can be observed from Figs.~\ref{fig:robot_h2}a and~\ref{fig:robot_h2}b that when $H=2$, downsampling the number of available macro-actions per input location to $20$ decreases average normalized output measurements by $0.106 \sigma_{y}$ and increases simple regret by $0.064 \sigma_y$ after $20$ observations, 	
but also reduces the number of explored nodes (see Table~\ref{table:robot_h2}). 
By increasing to a lookahead of $20$ observations, anytime $\epsilon$-Macro-GPO with $H=4$ and $20$ randomly selected macro-actions per input location achieves  average normalized output measurements comparable to that with $H=2$ and all available macro-actions, but $0.136 \sigma_y$ less simple regret at the cost of a larger number of explored nodes.
Though EI can access all available macro-actions of length $1$ (i.e, no restriction on action space of the mobile robot), it obtains much less average normalized output measurements and considerably more simple regret than anytime $\epsilon$-Macro-GPO with $H=4$ and $20$ randomly selected macro-actions per input location due to its myopia.	
\begin{figure}[t]
	\centering
	{\begin{tabular}{cc}
			\includegraphics[height=5.7cm]{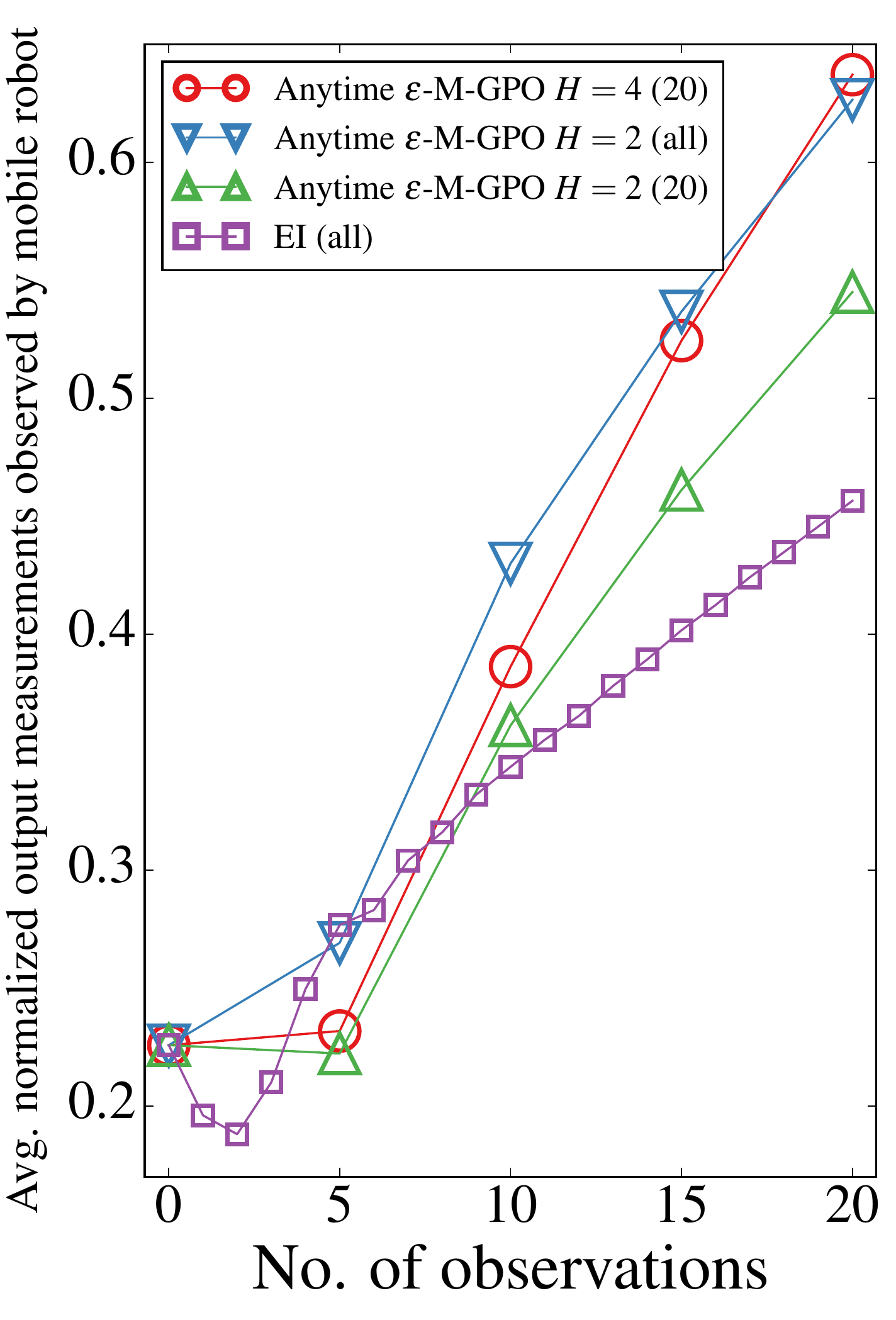} &
			\includegraphics[height=5.7cm]{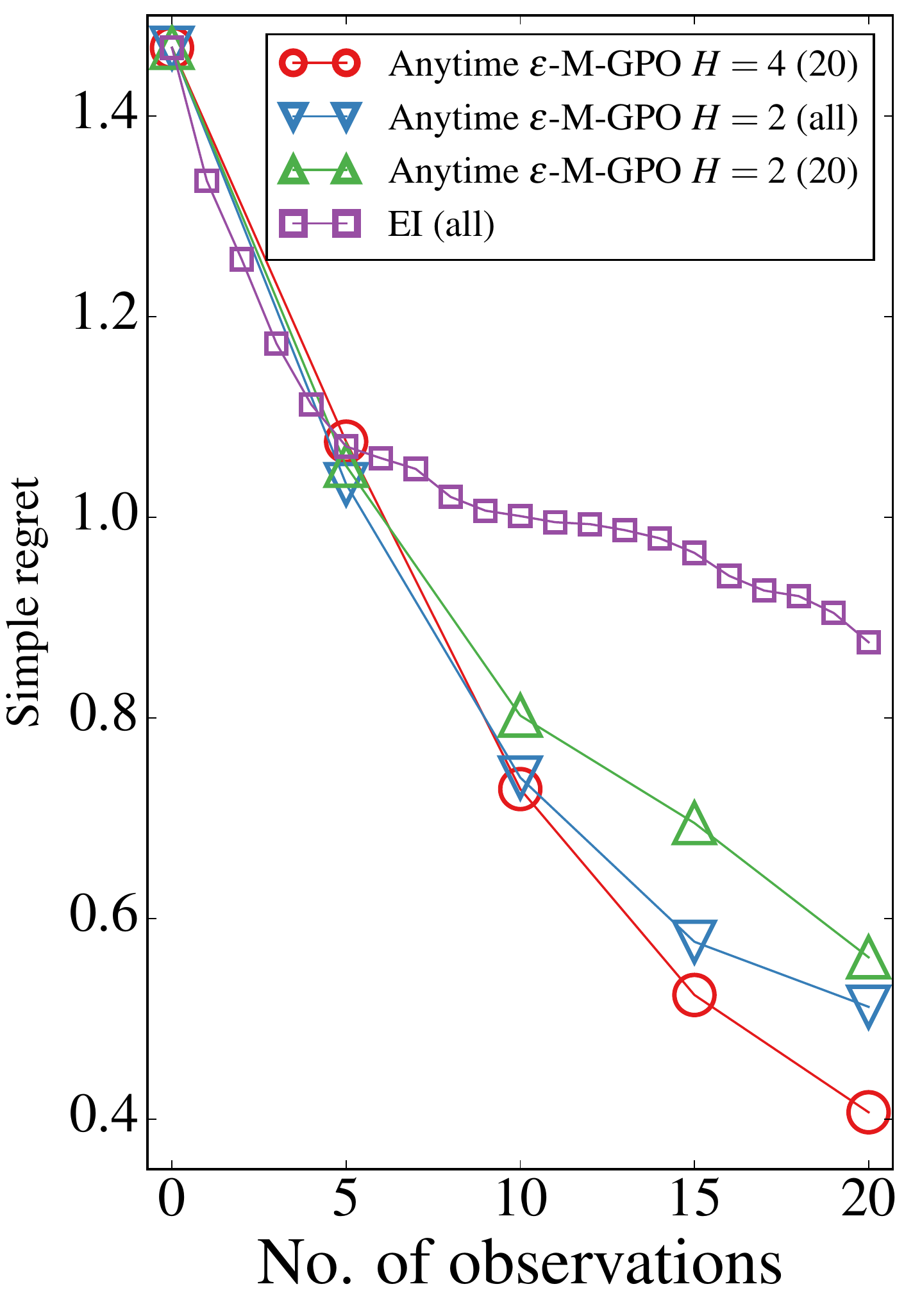}\\
			{(a)} & {(b)} 
	\end{tabular}}
	\caption{Graphs of (a) average normalized\cref{footnote-rewards} output measurements observed by the mobile robot and (b) simple regrets achieved by \emph{anytime} $\epsilon$-Macro-GPO with $H=2, 4$ and $20$ randomly selected macro-actions per input region, anytime $\epsilon$-Macro-GPO with $H=2$ and all available macro-actions (the no. of available macro-actions per input region is enclosed in brackets), and EI with all available macro-actions of length $1$ vs. no. of observations for the real-world temperature phenomenon over the Intel  Berkeley Research Lab. The standard errors are given in Table~\ref{table:robot-var-h2} in Appendix~\ref{sec:exp-robot}.}
	\label{fig:robot_h2}
\end{figure}
\begin{table}[h]
	\caption{No. of explored nodes by anytime $\epsilon$-Macro-GPO (the no. of available macro-actions per input region is enclosed in brackets)  for the real-world temperature phenomenon over the Intel  Berkeley Research Lab.}
	\centering
	\begin{tabular}{cccc}
		\hline
		$H = 2\ (20)$ & $H = 2\ (\text{all})$ & $H = 4\ (20)$\\
		\hline
		$8.88 \times 10^4$ & $2.49 \times 10^5$ & $1.12 \times 10^7$\\
		\hline
	\end{tabular}		
	\label{table:robot_h2}
\end{table}

\subsection{Comparison with Rollout~\cite{lam16}}
Our proposed algorithms are not benchmarked against Rollout~\cite{lam16} because Rollout~\cite{lam16} is not designed to handle macro-actions that are inherent to the structure of the task environments/applications considered in our work and experiments.
So, such a comparison would not be fair.
For a fair comparison with Rollout~\cite{lam16}, we set the macro-action length to $\kappa=1$ (i.e., primitive action) for our $\epsilon$-Macro-GPO and evaluate their performances using the metrics of average normalized output measurements observed by the agent and simple regret, and the synthetic dataset featuring the simulated plankton density phenomena in Section~\ref{expt}. 

Figs.~\ref{fig:rollout}a and~\ref{fig:rollout}b show results of the performances of  $\epsilon$-Macro-GPO ($H =4$, $\beta = 0$, and $N=20$) and the best-performing Rollout ($H=4$, $\gamma=1.0$, base policy: greedy EI-based policy defined in equations $22$ and $23$ in~\cite{lam16}) reported on page $7$ in~\cite{lam16}  averaged over $107$  independent realizations of the simulated phenomena.  
It can be observed that $\epsilon$-Macro-GPO achieves $0.143 \sigma_y$ more average normalized output measurement and $0.173 \sigma_y$ less simple regret than Rollout~\cite{lam16}.
To explain this, $\epsilon$-Macro-GPO considers all available actions from each input location during planning (equations~\ref{approx-policy},~\ref{ml-policy}, and~\ref{eq_4_8}) while Rollout utilizes only the action selected by the base policy (e.g., greedy EI) and ignores all the other available actions during planning, thus resulting in its suboptimal behavior.
\begin{figure}[h]
	\centering
	{\begin{tabular}{cc}
			\hspace{-1mm}\includegraphics[width=0.23 \textwidth]{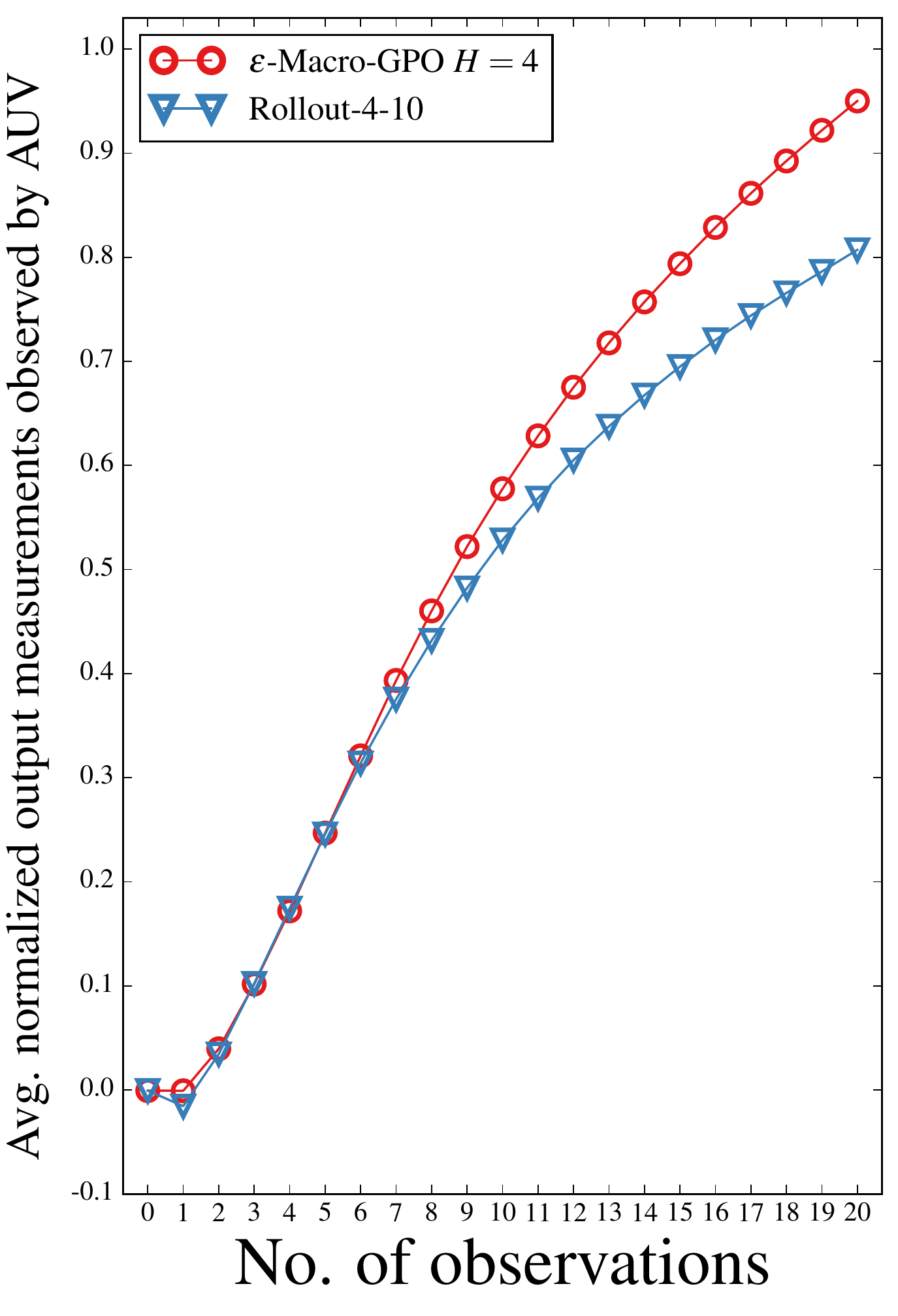}\hspace{-1mm} &
			\includegraphics[width=0.23 \textwidth]{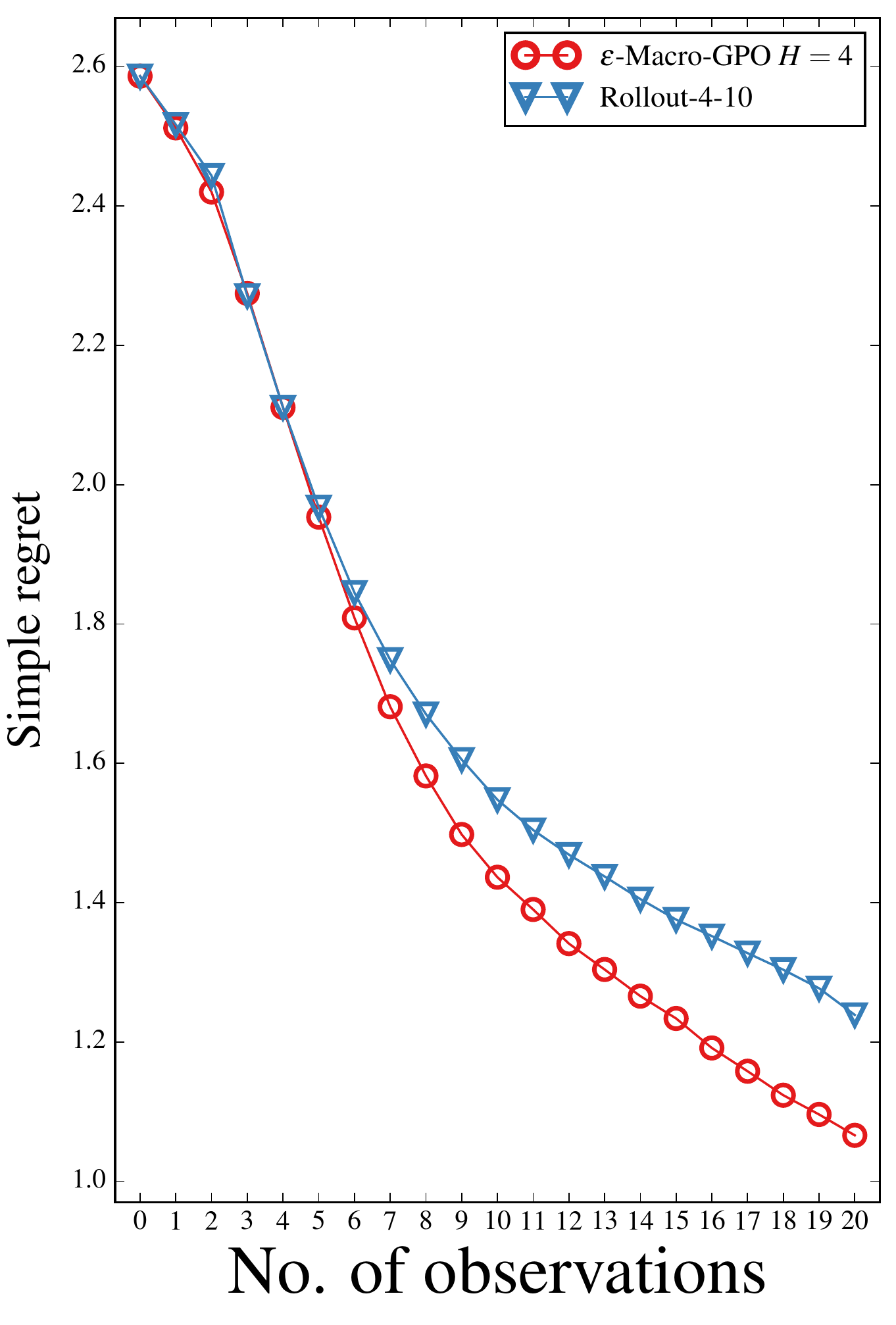}\\
			\hspace{-1mm}{\small (a)}\hspace{-1mm} & {\small (b)}\hspace{-1mm} \\
	\end{tabular}}
	\caption{Graphs of  (a) average normalized\cref{footnote-rewards} output measurements observed by AUV and (b) simple regrets achieved by $\epsilon$-Macro-GPO with $H=4$  and Rollout-$4$-$10$ vs. no. of observations for simulated plankton density phenomena. Standard errors are given in Table~\ref{table:rollout-var}  in Appendix~\ref{sec:rollout-add}.}
	\label{fig:rollout}
\end{figure}
\section{Conclusion}
This paper describes $\epsilon$-Macro-GPO and its anytime variant for nonmyopic adaptive BO that have been empirically shown to scale up to a lookahead of $20$ observations by exploiting macro-actions and consequently achieve superior BO performance.
Different from the asymptotic no-regret performance\cref{bravo} typical of GP-UCB and its variants, we theoretically guarantee the \emph{expected} performance loss of $\epsilon$-Macro-GPO and its anytime variant that can be specified to be arbitrarily small given a \emph{limited} budget.
Though this requires a polynomial number of stochastic samples in the macro-action length $\kappa$ in each planning stage (Theorem~\ref{th:expected}),
our experiments reveal that a relatively small sample size ($N$=$100$-$300$) is needed for $\epsilon$-Macro-GPO and its anytime variant to outperform state-of-the-art BO algorithms.
Though a sufficiently large exploration weight $\beta$ is usually needed to guarantee asymptotic no-regret performance\cref{bravo} for GP-UCB and its variants, we have observed in our experiments that their performances are highly sensitive to the chosen value of $\beta$ given a finite/limited budget and can be greatly hurt  by an often unknowingly ``large'' value of $\beta$ due to excessive exploration.
To sidestep this, $\epsilon$-Macro-GPO can eliminate the need of 
the additional weighted exploration term 
(i.e., $\beta= 0$) by utilizing a further lookahead, that is, if computational resources permit or are more affordable than the  cost of function evaluations.
\section*{Acknowledgment}
This research is supported by the Singapore Ministry of Education Academic Research Fund Tier $2$, MOE$2016$-T$2$-$2$-$156$.

%

%
%

\ifCLASSOPTIONcaptionsoff
  \newpage
\fi



%
\bibliographystyle{IEEEtran}
\bibliography{draft}

\clearpage

\appendices
\onecolumn

\section{Details on the Implementations of Batch BO Algorithms}
\label{used-algs}
See Table~\ref{table:used-algs}.
\begin{table}[h]
	\caption{Details on the available implementations of the batch BO algorithms for comparison with $\epsilon$-Macro-GPO in our experiments.}
	\centering
	\begin{scriptsize}
		\begin{tabular}{ | c | c |c |}
			\hline
			BO Algorithm & Language &  URL of Source Code \\
			\hline
			GP-BUCB & MATLAB & \url{http://www.gatsby.ucl.ac.uk/~tdesautels/} \\
			GP-UCB-PE & MATLAB  & \url{http://econtal.perso.math.cnrs.fr/software/} \\
			$q$-EI & Python & \url{https://github.com/oxfordcontrol/Bayesian-Optimization} \\
			BBO-LP & Python & \url{http://sheffieldml.github.io/GPyOpt/} \\
			\hline   
		\end{tabular}
	\end{scriptsize}
	\label{table:used-algs}	
\end{table}
\section{Additional Experimental Results for Simulated Plankton Density Phenomena}
\label{plankton}	
See Table~\ref{table:simulated-var-methods} and Table~\ref{table::simulated-var-beta}.
\begin{table}[h]
	\caption{Average normalized\cref{footnote-rewards} output measurements observed by the AUV and simple regrets achieved by the tested BO algorithms after $20$ observations.}
	\centering
	\begin{tabular}{ | c | c |c|}
		\hline
		BO Algorithm & Average normalized output measurements  &  Simple regret \\
		\hline
		$\epsilon$-Macro-GPO  $H = 4$ & $0.6310 \pm 0.0458$  & $1.2500 \pm 0.0541$ \\
		$\epsilon$-Macro-GPO  $H = 3$ & $0.5809 \pm 0.0486$ &  $1.3303 \pm 0.0542$\\
		$\epsilon$-Macro-GPO  $H = 2$ & $0.5446 \pm 0.0464$ & $1.3651 \pm 0.0550$ \\
		DB-GP-UCB & $0.5379 \pm 0.0462$ & $1.4612 \pm 0.0572$ \\
		Nonmyopic GP-UCB $H = 4$ & $0.5719 \pm 0.0467$ & $1.3984 \pm 0.0537$ \\
		GP-UCB-PE & $0.3635 \pm 0.0467$ & $1.4079 \pm 0.0568$ \\
		GP-BUCB & $0.3396 \pm 0.0486$ & $1.3717 \pm 0.0573$ \\
		$q$-EI & $0.2595 \pm 0.0444$ & $1.5104 \pm 0.0544$ \\
		BBO-LP & $0.3868 \pm 0.0444$ & $1.3666 \pm 0.0547$ \\
		\hline   
	\end{tabular}
	\label{table:simulated-var-methods}	
\end{table}

\begin{table}[h]
	\caption{Average normalized\cref{footnote-rewards} output measurements achieved by $\epsilon$-Macro-GPO with   $H=2$  and  $H = 3$ after $20$ observations.}
	\centering
	\begin{tabular}{ | c | c |c|}
		\hline
		Value of $\beta$ & $H=2$ & $H=3$ \\
		\hline
		$\beta = 0.0$  & $0.5563 \pm 0.0446$ & $0.5935 \pm 0.0461$\\
		$\beta = 0.1$  & $0.6207 \pm 0.0458$ &  $0.5842 \pm 0.0438$\\
		$\beta = 0.3$ &  $0.5357 \pm 0.0459$ & $0.5240 \pm 0.0446$\\
		$\beta =0.6$ &  $ 0.4226 \pm 0.0471$ & $0.5016 \pm 0.0470$\\
		$\beta =1.0$ &  $ 0.3746 \pm 0.0460$ & $0.4052 \pm 0.0489$\\
		$\beta = 2.0$ &  $ 0.2843 \pm 0.0478$ & $0.3566 \pm 0.0491$\\
		$\beta = 4.0$ &  $ 0.1919 \pm 0.0498$ & $0.2026 \pm 0.0441$\\
		$\beta = 10.0$ &  $ 0.0402 \pm 0.0468$ & $0.0569 \pm 0.0453$\\
		\hline   
	\end{tabular}
	\label{table::simulated-var-beta}	
\end{table}

\section{Additional Experimental Results for Real-World Traffic Phenomenon (i.e., Mobility Demand Pattern) over the Central Business District of an Urban City}
\label{sec:road-add}
See Tables~\ref{table:road-var-methods}, ~\ref{table::road-var-beta} and ~\ref{table:road-var-h2}.

\begin{table}[h]
	\caption{Average normalized\cref{footnote-rewards} output measurements observed by the AV and simple regrets achieved by the tested BO algorithms after $20$ observations for the real-world traffic phenomenon (i.e., mobility demand pattern).}
	\centering
	\begin{tabular}{ | c | c |c|}
		\hline
		BO Algorithm & Average normalized output measurements  &  Simple regret \\
		\hline
		Anytime $\epsilon$-Macro-GPO  $H = 4$ & $0.2700 \pm 0.1014$ & $1.5423 \pm 0.1047$ \\
		Anytime $\epsilon$-Macro-GPO  $H = 3$ & $0.2574 \pm 0.1019$ & $1.5843 \pm 0.0994$ \\
		Anytime $\epsilon$-Macro-GPO  $H = 2$ & $0.2357 \pm 0.1109$ & $1.7396 \pm 0.1179$ \\
		DB-GP-UCB & $0.2108 \pm 0.1081$ & $1.7050 \pm 0.1212$ \\
		Nonmyopic GP-UCB $H = 4$ & $0.2267 \pm 0.1134$ & $1.7314 \pm 0.1158$ \\
		GP-UCB-PE & $0.0770 \pm 0.0808$ & $1.5203 \pm 0.1247$ \\
		GP-BUCB & $0.0884 \pm 0.0819$ & $1.5177 \pm 0.1262$ \\
		$q$-EI & $0.0007 \pm 0.0945$ & $1.7945 \pm 0.1515$ \\
		BBO-LP & $-0.0077 \pm 0.0957$ & $1.7320 \pm 0.1149$ \\
		\hline   
	\end{tabular}
	\label{table:road-var-methods}	
\end{table}

\begin{table}[h]
	\caption{Average normalized\cref{footnote-rewards} output measurements achieved by  \emph{anytime} $\epsilon$-Macro-GPO with   $H=2, 3$ and varying exploration weights $\beta$ after $20$ observations for the real-world traffic phenomenon (i.e., mobility demand pattern).}
	\centering
	\begin{tabular}{ | c | c |c|}
		\hline
		Value of $\beta$ & $H=2$ & $H=3$ \\
		\hline
		$\beta = 0.0$ & $0.2357 \pm 0.1109$ & $0.2574 \pm 0.1019$ \\
		$\beta = 0.2$ & $0.2550 \pm 0.1032$ & $0.2069 \pm 0.0987$ \\
		$\beta = 0.5$ & $0.1364 \pm 0.0967$ & $0.1174 \pm 0.0893$ \\
		$\beta = 1.0$ & $0.1429 \pm 0.0967$ & $0.0911 \pm 0.0772$ \\
		$\beta = 2.0$ & $0.1174 \pm 0.0843$ & $0.0330 \pm 0.0755$ \\
		$\beta = 4.0$ & $0.0957 \pm 0.0841$ & $0.0403 \pm 0.0765$ \\
		$\beta = 10.0$ & $0.0944 \pm 0.0768$ & $-0.0046 \pm 0.0756$ \\
		\hline   
	\end{tabular}
	\label{table::road-var-beta}	
\end{table}

\begin{table}[h]
	\caption{Average normalized output measurements observed by the AV and simple regrets achieved by \emph{anytime} $\epsilon$-Macro-GPO with $H=2, 4$ and $20$ randomly selected macro-actions per input region, anytime $\epsilon$-Macro-GPO with $H=2$ and all available macro-actions (the no. of available macro-actions per input region is enclosed in brackets), and EI with all available macro-actions of length $1$ after $20$ observations for the real-world traffic phenomenon (i.e., mobility demand pattern).}
	\centering
	\begin{tabular}{ | c | c |c|}
		\hline
		BO Algorithm & Average normalized output measurements  &  Simple regret \\
		\hline
		Anytime $\epsilon$-Macro-GPO  $H = 4$  ($20$) & $0.2700 \pm 0.1014$ & $1.5423 \pm 0.1047$ \\
		Anytime $\epsilon$-Macro-GPO  $H = 2$ (all) & $0.2631 \pm 0.0918$ & $1.6427 \pm 0.0792$ \\
		Anytime $\epsilon$-Macro-GPO  $H = 2$  ($20$) & $0.2357 \pm 0.1109$ & $1.7396 \pm 0.1179$ \\
		EI (all) & $0.1469 \pm 0.1084$ & $1.6094 \pm 0.0946$ \\
		\hline   
	\end{tabular}
	\label{table:road-var-h2}	
\end{table}


\section{Additional Experimental Results for Real-World Temperature Phenomenon over an Office Environment}
\label{sec:exp-robot}
See Table~\ref{table:robot-var-methods}, Table~\ref{table::robot-var-beta} and Table~\ref{table:robot-var-h2}.

\begin{table}[h]
	\caption{Average normalized\cref{footnote-rewards} output measurements observed by the mobile robot and simple regrets achieved by the tested BO algorithms after $20$ observations for the real-world temperature phenomenon over the Intel Berkeley Research Lab.}
	\centering
	\begin{tabular}{ | c | c |c|}
		\hline
		BO Algorithm & Average normalized output measurements  &  Simple regret \\
		\hline
		Anytime $\epsilon$-Macro-GPO  $H = 4$ & $0.6371 \pm 0.0797$ & $0.4069 \pm 0.0723$ \\
		Anytime $\epsilon$-Macro-GPO  $H = 3$ & $0.6137 \pm 0.0829$ & $0.4285 \pm 0.0678$ \\
		Anytime $\epsilon$-Macro-GPO  $H = 2$ & $0.5450 \pm 0.0951$ & $0.5613 \pm 0.0834$ \\
		DB-GP-UCB & $0.4874 \pm 0.1017$ & $0.6734 \pm 0.0934$ \\
		Nonmyopic GP-UCB $H = 4$ & $0.5708 \pm 0.0908$ & $0.5911 \pm 0.0886$ \\
		GP-UCB-PE & $0.1377 \pm 0.0734$ & $0.6700 \pm 0.0758$ \\
		GP-BUCB & $0.2067 \pm 0.0758$ & $0.6670 \pm 0.0762$ \\
		$q$-EI & $0.3801 \pm 0.1044$ & $0.6868 \pm 0.1116$ \\
		BBO-LP & $0.2549 \pm 0.0833$ & $0.5168 \pm 0.0733$ \\	
		\hline   
	\end{tabular}
	\label{table:robot-var-methods}	
\end{table}
\begin{table}[h]
	\caption{Average normalized\cref{footnote-rewards} output measurements achieved by $\epsilon$-Macro-GPO with   $H=2, 3$ and varying exploration weights $\beta$ after $20$ observations for the real-world temperature phenomenon over the Intel Berkeley Research Lab.}
	\centering
	\begin{tabular}{ | c | c |c|}
		\hline
		Value of $\beta$ & $H=2$ & $H=3$ \\
		\hline
		$\beta=0.0$ & $0.5450 \pm 0.0951$ & $0.6137 \pm 0.0829$ \\
		$\beta=1.0$ & $0.6160 \pm 0.0820$ & $0.6047 \pm 0.0764$ \\
		$\beta=2.0$ & $0.5565 \pm 0.0765$ & $0.5787 \pm 0.0786$ \\
		$\beta=3.0$ & $0.3755 \pm 0.0670$ & $0.4468 \pm 0.0645$ \\
		$\beta=4.0$ & $0.1859 \pm 0.0608$ & $0.2294 \pm 0.0472$ \\
		\hline   
	\end{tabular}
	\label{table::robot-var-beta}	
\end{table}

\begin{table}[h]
	\caption{Average normalized output measurements observed by the mobile robot and simple regrets achieved by \emph{anytime} $\epsilon$-Macro-GPO with $H=2, 4$ and $20$ randomly selected macro-actions per input region, anytime $\epsilon$-Macro-GPO with $H=2$ and all available macro-actions (the no. of available macro-actions per input region is enclosed in brackets), and EI with all available macro-actions of length $1$ after $20$ observations for the real-world temperature phenomenon over the Intel Berkeley Research Lab.}
	\centering
	\begin{tabular}{ | c | c |c|}
		\hline
		BO Algorithm & Average normalized output measurements  &  Simple regret \\
		\hline
		Anytime $\epsilon$-Macro-GPO  $H = 4$  ($20$) & $0.6371 \pm 0.0797$ & $0.4069 \pm 0.0723$ \\
		Anytime $\epsilon$-Macro-GPO  $H = 2$ (all) & $0.6265 \pm 0.0861$ & $0.5119 \pm 0.0807$ \\
		Anytime $\epsilon$-Macro-GPO  $H = 2$  ($20$) & $0.5450 \pm 0.0951$ & $0.5613 \pm 0.0834$ \\
		EI (all) & $0.4565 \pm 0.1051$ & $0.8754 \pm 0.0941$ \\
		\hline   
	\end{tabular}
	\label{table:robot-var-h2}	
\end{table}

\section{Additional Experimental Results for Comparison with Rollout~\cite{lam16}}

\label{sec:rollout-add}
See Table~\ref{table:rollout-var}.

\begin{table}[h]
	\caption{Average normalized\cref{footnote-rewards} output measurements observed by AUV and (b) simple regrets achieved by $\epsilon$-Macro-GPO with $H=4$  and Rollout-$4$-$10$ vs. no. of observations for simulated plankton density phenomena.}
	\centering
	\begin{tabular}{ | c | c |c|}
		\hline
		BO Algorithm & Average normalized output measurements  &  Simple regret \\
		\hline
		$\epsilon$-Macro-GPO  $H = 4$  & $0.9501 \pm 0.0659$ & $1.066 \pm 0.0783$ \\
		Rollout-$4$-$10$ & $0.8071 \pm 0.0637$ & $1.2389 \pm 0.0808$ \\
		\hline   
	\end{tabular}
	\label{table:rollout-var}	
\end{table}

\section{Derivation of~\eqref{eq:general-policy}}
\label{general-policy-proof}
The second summand on RHS of~\eqref{eq:policy_value} can be re-written as
\begin{equation}
\label{eq:app-1}
\mathbb{I}   [y_{\Sdom} ; \zHist{H}| d_0, \pi] = \sum_{t = 1}^H {\mathbb{I}[y_{\Sdom} ; \zNew{t}|\langle \sHist{t-1},\zNew{0}\oplus \zHist{t - 1} \rangle, \pi]} = 0.5 \sum_{t = 0}^{H-1} {   \log | I + \sigma_n^{-2}  \Sigma_{\sNew{t+1} | \sHist{t},\pi} |}\ .
\end{equation}
The first equality is due to the chain rule for mutual information~\cite{cover-elements}. Let $\sNew{t-1} \triangleq (s_{t-1, 1},\ldots,s_{t-1, \kappa})$.
The last equality follows from
\begin{equation}
\begin{array}{l}
\displaystyle\mathbb{I}[y_{\Sdom} ; \zNew{t}|\langle \sHist{t-1},\zNew{0}\oplus \zHist{t - 1} \rangle, \pi] \\
\displaystyle= \mathbb{H}[\zNew{t}| \langle \sHist{t-1},\zNew{0}\oplus \zHist{t - 1} \rangle, \pi] - \mathbb{H}[\zNew{t}| \langle \sHist{t-1},\zNew{0}\oplus \zHist{t - 1} \rangle, y_{\Sdom}, \pi]  \\
= \mathbb{H}[\zNew{t}| \langle \sHist{t-1},\zNew{0}\oplus \zHist{t - 1}  \rangle, \pi] - \mathbb{H}[\zNew{t}| (y_{s_{t-1, 1}},\ldots,y_{s_{t-1, \kappa}}), \pi]  \\	
= 0.5 \kappa \log(2\pi e) + 0.5\log |\sigma^2_n I+\Sigma_{\sNew{t} | \sHist{t-1},\pi}| - 0.5  \kappa \log(2\pi e) - 0.5\log |\sigma^2_n I|  \\
= 0.5 \log (|\sigma^2_n I+\Sigma_{\sNew{t} | \sHist{t-1},\pi} | \ |\sigma^2_n I|^{-1}) \\
= 0.5\log (|\sigma^2_n I+\Sigma_{\sNew{t} | \sHist{t-1},\pi}| |\sigma^{-2}_n I|)  \\
= 0.5\log |I+\sigma^{-2}_n  \Sigma_{\sNew{t} | \sHist{t-1},\pi}|
\end{array}
\end{equation}
where the first equality is due to the definition of conditional mutual information, the third equality is due to the definition of Gaussian entropy, that is, $\mathbb{H}[\zNew{t}| \langle \sHist{t-1},\zNew{0}\oplus \zHist{t - 1} \rangle, \pi]\triangleq 0.5  \kappa\log(2\pi e) + 0.5\log |\sigma^2_n I+\Sigma_{\sNew{t} | \sHist{t-1},\pi}|$ and $\mathbb{H}[\zNew{t}|  (y_{s_{t-1,1}},\ldots,y_{s_{t-1, \kappa }}), \pi] \triangleq 0.5  \kappa  \log(2\pi e) + 0.5\log |\sigma^2_n I|$, the latter of which follows from $\varepsilon = z_{t, i} - y_{s_{t, i}} \sim \mathcal{N}({0},\sigma^2_n)$ for stage  $t = 0,\ldots, H -1$ and $i = 1,\ldots, \kappa$, and hence $p(\zNew{t}| (y_{s_{t-, 1}},\ldots,y_{s_{t-1, \kappa}}), \pi) = \mathcal{N}({\bf 0},\sigma^2_n I)$. So,~\eqref{eq:policy_value} can be re-expressed as
\begin{equation}\label{arghh}
V^{\pi}_0(d_0)=\mathbb{E}_{ \zHist{H} | d_0, \pi} [ \mathbf{1}^{\top} \zHist{H} ] + 0.5\beta \sum_{t = 0}^{H-1} {   \log | I + \sigma_n^{-2}  \Sigma_{\sNew{t+1} | \sHist{t},\pi} |}\ .
\end{equation}
Given an arbitrary positive integer $H'$ and denoting  $\mathbf{z}_{\tau+1:H'}$ 
as a vector of realized 
output measurements from stage $\tau + 1$ to stage $H'$,~\eqref{arghh} for $H=1,\ldots,H'$
are, respectively, equivalent to
\begin{equation}
\label{whyme}
V^{\pi}_{\tau}(d_{\tau})=\mathbb{E}_{\mathbf{z}_{\tau+1:H'} | d_{\tau}, \pi} [ \mathbf{1}^{\top} \mathbf{z}_{\tau+1:H'} ] + 0.5\beta \sum_{t = \tau}^{H'-1} {   \log | I + \sigma_n^{-2}  \Sigma_{\sNew{t+1} | \sHist{t},\pi} |}
\end{equation}
for $\tau=H'-1,\ldots,0$ by simply adding $\tau$ to the indices denoting the planning stage in~\eqref{arghh}. From~\eqref{whyme},
$$
\hspace{-1.7mm}
\begin{array}{l}
V^{\pi}_{\tau}(d_{\tau})\\
=\displaystyle\mathbb{E}_{ \mathbf{z}_{\tau+1:H'} | d_{\tau}, \pi } [ \mathbf{1}^{\top} \mathbf{z}_{\tau+1:H'} ] + 0.5\beta \sum_{t = \tau}^{H'-1} {   \log | I + \sigma_n^{-2}  \Sigma_{\sNew{t+1} | \sHist{t},\pi} |}\\
=\displaystyle\int \mathbf{1}^{\top} \mathbf{z}_{\tau+1:H'}\ p(\mathbf{z}_{\tau+1:H'} | d_{\tau}, \pi) \ \mathrm{d} \mathbf{z}_{\tau+1:H'} + 0.5\beta \sum_{t = \tau}^{H'-1} {   \log | I + \sigma_n^{-2}  \Sigma_{\sNew{t+1} | \sHist{t},\pi} |}\\
=\displaystyle\int (\mathbf{1}^{\top} \mathbf{z}_{\tau+1}+\mathbf{1}^{\top} \mathbf{z}_{\tau+2:H'})\ p(\mathbf{z}_{\tau+2:H'} | d_{\tau+1}, \pi)\ \mathrm{d} \mathbf{z}_{\tau+2:H'} \ p(\mathbf{z}_{\tau+1} | d_{\tau}, \pi)\ \mathrm{d} \mathbf{z}_{\tau+1}\\
\quad \displaystyle+\ 0.5\beta \sum_{t = \tau}^{H'-1} {   \log | I + \sigma_n^{-2}  \Sigma_{\sNew{t+1} | \sHist{t},\pi} |}\\
=\displaystyle\int \mathbf{1}^{\top} \mathbf{z}_{\tau+1}\int p(\mathbf{z}_{\tau+2:H'} | d_{\tau+1}, \pi)\ \mathrm{d} \mathbf{z}_{\tau+2:H'} \ p(\mathbf{z}_{\tau+1} | d_{\tau}, \pi)\ \mathrm{d} \mathbf{z}_{\tau+1}\\
\quad \displaystyle+\ 0.5\beta\  {\log | I + \sigma_n^{-2}  \Sigma_{\sNew{\tau+1} | \sHist{\tau},\pi} |}\\
\quad \displaystyle+\int \mathbf{1}^{\top} \mathbf{z}_{\tau+2:H'}\ p(\mathbf{z}_{\tau+2:H'} | d_{\tau+1}, \pi)\ \mathrm{d} \mathbf{z}_{\tau+2:H'} \ p(\mathbf{z}_{\tau+1} | d_{\tau}, \pi)\ \mathrm{d} \mathbf{z}_{\tau+1}\\
\quad \displaystyle+\ 0.5\beta \sum_{t = \tau+1}^{H'-1} {   \log | I + \sigma_n^{-2}  \Sigma_{\sNew{t+1} | \sHist{t},\pi} |}\\
=\displaystyle\int \mathbf{1}^{\top} \mathbf{z}_{\tau+1}\ p(\mathbf{z}_{\tau+1} | d_{\tau}, \pi)\ \mathrm{d} \mathbf{z}_{\tau+1}+\ 0.5\beta\  {\log | I + \sigma_n^{-2}  \Sigma_{\sNew{\tau+1} | \sHist{\tau},\pi} |}\\
\quad \displaystyle+\int\hspace{-1mm}\int \mathbf{1}^{\top} \mathbf{z}_{\tau+2:H'}\ p(\mathbf{z}_{\tau+2:H'} | d_{\tau+1}, \pi)\ \mathrm{d} \mathbf{z}_{\tau+2:H'} + 0.5\beta \hspace{-1mm}\sum_{t = \tau+1}^{H'-1}\hspace{-1mm} {   \log | I + \sigma_n^{-2}  \Sigma_{\sNew{t+1} |\sHist{t},\pi} |} \ p(\mathbf{z}_{\tau+1} | d_{\tau}, \pi)\ \mathrm{d} \mathbf{z}_{\tau+1}\\
=\displaystyle\mathbf{1}^{\top} \mu_{\mathbf{s}_{\tau+1}|d_{\tau},\pi}+\ 0.5\beta\  {\log | I + \sigma_n^{-2}  \Sigma_{\sNew{\tau+1} | \sHist{\tau},\pi} |}\\
\quad \displaystyle+\int\mathbb{E}_{\mathbf{z}_{\tau+2:H'} | d_{\tau+1}, \pi} [ \mathbf{1}^{\top} \mathbf{z}_{\tau+2:H'} ]+ 0.5\beta \hspace{-1mm}\sum_{t = \tau+1}^{H'-1}\hspace{-1mm} {   \log | I + \sigma_n^{-2}  \Sigma_{\sNew{t+1} | \sHist{t},\pi} |} \ p(\mathbf{z}_{\tau+1} | d_{\tau}, \pi)\ \mathrm{d} \mathbf{z}_{\tau+1}\\
=\displaystyle\mathbf{1}^{\top} \mu_{\pi(d_{\tau})|d_{\tau}}+\ 0.5\beta\  {\log | I + \sigma_n^{-2}  \Sigma_{\pi(d_{\tau}) | \sHist{\tau}} |}+\int V^{\pi}_{\tau+1}(d_{\tau+1}) \ p(\mathbf{z}_{\tau+1} | d_{\tau}, \pi)\ \mathrm{d} \mathbf{z}_{\tau+1}\\
=\displaystyle R(\pi(d_{\tau}),d_{\tau}) + \mathbb{E}_{\zNew{t+1} |\pi(d_{\tau}),d_{\tau}}
[V^{\pi}_{\tau+1}(\langle\sHist{t}\oplus\pi(d_{\tau}), \zHist{t}\oplus \zNew{t+1} \rangle)]\\
=\displaystyle Q_{\tau}^\pi( \pi(d_{\tau}), d_{\tau})
\end{array}	
$$	
for stages $\tau=0,\ldots,H'-1$ where	 the third last equality is due to~\eqref{whyme} and the last two equalities follow from the definitions of $R$ and $Q_{\tau}^\pi$ in~\eqref{eq:reward-def} and~\eqref{eq:general-policy}, respectively.

\section{Proof of Lemma~\ref{lemma:reward}}
\label{lemma:reward-proof}

\begin{proof}
	$$
	\begin{array}{l}
	\displaystyle | R(\sNew{t+1}, d_t) - 	R(\sNew{t+1}, d'_{t}) | \\
	\displaystyle= | \mathbf{1}^{\top}( \mu_{\sNew{t+1}|d_{t}} -  \mu_{\sNew{t+1}|d'_{t}})| \\
	\displaystyle\le \lVert \mu_{\sNew{t+1}|d_{t}} -  \mu_{\sNew{t+1}|d'_{t}} \rVert_1 \\
	\displaystyle = \lVert \Sigma_{\sNew{t+1} \sHist{t}}    \Sigma^{-1}_{\sHist{t} \sHist{t}}    (\zHist{t} - \zHist{t}')^{\top} \rVert_1 \\
	\displaystyle\le \sqrt{\kappa}\  \lVert \Sigma_{\sNew{t+1} \sHist{t}} \Sigma^{-1}_{\sHist{t} \sHist{t}}    (\zHist{t} - \zHist{t}')^{\top} \rVert\\
	\displaystyle = \sqrt{\kappa}\  \lVert \Sigma_{\sNew{t+1} \sHist{t}} \Sigma^{-1}_{\sHist{t} \sHist{t}}    (\zHist{t} - \zHist{t}')^{\top} \rVert_F\\
	\displaystyle\le \sqrt{\kappa}\  \lVert \Sigma_{\sNew{t+1} \sHist{t}} \Sigma^{-1}_{\sHist{t} \sHist{t}}\rVert_F   \lVert \zHist{t} - \zHist{t}' \rVert_F\\
	\displaystyle = \sqrt{\kappa}\  \lVert \Sigma_{\sNew{t+1} \sHist{t}} \Sigma^{-1}_{\sHist{t} \sHist{t}}\rVert_F   \lVert \zHist{t} - \zHist{t}' \rVert\\				
	\displaystyle= \sqrt{\kappa}\ \alpha(\sHist{t+1}) \lVert\zHist{t} - \zHist{t}'\rVert\ .
	\end{array}
	$$	
	The first equality is due to \eqref{eq:reward-def}. The first inequality is due to triangle inequality. The second equality is due to \eqref{gp-posteriors}. The second inequality follows from a property of vector norms (see Section $2.2.2$ in~\cite{Golub96}). 
	The last inequality is due to the submultiplicativity of the Frobenius norm (see Section II.$2.1$ in~\cite{Stewart90}).
	The last equality follows from the definition of $\alpha(\sHist{t+1})$.
\end{proof}
\section{Lipschitz Continuity of $V^*_t(d_t)$~\eqref{eq:OptimalValFunDef}}
\label{lip-optimal-value}
\begin{definition}
	\label{definition-L}
	Let $L_H({\sHist{H}}) \triangleq 0$. Define
	\begin{equation*}
	L_t(\sHist{t}) \triangleq \max_{{\sNew{t+1} \in \Adom(\sNew{t})}} 
	\sqrt{\kappa}\ \alpha(\sHist{t+1}) + L_{t+1}(\sHist{t+1}) \sqrt{1 + \alpha(\sHist{t+1})^2}
	\end{equation*}
	for $t=0,\ldots, H-1$ where the function $\alpha$ is previously defined in Lemma~\ref{lemma:reward}.
\end{definition}
The following result shows that $V_t^*(d_t)$~\eqref{eq:OptimalValFunDef} is Lipschitz continuous in the realized output measurements $\zHist{t}$ with Lipschitz constant $L_t(\sHist{t})$:
\begin{theorem}
	\label{th:lip-optimal}
	For $t=0,\ldots, H$,
	\begin{equation}\label{whata}
	| V^{*}_t(d_t) - V^{*}_t(d'_t)| \le L_t(\sHist{t}) \lVert\zHist{t} - \mathbf{z}'_{0:t}\rVert
	\end{equation}
	where $d'_t$ is previously defined in Lemma~\ref{lemma:reward}.
\end{theorem}
\begin{proof}
	We give a proof by induction on $t$.
	When $t = H$ (i.e., base case), $V^*_H(d_H) = 0$ for any $d_H$. So, $| V^{*}_H(d_H) - V^{*}_H(d'_H)| = 0 \le L_H(\sHist{H}) \lVert\zHist{H} - \mathbf{z}'_{0:H}\rVert$.
	Supposing~\eqref{whata} holds for $t+1$ (i.e., induction hypothesis), we will prove that it holds for $t=0,\ldots,H-1$.
	Let $\sNew{t+1}^* \triangleq \pi^*(d_t)$ and $\Delta_{t+1} \triangleq \mu_{\sNew{t+1}^*
		| d_t } - \mu_{\sNew{t+1}^*| d'_t }$. Using~\eqref{gp-posteriors}, the submultiplicativity of the Frobenius norm (see Section II.$2.1$ in~\cite{Stewart90}), and the definition of $\alpha(\sHist{t+1})$,
	\begin{equation}
	\label{eq:app_delta}
	\lVert\Delta_{t+1}\rVert \le \alpha(\sHist{t} \oplus \sNew{t+1}^*) \lVert\zHist{t} - \zHist{t}' \rVert\ .
	\end{equation}
	Without loss of generality, assume that $V^*_t (d_t)  \ge  V^*_t (d'_t)$. From \eqref{eq:OptimalValFunDef},
	\begin{equation} 
	\hspace{-1.7mm}
	\label{eq:ap-2}
	\begin{array}{l}
	\displaystyle V^*_t (d_t)  -  V^*_t (d'_t)\\
	\displaystyle\le   Q^*_t (\sNew{t+1}^*, d_t) - Q^*_t (\sNew{t+1}^*, d'_t) \\
	\displaystyle\le | Q^*_t(\sNew{t+1}^*, d_t) - Q^*_t(\sNew{t+1}^*, d'_t) | \\
	\displaystyle\le \left| R(\sNew{t+1}^*, d_t) - R(\sNew{t+1}^*, d'_{t}) \right| +
	\Bigg| \int {p(\zNew{t+1} | \sNew{t+1}^*, d_t)\ V^*_{t+1}(
		\langle\sHist{t} \oplus \sNew{t+1}^*, \zHist{t+1} \rangle
		) }\ \text{d}\zNew{t+1}\\
	\quad\displaystyle - \int {p(\zNew{t+1}' | \sNew{t+1}^*, d'_t)\ V^*_{t+1}(\langle\sHist{t} \oplus \sNew{t+1}^*, \zHist{t+1}' \rangle		
		) }\ \text{d}\zNew{t+1}' \Bigg| \\
	\displaystyle\le 
	\sqrt{\kappa}\ \alpha(\sHist{t} \oplus \sNew{t+1}^*) \lVert\zHist{t} - \zHist{t}'\rVert + \int {p(\zNew{t+1} | \sNew{t+1}^*, d_t)\ L_{t+1}(\sHist{t} \oplus \sNew{t+1}^*) \lVert(\zHist{t} - \zHist{t}') \oplus \Delta_{t+1}\rVert}\ \text{d}\zNew{t+1} \\
	\displaystyle =
	\sqrt{\kappa}\ \alpha(\sHist{t} \oplus \sNew{t+1}^*) \lVert\zHist{t} - \zHist{t}'\rVert + L_{t+1}(\sHist{t} \oplus \sNew{t+1}^*)  \lVert(\zHist{t} - \zHist{t}') \oplus \Delta_{t+1}\rVert\\
	\displaystyle\le\sqrt{\kappa}\ \alpha(\sHist{t} \oplus \sNew{t+1}^*) \lVert\zHist{t} - \zHist{t}'\rVert + L_{t+1}(\sHist{t} \oplus \sNew{t+1}^*) \sqrt{1 + \alpha(\sHist{t} \oplus \sNew{t+1}^*)^2}\ \lVert\zHist{t} - \zHist{t}'\rVert\\		
	\displaystyle\le L_t(\sHist{t}) \lVert \zHist{t} - \zHist{t}' \rVert
	\end{array}
	\end{equation}
	where the third inequality follows from \eqref{eq:OptimalValFunDef} and triangle inequality, the fourth inequality follows from Lemma~\ref{lemma:reward}, change of variable $\zNew{t+1}' \triangleq \zNew{t+1} - \Delta_{t+1}$, and the induction hypothesis,  
	the second last inequality in~\eqref{eq:ap-2} is due to	
	$$
	\lVert(\zHist{t} - \zHist{t}') \oplus \Delta_{t+1}\rVert = \sqrt{\lVert\zHist{t} - \zHist{t}' \rVert^2 + \lVert\Delta_{t+1}\rVert^2 } \le \sqrt{1 + \alpha(\sHist{t} \oplus \sNew{t+1}^*)^2}\ \lVert\zHist{t} - \zHist{t}'\rVert	 
	$$ 
	with the inequality following from~\eqref{eq:app_delta}, and 
	the last inequality in~\eqref{eq:ap-2} is due to the definition of $L_t$ (Definition~\ref{definition-L}).
\end{proof}

\section{Proof of Theorem~\ref{th:1_new}}
\label{sec:th:1_new_proof}

\begin{proof}
	There are two sources of error arising in using 
	$\mathcal{Q}_t (\sNew{t+1}, d_t)$
	to approximate 
	${Q}^*_t (\sNew{t+1}, d_t)$:
	(a) Every stage-wise expectation term in \eqref{eq:OptimalValFunDef} is approximated via stochastic sampling~\eqref{approx-policy} of a finite number $N$ of i.i.d. multivariate Gaussian vectors $\zNew{}^1,\ldots, \zNew{}^N$ from the GP posterior belief $p(\zNew{t+1} | \sNew{t+1}, d_t ) = \mathcal{N}(\mu_{\sNew{t+1} | d_t}, \Sigma_{\sNew{t+1} | \sHist{t}})$~\eqref{gp-posteriors},
	and (b) evaluating 
	$\mathcal{Q}_t(\sNew{t+1}, d_t)$
	does not involve utilizing the values of $V_{t+1}^*$ but rather that of its approximation $\mathcal{V}_{t+1}$. To facilitate capturing the error due to finite
	stochastic sampling described in (a), the following intermediate function is introduced:
	\begin{equation}
	\label{eq:u-function}
	\mathcal{U}_t (\sNew{t+1}, d_t) \triangleq R(\sNew{t+1}, d_t) + 
	\frac{1}{N}  
	\sum_{\ell= 1}^{N} 
	V^{*}_{t+1} (\langle \sHist{t+1}, \zHist{t}\oplus\zNew{}^\ell \rangle) 
	\end{equation}
	for $t =0, \ldots, H-1$. The following lemma shows that $ \mathcal{U}_t (\sNew{t+1}, d_t)$ can approximate  $Q^{*}_t (\sNew{t+1}, d_t)$ arbitrarily closely:
	\begin{lemma}
		\label{lemma:concentration}
		Suppose that the observations $d_{t'}$, $H\in\mathbb{Z}^+$, a budget of $\kappa(H-t')$ input locations for $t'=0,\ldots, H-1$, $\lambda > 0$, and $N \in \mathbb{Z}^+$ are given.		
		For all tuples $\langle t, \sNew{t+1}, d_t\rangle$ generated at stage $t = t', \ldots, H-1$ by \eqref{approx-policy} to compute $\mathcal{V}_{t'} (d_{t'})$,
		$$
		{P}(| \mathcal{U}_t (\sNew{t+1}, d_t) - Q^{*}_t (\sNew{t+1}, d_t)| \le \lambda) \ge 1 -  2 \exp\left(-\frac{N \lambda^2}{2{K}^2}\right)
		$$
		where $K \triangleq\mathcal{O}(\kappa^{H}\sqrt{H!}\ \sigma_n (1+\sigma^2_y/\sigma^2_n)^{H})$.
	\end{lemma}
	\begin{proof}
		For any tuple $\langle t, \sNew{t+1}, d_t\rangle$, define the following auxiliary function:
		%
		\begin{equation}
		\label{eq_2_9}
		\begin{array}{rl}
		\mathcal{G}( \zNew{}^1,\ldots, \zNew{}^N) \triangleq&\hspace{-2.4mm}  \displaystyle
		\frac{1}{N}  
		\sum_{\ell = 1}^{N} 
		V^{*}_{t+1} ( \langle \sHist{t+1}, \zHist{t} \oplus  \zNew{}^\ell \rangle )\\
		=&\hspace{-2.4mm}\displaystyle
		\mathcal{U}_t (\sNew{t+1}, d_t) -   R(\sNew{t+1}, d_t)
		\end{array}
		\end{equation}	
		which follows from~\eqref{eq:u-function}.	
		Taking an expectation of~\eqref{eq_2_9} with respect to GP posterior belief $p(\zNew{t+1} | \sNew{t+1}, d_t ) =\mathcal{N}(\mu_{\sNew{t+1} | d_t}, \Sigma_{\sNew{t+1} | \sHist{t}})$ gives
		\begin{equation}
		\label{eq_2_10}	
		\begin{array}{l}
		\displaystyle \mathbb{E}_{ \zNew{}^1, \ldots,  \zNew{}^N \sim \mathcal{N}(\mu_{\sNew{t+1} | d_t}, \Sigma_{\sNew{t+1} | \sHist{t}})}\left[\mathcal{G}(\zNew{}^1, \dots, \zNew{}^N )\right] \\
		\displaystyle = \mathbb{E}_{ \zNew{}^1, \ldots,  \zNew{}^N \sim \mathcal{N}(\mu_{\sNew{t+1} | d_t}, \Sigma_{\sNew{t+1} | \sHist{t}})}\left[ \frac{1}{N}  
		\sum_{\ell = 1}^{N} 
		V^{*}_{t+1} ( \langle \sHist{t+1}, \zHist{t} \oplus  \zNew{}^\ell \rangle )\right] \\
		\displaystyle = \frac{1}{N}  
		\sum_{\ell = 1}^{N}\mathbb{E}_{ \zNew{}^1, \ldots,  \zNew{}^N \sim \mathcal{N}(\mu_{\sNew{t+1} | d_t}, \Sigma_{\sNew{t+1} | \sHist{t}})}[  
		V^{*}_{t+1} ( \langle \sHist{t+1}, \zHist{t} \oplus  \zNew{}^\ell \rangle )] \\
		\displaystyle = \frac{1}{N}  
		\sum_{\ell = 1}^{N}\mathbb{E}_{ \zNew{}^\ell \sim \mathcal{N}(\mu_{\sNew{t+1} | d_t}, \Sigma_{\sNew{t+1} | \sHist{t}})}[  
		V^{*}_{t+1} ( \langle \sHist{t+1}, \zHist{t} \oplus  \zNew{}^\ell \rangle )] \\	
		\displaystyle = \frac{1}{N}  
		\sum_{\ell = 1}^{N}			\mathbb{E}_{\zNew{t+1} | \sNew{t+1}, d_{t}}[ V_{t+1}^* (\langle \sHist{t+1},\zHist{t}\oplus
		\zNew{t+1} \rangle) ]\\
		= \mathbb{E}_{\zNew{t+1} | \sNew{t+1}, d_{t}} [ V_{t+1}^* (\langle \sHist{t+1},\zHist{t}\oplus\zNew{t+1} \rangle)]\\	
		\displaystyle = Q^{*}_t (\sNew{t+1}, d_t) - R(\sNew{t+1}, d_t)
		\end{array}
		\end{equation}
		such that the last equality is due to~\eqref{eq:OptimalValFunDef}. From~\eqref{eq_2_9} and~\eqref{eq_2_10}, 
		\begin{equation}
		|\mathcal{U}_t (\sNew{t+1}, d_t)- Q^{*}_t (\sNew{t+1}, d_t)| = \left|\displaystyle  \mathcal{G}( \zNew{}^1,\ldots, \zNew{}^N)- \mathbb{E}_{ \zNew{}^1, \ldots,  \zNew{}^N \sim \mathcal{N}(\mu_{\sNew{t+1} | d_t}, \Sigma_{\sNew{t+1} | \sHist{t}})}\left[\mathcal{G}(\zNew{}^1, \dots, \zNew{}^N )\right]\right|.
		\label{eq_2_9a}
		\end{equation}	
		%
		The RHS of~\eqref{eq_2_9a} can usually be bounded using a concentration inequality that involves independent Gaussian random variables. However, the components of the multivariate Gaussian vector $\zNew{}^\ell$ are correlated.
		To resolve this complication, we exploit a change of variables trick to make the components independent:
		\begin{equation}
		\label{eq_2_11}
		\zNew{}^\ell = \mu_{\sNew{t+1} | d_t} + \Psi \xNew{}^\ell
		\end{equation}
		for $\ell = 1,\ldots, N$ where $\Psi$ is a $\kappa \times \kappa$ lower triangular matrix satisfying the Cholesky decomposition of the symmetric and positive definite $\Sigma_{\sNew{t+1} | \sHist{t}} = \Psi \Psi^{\top}$ and $\xNew{}^\ell$ is a standard multivariate Gaussian vector with independent components (see Section $53.2.2$ in~\cite{statlect}). 
		
		Define a new auxiliary function $G$ in terms of $\mathcal{G}$
		by plugging \eqref{eq_2_11} 
		into \eqref{eq_2_9}:
		\begin{equation}
		\label{eq_2_12}
		G(\xNew{}^1, \ldots, \xNew{}^{N} ) \triangleq \mathcal{G}(\zNew{}^1, \ldots, \zNew{}^{N})\ . 
		\end{equation}	
		We will first prove that $G$ is Lipschitz continuous in $\xNew{}^1 \oplus \ldots \oplus \xNew{}^N$  with Lipschitz constant ${L_{t+1}(\sHist{t+1})} \sqrt{\mathrm{Tr}(\Sigma_{\sNew{t+1} | \sHist{t}})/{N}}$, which is a sufficient condition for using the Tsirelson-Ibragimov-Sudakov inequality \cite{bouch} to prove the probabilistic bound in  Lemma~\ref{lemma:concentration}. To simplify notations, let $\overline{\xNew{}} \triangleq \xNew{}^1 \oplus \ldots \oplus \xNew{}^N$ and 
		$\overline{\xNew{}}' \triangleq \xNew{}'^1 \oplus \ldots \oplus \xNew{}'^N$.
		Then,
		%
		\begin{equation}
		\label{eq_2_14}
		\begin{array}{l}
		\displaystyle |G(\xNew{}^1, \ldots, \xNew{}^N) -  G({\xNew{}}'^{1}, \ldots, {\xNew{}}'^{N})|  \\
		\displaystyle =  |\mathcal{G}(\zNew{}^1, \ldots, \zNew{}^N) -  \mathcal{G}({\zNew{}}'^{1}, \ldots, {\zNew{}}'^N)|  \\
		\displaystyle\le    \frac{1}{N}  	\sum_{\ell = 1}^{N} 
		\left| V^{*}_{t+1} ( \langle \sHist{t+1}, \zHist{t} \oplus \zNew{}^\ell \rangle) -  V^{*}_{t+1} ( \langle \sHist{t+1}, \zHist{t} \oplus {\zNew{}}'^\ell \rangle) \right| \\
		\displaystyle  \le  \frac{L_{t+1}(\sHist{t+1})}{N} \sum_{\ell= 1}^{N}  \lVert \zNew{}^\ell - {\zNew{}}'^\ell  \rVert \\
		\displaystyle \le  \frac{L_{t+1}(\sHist{t+1})}{N} \sqrt{N}  \lVert \Psi \rVert_F \lVert \overline{\xNew{}} -{\overline{\xNew{}}}' \rVert \\
		\displaystyle  =  \frac{L_{t+1}(\sHist{t+1})}{\sqrt{N}} \lVert \Psi \rVert_F \lVert \overline{\xNew{}} - {\overline{\xNew{}}}' \rVert \\
		\displaystyle =  L_{t+1}(\sHist{t+1})\sqrt{\frac{\mathrm{Tr}(\Sigma_{\sNew{t+1} | \sHist{t}})}{N}}  \lVert \overline{\xNew{}} - {\overline{\xNew{}}}' \rVert 
		\end{array}
		\end{equation}
		where the first equality is due to~\eqref{eq_2_12},
		the last equality follows from a property of Frobenius norm (see Section $10.4.3$ in \cite{cookbook}), 
		the first inequality is due to~\eqref{eq_2_9} and triangle inequality,
		the second inequality is a direct consequence of Theorem~\ref{th:lip-optimal} in Appendix~\ref{lip-optimal-value}, and
		the third inequality follows from 
		$$
		\begin{array}{l}
		\displaystyle  \sum_{\ell= 1}^{N} \lVert \zNew{}^\ell- {\zNew{}'^\ell}\lVert   \\
		\displaystyle =  \sum_{\ell = 1}^{N} \lVert \Psi(\xNew{}^\ell - {\xNew{}}'^\ell)\lVert  \\
		\displaystyle =  \sum_{\ell = 1}^{N} \lVert \Psi(\xNew{}^\ell - {\xNew{}}'^\ell)\lVert_F  \\
		\displaystyle \le  \sum_{\ell = 1}^{N} \lVert \Psi \lVert_F  \lVert \xNew{}^\ell - {\xNew{}}'^\ell\lVert_F  \\
		\displaystyle = \lVert \Psi\lVert_F \sum_{\ell = 1}^{N} \lVert \xNew{}^\ell- {\xNew{}}'^\ell \lVert \\
		\displaystyle \le \sqrt{N} \lVert \Psi \lVert_F \lVert \overline{\xNew{}} -  {\overline{\xNew{}}}' \lVert
		\end{array}
		$$			
		where the first equality is due to~\eqref{eq_2_11}, the first inequality is due to the submultiplicativity of the Frobenius norm (see Section II.$2.1$ in~\cite{Stewart90}), and the last inequality is due to Cauchy-Schwarz inequality.
		Since conditioning does not increase GP posterior variance,
		\begin{equation}
		\mathrm{Tr}(\Sigma_{\sNew{t+1} | \sHist{t}}) \le \mathrm{Tr}(\Sigma_{\sNew{t+1} \sNew{t+1}}) = \kappa( \sigma^2_y + \sigma^2_n)\ .
		\label{app-1-lemma-trace}
		\end{equation}	
		From~\eqref{app-1-lemma-trace} and Lemma~\ref{lem:L-bound}, 
		\begin{equation}
		\label{eq_2_15}
		\begin{array}{l}
		\displaystyle  L_{t+1}(\sHist{t+1}) \sqrt{\mathrm{Tr}(\Sigma_{\sNew{t+1} | \sHist{t}})}\\ 
		\displaystyle = \mathcal{O}(\kappa^{H - t -{1}/{2}}\sqrt{H!/(t+1)!}\ (1+\sigma^2_y/\sigma^2_n)^{H-t-1})
		\ \mathcal{O}(\kappa^{1/2}( \sigma^2_y + \sigma^2_n)^{1/2})\\
		\displaystyle =\mathcal{O}(\kappa^{H - t}\sqrt{H!/(t+1)!}\ \sigma_n (1+\sigma^2_y/\sigma^2_n)^{H-t-1/2})\ .
		\end{array}	 
		\end{equation}
		It follows from~\eqref{eq_2_15} that  
		\begin{equation}
		\label{eq_2_15a}
		K \triangleq \max_{\langle t, \sNew{t+1}, d_t\rangle}  L_{t+1}(\sHist{t+1}) \sqrt{\mathrm{Tr}(\Sigma_{\sNew{t+1} | \sHist{t}})} = \mathcal{O}(\kappa^{H}\sqrt{H!}\ \sigma_n (1+\sigma^2_y/\sigma^2_n)^{H})\ .
		\end{equation}
		Finally,	 
		%
		$$
		\begin{array}{l}
		\displaystyle {P}(| \mathcal{U}_t (\sNew{t+1}, d_t) - Q^{*}_t (\sNew{t+1}, d_t)| > \lambda) \\
		\displaystyle= {P} (| \mathcal{G}(\zNew{}^1, \ldots, \zNew{}^{N}) - \mathbb{E}_{\zNew{}^1, \ldots, \zNew{}^{N}}[\mathcal{G}(\zNew{}^1, \ldots, \zNew{}^{N})] | > \lambda)  \\
		\displaystyle =  {P}( | G(\xNew{}^1, \ldots, \xNew{}^{N}) - \mathbb{E}_{\xNew{}^1, \ldots, \xNew{}^{N}}[G(\xNew{}^1, \ldots, \xNew{}^{N})]| > \lambda ) \\
		\displaystyle \le 2 \exp\left(-\frac{N \lambda^2}{2 L^2_{t+1}(\sHist{t+1}) {\mathrm{Tr}(\Sigma_{\sNew{t+1} | \sHist{t}})}}\right) \\
		\displaystyle \le 2 \exp\left(-\frac{N \lambda^2}{2{K}^2}\right)
		\end{array}			
		$$
		
		where the first equality is due to~\eqref{eq_2_9a}, the second equality is due to \eqref{eq_2_12} above and~\eqref{eq_2_13} below, 
		the first inequality is due to the Tsirelson-Ibragimov-Sudakov inequality that requires $G$ to be Lipschitz continuous in $\xNew{}^1\oplus \ldots\oplus \xNew{}^{N}$ which is shown in~\eqref{eq_2_14} (see Section $5.4$ on page $125$ in~\cite{bouch}), and the last inequality is due to~\eqref{eq_2_15a}. 
		%
		\begin{equation}
		\label{eq_2_13}
		\begin{array}{l}
		\displaystyle	\mathbb{E}_{\zNew{}^1, \dots, \zNew{}^{N}}[\mathcal{G}(\zNew{}^1, \dots, \zNew{}^{N})] \\ 
		\displaystyle = \mathbb{E}_{\zNew{t+1} | \sNew{t+1}, d_{t}}[ V_{t+1}^* (\langle \sHist{t+1},\zHist{t}\oplus \zNew{t+1} \rangle)]\\
		\displaystyle =   \int_{\mathbb{R}^\kappa}  
		V^{*}_{t+1} (\langle \sHist{t+1}, \zHist{t} \oplus \zNew{t+1} \rangle )\ 
		p(\zNew{t+1} | \sNew{t+1}, d_t )\ \text{d}\zNew{t+1} \\ 
		\displaystyle =   \int_{\mathbb{R}^\kappa} V^{*}_{t+1} (\langle \sHist{t+1}, \zHist{t} \oplus  (\mu_{\sNew{t+1} | d_t}+ \Psi \xNew{t+1}) \rangle ) \frac{1}{|\Psi|} p(\xNew{t+1}) \left|\frac{\partial \zNew{t+1}}{\partial \xNew{t+1}}\right| \text{d}\xNew{t+1}   \\ 
		\displaystyle  = \int_{\mathbb{R}^\kappa} V^{*}_{t+1} (\langle \sHist{t+1}, \zHist{t} \oplus (\mu_{\sNew{t+1} | d_t}+ \Psi \xNew{t+1}) \rangle )
		\ p(\xNew{t+1}) \ \text{d}\xNew{t+1} \\ 
		\displaystyle = \mathbb{E}_{\xNew{t+1}} [ V^{*}_{t+1} (\langle \sHist{t+1}, \zHist{t} \oplus (\mu_{\sNew{t+1} | d_t}+ \Psi \xNew{t+1})\rangle )]  \\ 
		\displaystyle= \mathbb{E}_{\xNew{}^1, \ldots, \xNew{}^N} \left[{\frac{1}{N} \sum_{\ell = 1}^{N}  V^{*}_{t+1} (\langle \sHist{t+1}, \zHist{t} \oplus (\mu_{\sNew{t+1} | d_t}+\Psi\xNew{}^\ell ) \rangle ) }\right] \\
		\displaystyle= \mathbb{E}_{\xNew{}^1, \ldots, \xNew{}^{N}}[G(\xNew{}^1, \ldots, \xNew{}^{N})]
		\end{array}
		\end{equation}
		where the first equality is due to~\eqref{eq_2_10}, 
		the third equality follows from~\eqref{eq_2_11}, $p(\zNew{t+1} | \sNew{t+1}, d_t) = p(\xNew{t+1} =\Psi^{-1}(\zNew{t+1} - \mu_{\sNew{t+1} | d_t}) )/|\Psi|$ (see Section~$35.1.2$ in~\cite{statlect}), 
		and an integration by substitution for multiple variables,
		the fourth equality is due to $|{\partial \zNew{t+1}/\partial \xNew{t+1}}| = |\Psi|$, and the last two equalities can be derived in a similar manner as~\eqref{eq_2_10} using~\eqref{eq_2_12}. 
		%
		%
	\end{proof}
	%
	\begin{lemma}
		\label{lemma:3}
		Suppose that the observations $d_{t'}$, $H\in\mathbb{Z}^+$, a budget of $\kappa(H-t')$ input locations for $t'=0,\ldots, H-1$, $\lambda > 0$, and $N \in \mathbb{Z}^+$ are given.		
		The probability of $| \mathcal{U}_t (\sNew{t+1}, d_t) - Q^{*}_t (\sNew{t+1}, d_t)| \le \lambda$ for all tuples $\langle t, \sNew{t+1}, d_t\rangle$ generated at stage $t = t', \ldots, H-1$ by \eqref{approx-policy} to compute $\mathcal{V}_{t'} (d_{t'})$ is at least 
		$$
		1 -  2 \left(N A\right)^H \exp\left(-\frac{N \lambda^2}{2{K}^2}\right)
		$$ 
		where $K$ is previously defined in Lemma~\ref{lemma:concentration}.
	\end{lemma}
	\begin{proof}
		From Lemma~\ref{lemma:concentration},
		$${P}(| \mathcal{U}_t (\sNew{t+1}, d_t) - Q^{*}_t (\sNew{t+1}, d_t)| >  \lambda) \le   2 \exp\left(-\frac{N \lambda^2}{2{K}^2}\right)$$
		for each tuple $\langle t, \sNew{t+1}, d_t\rangle$ generated at stage $t = t', \ldots, H-1$ by \eqref{approx-policy} to compute $\mathcal{V}_{t'} (d_{t'})$. Since there will be no more than $(N A)^H$ tuples $\langle t, \sNew{t+1}, d_t\rangle$  generated at stage $t = t', \ldots, H-1$ by \eqref{approx-policy} to compute $\mathcal{V}_{t'} (d_{t'})$, the probability of $| \mathcal{U}_t (\sNew{t+1}, d_t) - Q^{*}_t (\sNew{t+1}, d_t)| >  \lambda$ for some generated tuple  $\langle t, \sNew{t+1}, d_t\rangle$ is at most $2 (NA)^H \exp(-N \lambda^2 / (2{K}^2))$ by applying the union bound. Lemma~\ref{lemma:3} directly follows. 
	\end{proof}
	%
	\begin{lemma}
		\label{lemma:4}
		Suppose that the observations $d_{t'}$, $H\in\mathbb{Z}^+$, a budget of $\kappa(H-t')$ input locations for $t'=0,\ldots, H-1$, $\lambda > 0$, and $N \in \mathbb{Z}^+$ are given.	
		If 
		\begin{equation}
		\label{eq:lemma4-statement}
		| \mathcal{U}_t (\sNew{t+1}, d_t) - Q^{*}_t (\sNew{t+1}, d_t)| \le \lambda
		\end{equation}
		for all tuples $\langle t, \sNew{t+1}, d_t\rangle$ generated at stage $t = t',  \ldots, H-1$ by \eqref{approx-policy} to compute $\mathcal{V}_{t'} (d_{t'})$,
		then, for all $\sNew{t'+1}\in\mathcal{A}(\sNew{t'})$, 
		\begin{equation}
		\label{eq:lemma4-conclusion}
		| \mathcal{Q}_{t'} (\sNew{t'+1}, d_{t'}) - {Q}^*_{t'} (\sNew{t'+1}, d_{t'}) | \le \lambda(H - t')\ .
		\end{equation} 
	\end{lemma}
	\begin{proof}
		We will give a proof by induction on $t$ that $| \mathcal{Q}_{t} (\sNew{t+1}, d_{t}) - {Q}^*_{t} (\sNew{t+1}, d_{t}) | \le \lambda(H - t)$ for all tuples $\langle t, \sNew{t+1}, d_t\rangle$ generated at stage $t = t' \ldots, H-1$ by~\eqref{approx-policy}  to compute $\mathcal{V}_{t'} (d_{t'})$.
		
		When $t = H -1$, $\mathcal{U}_t (\sNew{t+1}, d_t)  = \mathcal{Q}_t (\sNew{t+1}, d_t) $ in \eqref{eq:lemma4-statement}, by definition. So, \eqref{eq:lemma4-conclusion} holds for the base case. 
		Supposing \eqref{eq:lemma4-conclusion} holds for $t+1$ (i.e. induction hypothesis), we will prove that it holds for $t = t',\ldots, H - 2$:
		$$
		\begin{array}{l}
		\displaystyle | \mathcal{Q}_t(\sNew{t+1}, d_t) - Q^{*}_t(\sNew{t+1}, d_t) |\\
		\displaystyle \le  | \mathcal{Q}_t(\sNew{t+1}, d_t) - \mathcal{U}_t(\sNew{t+1}, d_t) | + 
		| \mathcal{U}_t(\sNew{t+1}, d_t)- Q^{*}_t(\sNew{t+1}, d_t) | \\
		\displaystyle\le | \mathcal{Q}_t(\sNew{t+1}, d_t) - \mathcal{U}_t(\sNew{t+1}, d_t) | + \lambda \\  
		\displaystyle\le \lambda (H - t -1)  + \lambda \\
		\displaystyle = \lambda(H - t) 
		\end{array}
		$$
		where the first and the second inequalities follow, respectively, from the triangle inequality and \eqref{eq:lemma4-statement}, and the last inequality is due to 
		\begin{equation}
		\begin{array}{l}
		\displaystyle | \mathcal{Q}_t(\sNew{t+1}, d_t) - \mathcal{U}_t(\sNew{t+1}, d_t) | \\
		\displaystyle \le  \frac{1}{N} \sum_{\ell= 1}^{N} |	\mathcal{V}_{t+1} (\langle \sHist{t+1}, \zHist{t} \oplus \zNew{}^\ell\rangle) -  V^{*}_{t+1} (\langle \sHist{t+1}, \zHist{t} \oplus \zNew{}^\ell\rangle)| \\
		\displaystyle  \le   \frac{1}{N} \sum_{\ell = 1}^{N}  \max_{\sNew{t+2} \in \Adom(\sNew{t+1} )} |  \mathcal{Q}_{t+1}(\sNew{t+2}, \langle \sHist{t+1}, \zHist{t} \oplus \zNew{}^\ell\rangle)  -  Q_{t+1}^{*}(\sNew{t+2},\langle \sHist{t+1}, \zHist{t} \oplus \zNew{}^\ell\rangle) | \\
		\displaystyle\le  \lambda(H - t -1) 
		\end{array}
		\end{equation}
		where the first inequality is due to  triangle inequality and 
		the last inequality follows from induction hypothesis. 
		
		Finally, when $t = t'$,  $| \mathcal{Q}_{t'} (\sNew{t'+1}, d_{t'}) - {Q}^*_{t'} (\sNew{t'+1}, d_{t'}) | \le \lambda(H - t')$~\eqref{eq:lemma4-conclusion} for all $\sNew{t'+1}\in\mathcal{A}(\sNew{t'})$
		since $d_t=d_{t'}$.
	\end{proof}
	%
	\emph{Main proof.} It follows immediately from Lemmas~\ref{lemma:3} and~\ref{lemma:4} that the probability of $| \mathcal{Q}_t (\sNew{t+1}, d_t) - {Q}^*_t (\sNew{t+1}, d_t) | \le \lambda H$ for all $\sNew{t+1} \in \Adom(\sNew{t} )$ is at least $1 -  2 (N A)^H \exp(-N \lambda^2 / (2{K}^2))$ where $K$ is previously defined in Lemma~\ref{lemma:concentration}. 
	
	To guarantee that the probability of $| \mathcal{Q}_t (\sNew{t+1}, d_t) - {Q}^*_t (\sNew{t+1}, d_t) | \le \lambda H$ for all $\sNew{t+1} \in \Adom(\sNew{t} )$ is at least $1 - \delta$, the value of $N$ has to satisfy the following inequality:
	$$1 -  2 \left(N A |\right)^H \exp\left(-\frac{N \lambda^2}{2{K}^2}\right) \ge 1 - \delta\ ,$$
	which is equivalent to
	\begin{equation}
	\label{eq:th1-1}
	N \ge \frac{2 K^2}{\lambda^2} \left( H \log{N}  + H \log\left(A \right) + \log{\frac{2}{\delta}} \right).
	\end{equation}
	Using the identity $\log{N} \le \nu N - \log{\nu} - 1$ for $\nu = \lambda^2/(4 K^2 H)$, the RHS of \eqref{eq:th1-1} can be bounded from above by 
	$$\frac{N}{2} + \frac{2 K^2}{\lambda^2} \left( H \log\left(\frac{4 K^2 H  A}{e\lambda^2}\right) + \log{\frac{2}{\delta}} \right).$$
	Therefore, to satisfy~\eqref{eq:th1-1}, it suffices to determine the value of $N$ such that 
	$$
	N \ge \frac{N}{2} + \frac{2 K^2}{\lambda^2} \left( H \log\left(\frac{4 K^2 H A}{e\lambda^2}\right) + \log{\frac{2}{\delta}} \right)
	$$
	by setting 
	$$
	N = \frac{4 K^2}{\lambda^2} \left( H \log\left(\frac{4 K^2 H A}{e\lambda^2}\right) + \log{\frac{2}{\delta}} \right)
	$$
	where $K$ is previously defined in Lemma~\ref{lemma:concentration}. 	
	By assuming $H$, $\sigma^2_y$, and $\sigma^2_n$ as constants,
	$$
	N =\mathcal{O} \left(\frac{\kappa^{2H}}{\lambda^2} \log\left(\frac{\kappa A}{\delta\lambda}\right)\right).
	$$
\end{proof}

\section{Proof of Theorem~\ref{th-mle_bound}}
\label{sec:th-mle_bound_proof}

\begin{proof}
	Similar to~\eqref{eq:u-function}, the following intermediate function is introduced: 
	\begin{equation}
	\label{eq:u-mle}
	\mathds{U}_t (\sNew{t+1}, d_t) \triangleq R(\sNew{t+1}, d_t) + 
	V^{*}_{t+1} (\langle \sHist{t+1}, \zHist{t} \oplus \mu_{\sNew{t+1} | d_t} \rangle ).
	\end{equation}
	for $t =0, \ldots, H-1$. 
	
	We will first bound $| Q^{*}_t (\sNew{t+1}, d_t) - \mathds{U}_t (\sNew{t+1}, d_t)|$: 
	\begin{equation} 
	\label{eq:ap-22}
	\begin{array}{l}
	\displaystyle | Q^{*}_t (\sNew{t+1}, d_t) - \mathds{U}_t (\sNew{t+1}, d_t)| \\
	\displaystyle =  \left| \int_{\mathbb{R}^{\kappa}}{ \left(  V^*_{t+1}(\langle \sHist{t+1}, \zHist{t} \oplus \zNew{t+1}  \rangle) - V^{*}_{t+1} (\langle \sHist{t+1}, \zHist{t} \oplus \mu_{\sNew{t+1} | d_t} \rangle )  \right)} \ 
	{ p(\zNew{t+1} | \sNew{t+1}, d_t)}\ \text{d}\zNew{t+1} \right| \\ 
	\displaystyle \le  L_{t+1}(\sHist{t+1}) \  \int_{\mathbb{R}^{\kappa}}{ \lVert \zNew{t+1} - \mu_{\sNew{t+1} | d_t} \rVert \ p(\zNew{t+1} | \sNew{t+1}, d_t) }\ \text{d} \zNew{t+1} \\ 
	\displaystyle =  L_{t+1}(\sHist{t+1}) \  \int_{\mathbb{R}^{\kappa}} \lVert \Psi \mathbf{x}_{t+1}\rVert  \frac{1}{|\Psi|} p(\xNew{t+1}) \left|\frac{\partial \zNew{t+1}}{\partial \xNew{t+1}}\right|  \text{d} \mathbf{x}_{t+1}  \\ 
	\displaystyle =  L_{t+1}(\sHist{t+1}) \  \int_{\mathbb{R}^{\kappa}} \lVert \Psi \mathbf{x}_{t+1}\rVert \ p(\xNew{t+1}) \ \text{d} \mathbf{x}_{t+1}  \\ 
	\displaystyle \le  L_{t+1}(\sHist{t+1}) \  \lVert \Psi \lVert_{F} \ \mathbb{E}_{\xNew{t+1}}[\lVert\xNew{t+1} \lVert] \\ 
	\displaystyle =   L_{t+1}(\sHist{t+1}) \ \sqrt{\mathrm{Tr}(\Sigma_{\sNew{t+1} | \sHist{t}})} \  \mathbb{E}_{\xNew{t+1}}[\lVert\xNew{t+1} \lVert] \\ 
	\displaystyle =\mathcal{O}(\kappa^{H - t}\sqrt{H!/(t+1)!}\ \sigma_n (1+\sigma^2_y/\sigma^2_n)^{H-t-1/2})
	\ \mathbb{E}_{\xNew{t+1}}[\lVert\xNew{t+1} \lVert] \\
	\displaystyle =\mathcal{O}(\kappa^{H - t+1/2}\sqrt{H!/(t+1)!}\ \sigma_n (1+\sigma^2_y/\sigma^2_n)^{H-t-1/2})
	\end{array}
	\end{equation}
	where the first equality is due to \eqref{eq:OptimalValFunDef} and \eqref{eq:u-mle}, 
	the first inequality is a direct consequence of Theorem~\ref{th:lip-optimal} in Appendix~\ref{lip-optimal-value},
	the second equality follows from~\eqref{eq_2_11}, $p(\zNew{t+1} | \sNew{t+1}, d_t) = p(\xNew{t+1} =\Psi^{-1}(\zNew{t+1} - \mu_{\sNew{t+1} | d_t}) )/|\Psi|$ (see Section~$35.1.2$ in~\cite{statlect}), 
	and an integration by substitution for multiple variables,
	the third equality is due to $|{\partial \zNew{t+1}/\partial \xNew{t+1}}| = |\Psi|$, 
	the second inequality is due to the submultiplicativity of the Frobenius norm (see Section II.$2.1$ in~\cite{Stewart90}),
	the fourth equality follows from a property of Frobenius norm (see Section $10.4.3$ in \cite{cookbook}),
	the second last equality is due to~\eqref{eq_2_15}, and 
	the last equality follows from $\mathbb{E}_{\xNew{t+1}}[\lVert\xNew{t+1} \lVert]\le\sqrt{\kappa}$ (see Section $3.1$ in \cite{chandrasekaran2012}).
	
	We will now give a proof by induction on $t$ that  
	\begin{equation}
	\label{bobo}
	| Q^{*}_t (\sNew{t+1}, d_t) - \mathds{Q}_t (\sNew{t+1}, d_t)| \le \theta_t
	\end{equation}
	for all $\sNew{t+1} \in \Adom(\sNew{t})$ where 
	\begin{equation}
	\theta_t \triangleq \mathcal{O}(\kappa^{H - t+1/2}\sqrt{H!/(t+1)!}\ \sigma_n (1+\sigma^2_y/\sigma^2_n)^{H-t-1/2})\ .
	\label{grabtaxi}
	\end{equation} 
	When $t = H-1$, $Q_t^*(\sNew{t+1}, d_t)  -  \mathds{Q}_t(\sNew{t+1}, d_t) = 0$.
	So,~\eqref{bobo} holds for the base case. 
	Supposing~\eqref{bobo} holds for $t+1$ (i.e. induction hypothesis), we will prove that it holds for $t = 0,\ldots, H - 2$:
	\begin{equation} 
	\label{eq:app-36}
	\begin{array}{l}
	\displaystyle   | Q^{*}_t (\sNew{t+1}, d_t) - \mathds{Q}_t (\sNew{t+1}, d_t)| \\ 
	\displaystyle \le  	| Q^{*}_t (\sNew{t+1}, d_t) - \mathds{U}_t (\sNew{t+1}, d_t)| + 	| \mathds{U}_t (\sNew{t+1}, d_t) - \mathds{Q}_t (\sNew{t+1}, d_t) | \\
	\displaystyle  \le \mathcal{O}(\kappa^{H - t+1/2}\sqrt{H!/(t+1)!}\ \sigma_n (1+\sigma^2_y/\sigma^2_n)^{H-t-1/2})  \\
	\displaystyle\quad+  |V^{*}_{t+1} (\langle \sHist{t+1}, \zHist{t} \oplus \mu_{\sNew{t+1} | d_t} \rangle ) - \mathds{V}_{t+1} (\langle \sHist{t+1}, \zHist{t} \oplus \mu_{\sNew{t+1} | d_t} \rangle )| \\
	
	\displaystyle  \le \mathcal{O}(\kappa^{H - t+1/2}\sqrt{H!/(t+1)!}\ \sigma_n (1+\sigma^2_y/\sigma^2_n)^{H-t-1/2}) \\
	\displaystyle\quad +  \max_{\sNew{t+2} \in \Adom(\sNew{t+1})} |  Q_{t+1}^{*}(\sNew{t+2}, \langle \sHist{t+1}, \zHist{t} \oplus \mu_{\sNew{t+1} | d_t} \rangle)  -  \mathds{Q}_{t+1}(\sNew{t+2}, \langle \sHist{t+1}, \zHist{t} \oplus \mu_{\sNew{t+1} | d_t} \rangle) | \\
	\displaystyle \le \mathcal{O}(\kappa^{H - t+1/2}\sqrt{H!/(t+1)!}\ \sigma_n (1+\sigma^2_y/\sigma^2_n)^{H-t-1/2}) +  \theta_{t+1}\\
	\displaystyle = \mathcal{O}(\kappa^{H - t+1/2}\sqrt{H!/(t+1)!}\ \sigma_n (1+\sigma^2_y/\sigma^2_n)^{H-t-1/2}) \\
	\displaystyle = \theta_t
	\end{array}
	\end{equation}
	where the first inequality is due to triangle inequality, 
	the second inequality is due to~\eqref{eq:ap-22}, \eqref{ml-policy}, and \eqref{eq:u-mle},
	and the last inequality is due to the induction hypothesis. 
	
	Finally, by assuming $H$, $\sigma^2_y$, and $\sigma^2_n$ as constants, it follows from~\eqref{eq:app-36} that $\theta \triangleq \max_{t}\theta_t = \mathcal{O}(\kappa^{H +1/2})$ and Theorem~\ref{th-mle_bound} follows.
\end{proof}

\section{Proof of Theorem~\ref{th:expected}}
\label{th:expected-proof}
We first formally discuss the implications of our tractable choice of the if condition in~\eqref{eq_4_8}
for theoretically guaranteeing the performance of our $\epsilon$-Macro-GPO policy $\pi^\epsilon$:

{\bf I.} In the likely event (with a high probability of at least $1-\delta$) that $| \mathcal{Q}_t (\sNew{t+1}, d_t) - {Q}^*_t (\sNew{t+1}, d_t) | \le \lambda H$ for all $\sNew{t+1} \in \Adom(\sNew{t})$ (Theorem~\ref{th:1_new}), 
\begin{equation*}
\begin{array}{l}
\displaystyle |\mathcal{Q}_t (\sNew{t+1}, d_t) -\mathds{Q}_t(\sNew{t+1}, d_t)|\\
\displaystyle \le | \mathcal{Q}_t (\sNew{t+1}, d_t) - {Q}^*_t (\sNew{t+1}, d_t) | + |{Q}^*_t (\sNew{t+1}, d_t) - \mathds{Q}_t \hspace{-0.3mm}(\sNew{t+1}, d_t)| \\
\displaystyle \leq \lambda H  + \theta
\end{array}
\end{equation*} 
for all $\sNew{t+1} \in \Adom(\sNew{t})$ such that the first inequality is due to triangle inequality and the second inequality is due to Theorems~\ref{th:1_new} and~\ref{th-mle_bound}. 
Consequently, according to~\eqref{eq_4_8}, $Q^{\epsilon}_t (\sNew{t+1}, d_{t}) = \mathcal{Q}_t (\sNew{t+1}, d_t)$ for all $\sNew{t+1} \in \Adom(\sNew{t})$ and $\pi^\epsilon(d_{t})$ thus selects the same macro-action as the policy induced by stochastic sampling~\eqref{approx-policy}.

{\bf II.}  In the unlikely event (with an arbitrarily small probability of at most $\delta$) that 
$\mathcal{Q}_t (\sNew{t+1}, d_t)$~\eqref{approx-policy} is unboundedly far from ${Q}^*_t (\sNew{t+1}, d_t)$~\eqref{eq:OptimalValFunDef}  
(i.e., $| \mathcal{Q}_t (\sNew{t+1}, d_t) - {Q}^*_t (\sNew{t+1}, d_t) | > \lambda H$) for some $\sNew{t+1} \in \Adom(\sNew{t})$, 
$\pi^\epsilon(d_{t})$~\eqref{eq_4_8} guarantees that, for any selected macro-action $\sNew{t+1} \in \Adom(\sNew{t})$,
$$
\hspace{-1.9mm}
\begin{array}{l}
| {Q}^{\epsilon}_t (\sNew{t+1}, d_t) - {Q}^*_t (\sNew{t+1}, d_t) |\\
=\hspace{-1mm}
\begin{cases}
|\mathcal{Q}_t \hspace{-0.3mm}(\sNew{t+1}, d_t)\hspace{-0.8mm} -\hspace{-0.6mm} {Q}^*_t \hspace{-0.3mm}(\sNew{t+1}, d_t) | &\hspace{-1.05mm}  
\begin{array}{l}
\text{if } |\mathcal{Q}_t\hspace{-0.3mm} (\sNew{t+1}, d_t)\hspace{-0.8mm} -\hspace{-0.6mm} \mathds{Q}_t\hspace{-0.3mm}(\sNew{t+1}, d_t)| \le \lambda H  + \theta , 
\end{array}
\\
|\mathds{Q}_t \hspace{-0.3mm}(\sNew{t+1}, d_t)\hspace{-0.8mm} -\hspace{-0.6mm} {Q}^*_t \hspace{-0.3mm}(\sNew{t+1}, d_t) |     & \hspace{0.65mm}\text{otherwise};
\end{cases}
\vspace{0.5mm}\\
\le\hspace{-1mm}
\begin{cases}
\hspace{-1.78mm}
\begin{array}{l}
|\mathcal{Q}_t \hspace{-0.3mm}(\sNew{t+1}, d_t)\hspace{-0.8mm} -\hspace{-0.6mm} \mathds{Q}_t \hspace{-0.3mm}(\sNew{t+1}, d_t) |
+ |\mathds{Q}_t \hspace{-0.3mm}(\sNew{t+1}, d_t)\hspace{-0.8mm} -\hspace{-0.6mm} {Q}^*_t \hspace{-0.3mm}(\sNew{t+1}, d_t) |
\end{array}
&\hspace{-5.5mm}  
\begin{array}{l}
\text{if } |\mathcal{Q}_t\hspace{-0.3mm} (\sNew{t+1}, d_t)\hspace{-0.8mm} -\hspace{-0.6mm} \mathds{Q}_t\hspace{-0.3mm}(\sNew{t+1}, d_t)| 
\le \lambda H  + \theta , 
\end{array}
\\
\theta     & \hspace{-3.8mm}\text{otherwise};
\end{cases}
\\
\le \lambda H +2\theta\ ,\quad\text{by triangle inequality and Theorem~\ref{th-mle_bound}.}
\end{array}
$$
%
The above two implications of our tractable choice of the if condition in~\eqref{eq_4_8} are central to establishing our main result deterministically  	
bounding the \emph{expected} performance loss of $\pi^\epsilon$  relative to that of Bayes-optimal Macro-GPO policy $\pi^*$, that is, policy $\pi^\epsilon$ is $\epsilon$-Bayes-optimal.

The following lemmas are needed to prove our main result here:\\
\begin{lemma}
	Suppose that the observations $d_{t}$, $H\in\mathbb{Z}^+$, a budget of $\kappa(H-t)$ input locations for $t=0,\ldots, H-1$, $\delta\in(0, 1)$, and $\lambda > 0$ are given.
	Then, the probability of
	$$| {Q}^*_t (\pi^*(d_t), d_t) - {Q}^*_t (\pi^\epsilon(d_{t}), d_t) | \le 2 \lambda H$$ 
	is at least $1-\delta$ by setting $N$ according to that in Theorem~\ref{th:1_new}.
	\label{smith}
\end{lemma}
\begin{proof}
	$$
	\begin{array}{l}
	\displaystyle {Q}^*_t (\pi^*(d_t), d_t) - {Q}^*_t (\pi^\epsilon(d_{t}), d_t) \\
	\displaystyle\le {Q}^*_t (\pi^*(d_t), d_t) - \mathcal{Q}_t (\pi^\epsilon(d_{t}), d_t) + \lambda H\\
	\displaystyle\le\max_{\sNew{t+1} \in \Adom(\sNew{t})} | \mathcal{Q}_t (\sNew{t+1}, d_t) - {Q}^*_t (\sNew{t+1}, d_t) | + \lambda H\\
	\displaystyle  \le\lambda H + \lambda H\\
	\displaystyle = 2 \lambda H
	\end{array}
	$$
	where the first and last inequalities are due to Theorem~\ref{th:1_new} and the second inequality is further due to implication I.
\end{proof}

\begin{lemma}
	Suppose that the observations $d_{t}$, $H\in\mathbb{Z}^+$, a budget of $\kappa(H-t)$ input locations for $t=0,\ldots, H-1$, $\delta\in(0, 1)$, and $\lambda > 0$ are given.
	Then, 
	$$Q^*_t (\pi^*(d_t), d_t) - \mathbb{E}_{\pi^{\epsilon}(d_t)}[Q^*_t (\pi^{\epsilon}(d_t), d_t)]\le 2\lambda H + 4\delta\theta$$ 
	where $\theta$ is previously defined in Theorem~\ref{th-mle_bound}.
	\label{lem:fixed_policy_beta}	
\end{lemma}	
\begin{proof}
	By Lemma~\ref{smith}, the probability of $| {Q}^*_t (\pi^*(d_t), d_t) - {Q}^*_t (\pi^\epsilon(d_{t}), d_t) | \le 2 \lambda H$ is at least $1-\delta$. 
	Otherwise, the probability of $| {Q}^*_t (\pi^*(d_t), d_t) - {Q}^*_t (\pi^\epsilon(d_{t}), d_t) | > 2 \lambda H$ is at most $\delta$. In the latter case,
	\begin{equation}
	\begin{array}{l}
	\displaystyle | {Q}^*_t (\pi^*(d_t), d_t) - {Q}^*_t (\pi^\epsilon(d_{t}), d_t) |\\
	\displaystyle\le | {Q}^*_t (\pi^*(d_t), d_t) -{Q}^{\epsilon}_t (\pi^\epsilon(d_{t}), d_t)| + |{Q}^{\epsilon}_t (\pi^\epsilon(d_{t}), d_t) -{Q}^*_t (\pi^\epsilon(d_{t}), d_t) |\\
	\displaystyle\le \max_{\sNew{t+1} \in \Adom(\sNew{t})}  | {Q}^{\epsilon}_t (\sNew{t+1}, d_t) - {Q}^*_t (\sNew{t+1}, d_t) |  + \lambda H +2\theta\\
	\displaystyle\le \lambda H +2\theta + \lambda H +2\theta\\
	\displaystyle = 2\lambda H +4\theta
	\end{array}
	\label{smith2}
	\end{equation}
	where the first inequality is due to triangle inequality and the last two inequalities are due to implication II.
	Recall that $\pi^{\epsilon}$ is a stochastic policy due to its use of stochastic sampling in $\mathcal{Q}_t$~\eqref{approx-policy}, which implies that 
	$\pi^{\epsilon}(d_t)$ is a random variable. 
	Then,	
	$$
	\begin{array}{l}
	\displaystyle Q^*_t (\pi^*(d_t), d_t) - \mathbb{E}_{\pi^{\epsilon}(d_t)}[Q^*_t (\pi^{\epsilon}(d_t), d_t)]\\
	\displaystyle = \mathbb{E}_{\pi^{\epsilon}(d_t)}[ Q^*_t (\pi^*(d_t), d_t) - Q^*_t (\pi^{\epsilon}(d_t), d_t)]\\
	\le (1-\delta) (2\lambda H) + \delta (2\lambda H +4\theta)\\
	= 2\lambda H + 4\delta\theta
	\end{array}
	$$
	where the expectation is with respect to random variable $\pi^{\epsilon}(d_t)$ and the inequality follows from Lemma~\ref{smith} and~\eqref{smith2}.
\end{proof}
\emph{Main proof.} 
We will give a proof by induction on $t$ that 
\begin{equation}
V_t^*(d_t) - \mathbb{E}_{\pi^\epsilon}[ V_t^{\pi^{\epsilon}}(d_t)] \le (2\lambda H + 4\delta\theta) (H - t)\ .
\label{beta_V}
\end{equation}
When $t = H-1$ (i.e., base case), 
$$
\begin{array}{l}
\displaystyle  V_{H-1}^*(d_{H-1}) - \mathbb{E}_{\pi^\epsilon} [ V_{H-1}^{\pi^{\epsilon}}(d_{H-1}) ] \\ 
\displaystyle = Q_{H-1}^*(\pi^*(d_{H-1}), d_{H-1}) -  \mathbb{E}_{\pi^\epsilon} [ Q_t^{\pi^{\epsilon}} (\pi^\epsilon(d_{H-1}), d_{H-1})] \\ 
\displaystyle = Q_{H-1}^*(\pi^*(d_{H-1}), d_{H-1}) -  \mathbb{E}_{\pi^\epsilon(d_{H-1})} [ R (\pi^\epsilon(d_{H-1}), d_{H-1})] \\ 
\displaystyle =   
Q_{H-1}^*(\pi^*(d_{H-1}), d_{H-1}) -  \mathbb{E}_{\pi^\epsilon(d_{H-1})} [ Q_t^{*} (\pi^\epsilon(d_{H-1}), d_{H-1})] \\ 		
\displaystyle \le 2\lambda H + 4\delta\theta
\end{array}
$$
where the first equality is due to~\eqref{eq:general-policy} and~\eqref{eq:OptimalValFunDef}, the second equality is due to~\eqref{eq:general-policy}, the third equality is due to~\eqref{eq:OptimalValFunDef}, and the inequality is due to Lemma~\ref{lem:fixed_policy_beta}. 
So,~\eqref{beta_V} holds for the base case.
Supposing~\eqref{beta_V} holds for $t+1$ (i.e., induction hypothesis), we will prove that it holds for $t = 0, \ldots, H-2$:			
\begin{equation}
\begin{array}{l}
\displaystyle  V_t^*(d_t) - \mathbb{E}_{\pi^\epsilon}[ V_t^{\pi^{\epsilon}}(d_t) ]\\ 
\displaystyle = Q_t^*(\pi^*(d_t), d_t) -  \mathbb{E}_{\pi^\epsilon} [ Q_t^{\pi^{\epsilon}} (\pi^\epsilon(d_t), d_t)]\\ 
\displaystyle = Q_t^*(\pi^*(d_t), d_t) -  \mathbb{E}_{\pi^{\epsilon}} [Q_t^* (\pi^{\epsilon}(d_t), d_t)]  + \mathbb{E}_{\pi^{\epsilon}}[ Q_t^* (\pi^{\epsilon}(d_t), d_t) ] -  \mathbb{E}_{\pi^{\epsilon}}[ Q_t^{\pi^{\epsilon}} (\pi^\epsilon(d_t), d_t)]\\ 
\displaystyle = Q_t^*(\pi^*(d_t), d_t) -  \mathbb{E}_{\pi^{\epsilon}(d_t)} [Q_t^* (\pi^{\epsilon}(d_t), d_t)]  +   \mathbb{E}_{\pi^\epsilon} [ Q_t^* (\pi^\epsilon(d_t), d_t)   - Q_t^{\pi^\epsilon} (\pi^\epsilon(d_t), d_t)] \\ 
\displaystyle \le  2\lambda H + 4\delta\theta +  \mathbb{E}_{\pi^\epsilon} [ Q_t^* (\pi^\epsilon(d_t), d_t)   - Q_t^{\pi^\epsilon} (\pi^\epsilon(d_t), d_t)] \\ 
\displaystyle  = 2\lambda H + 4\delta\theta +  \mathbb{E}_{\pi^\epsilon} [ \mathbb{E}_{ \zNew{t+1} | \pi^\epsilon(d_t), d_t} [V^*_{t+1} ( \langle \sHist{t} \oplus \pi^\epsilon(d_t),  \zHist{t} \oplus \zNew{t+1} \rangle) -  V^{\pi^\epsilon}_{t+1} (\langle \sHist{t} \oplus \pi^\epsilon(d_t),  \zHist{t} \oplus \zNew{t+1} \rangle)] ]\\ 
\displaystyle  = 2\lambda H + 4\delta\theta +  \mathbb{E}_{\pi^\epsilon(d_t)} [ \mathbb{E}_{\zNew{t+1} | \pi^\epsilon(d_t), d_t} [V^*_{t+1} ( \langle \sHist{t} \oplus \pi^\epsilon(d_t),  \zHist{t} \oplus \zNew{t+1} \rangle) -  \mathbb{E}_{\pi^\epsilon}[V^{\pi^\epsilon}_{t+1} (\langle \sHist{t} \oplus \pi^\epsilon(d_t),  \zHist{t} \oplus \zNew{t+1} \rangle)]] ]\\ 
\displaystyle \le  2\lambda H + 4\delta\theta + \mathbb{E}_{\pi^\epsilon(d_t)} [ \mathbb{E}_{\zNew{t+1} | \pi^\epsilon(d_t), d_t} [(2\lambda H + 4\delta\theta)(H - t - 1)]]\\
\displaystyle  = (2\lambda H + 4\delta\theta) (H- t)
\end{array}
\label{eq:smith3}		
\end{equation}
where the first and fourth equalities are due to~\eqref{eq:general-policy} and~\eqref{eq:OptimalValFunDef}, 
the first inequality is due to Lemma~\ref{lem:fixed_policy_beta}, 
and the last inequality is due to the induction hypothesis.

From~\eqref{eq:smith3}, when $t = 0$,
$$
V_0^*(d_0) - \mathbb{E}_{\pi^\epsilon}  [V_0^{\pi^{\epsilon}}(d_0)] \le 2H (\lambda H  + 2 \delta  \theta)\ .
$$
Let $\epsilon = 2H (\lambda H  + 2 \delta  \theta)$ by setting $\lambda = {\epsilon}/({4 H^2})$ and $\delta = {\epsilon}/({8 \theta H})$. Consequently, using Lemma~\ref{smith} and $\theta = \mathcal{O}( \kappa^{H + 1/2})$ previously defined in Theorem~\ref{th-mle_bound}, 	 
$$
N  = \mathcal{O} \left(  \frac{\kappa^{2H}}{\epsilon^2}  \log{\frac{\kappa A}{\epsilon}} \right)\ .
$$
%
\section{Anytime $\epsilon$-Macro-GPO}
\label{sec:anytime}
\subsection{Pseudocode}
\label{sec:pseudo}
The pseudocode is described in Algorithm~\ref{alg:Anytime} and explained below. The essential steps of the main function Anytime-$\epsilon$-Macro-GPO are as follows:
\begin{enumerate}
	\item Preprocessing (lines $40$-$42$): Compute $\Sigma_{\sNew{t+1} | \sHist{t}}$~\eqref{gp-posteriors}, $L_{t+1}(\sHist{t+1})$ (Definition~\ref{definition-L}), and $\mathds{Q}_t (\sNew{t+1}, d_t)$~\eqref{ml-policy} for all $\sHist{t+1}$ reachable from $\mathbf{s}_{0}$ and $t=0,\ldots,H-1$, and set $\theta$ according to Theorem~\ref{th-mle_bound} (Appendix~\ref{sec:th-mle_bound_proof});
	\item Iteratively and incrementally expand the partially constructed search tree rooted at node $d_0$ 
	by calling the recursive function ConstructTree (lines $44$-$45$) so as to tighten the upper heuristic bound $\overline{V}^{*}_{0}(d_0)$ and lower heuristic bound $\underline{V}^{*}_{0}(d_0)$ of $V^{*}_{0}(d_0)$, hence reducing the gap $\omega \triangleq \overline{V}^{*}_0(d_0) - \underline{V}^{*}_0(d_0)$ (line $46$); and
	\item Compute our anytime $\langle\omega,\epsilon\rangle$-Macro-GPO policy $\pi^{\omega\epsilon}(d_0)$ according to~\eqref{eq:anytime_policy} (lines $47$-$51$).
\end{enumerate}
The recursive function ConstructTree traverses down the partially constructed search tree by repeatedly selecting nodes $d_t$ with the largest uncertainty of their corresponding values $V^*_t(d_t)$ (i.e., largest gap $\overline{V}^{*}_t(d_t) - \underline{V}^{*}_t(d_t)$ between the upper and lower heuristic bounds of $V^*_t(d_t)$ so as to tighten them) until an unexplored node is reached.
%
%
Specifically, if the function ConstructTree selects an explored node $d_t$, then the following steps are  performed:
\begin{enumerate}
	\item Choose the macro-action $\sNew{t+1}$ with the tightest lower heuristic bound $\underline{Q}^{*}_t(\sNew{t+1}, d_t)$ of $Q^*_t(\sNew{t+1},d_t)$ (line $26$);
	\item Retrieve the samples $\{\zNew{}^\ell \}_{\ell = 1 ,\ldots, N}$ previously generated by function ExpandTree at node $d_t$ for macro-action $\sNew{t+1}$ (line $27$);
	\item Recursively and incrementally expand the partially constructed sub-tree rooted at node $\langle \sHist{t+1}, \zHist{t} \oplus \zNew{}^{\ell^*} \rangle$ with the largest uncertainty of its corresponding value $V^*_{t+1}(\langle \sHist{t+1}, \zHist{t} \oplus \zNew{}^{\ell^*} \rangle)$ (i.e., largest gap $\overline{V}^{*}_{t+1}(\langle \sHist{t+1}, \zHist{t} \oplus \zNew{}^{\ell^*} \rangle)  - \underline{V}^{*}_{t+1}(\langle \sHist{t+1}, \zHist{t} \oplus \zNew{}^{\ell^*} \rangle)$ between the upper and lower heuristic bounds of ${V}^{*}_{t+1}(\langle \sHist{t+1}, \zHist{t} \oplus \zNew{}^{\ell^*} \rangle)$ so as to tighten them) 
	(lines $28$-$29$);
	\item Use the tightened upper and lower heuristic bounds of ${V}^{*}_{t+1}(\langle \sHist{t+1}, \zHist{t} \oplus \zNew{}^{\ell^*} \rangle)$ at node $\langle \sHist{t+1}, \zHist{t} \oplus \zNew{}^{\ell^*} \rangle$ to refine the heuristic bounds at its siblings (see Corollary~\ref{th:anytime-2}) by exploiting the Lipschitz continuity of $V^*_{t+1}$ (Theorem~\ref{th:lip-optimal}) (line $30$); and
	\item Backpropagate the tightened/refined heuristic bounds at node $\langle \sHist{t+1}, \zHist{t} \oplus \zNew{}^{\ell^*} \rangle$ and its siblings to that at their parent node $d_t$
	(lines $31$-$35$).
\end{enumerate}
Otherwise, the function ConstructTree selects an unexplored node $d_t$ and constructs a ``minimal'' sub-tree rooted at node $d_t$ via the function ExpandTree (line $38$), the latter of which involves the following steps:	 	 
\begin{enumerate}
	\item For every macro-action $\sNew{t+1} \in \Adom(\sNew{t})$,
	\begin{enumerate}
		\item Draw $N$ i.i.d. multivariate Gaussian vectors $\{\zNew{}^\ell \}_{\ell = 1 ,\ldots, N}$ from GP posterior belief $p(\zNew{t+1} | \sNew{t+1}, d_t )$ (line $5$);
		\item For every child node $\langle \sHist{t+1}, \zHist{t} \oplus \zNew{}^\ell \rangle$, initialize the upper and lower heuristic bounds of ${V}^{*}_{t+1}(\langle \sHist{t+1}, \zHist{t} \oplus \zNew{}^{\ell} \rangle)$ (lines $6$-$8$) using Theorem~\ref{th-mle_bound}:
		\begin{equation}				
		\begin{array}{l}
		\displaystyle |\mathds{V}_{t+1} (\langle \sHist{t+1}, \zHist{t} \oplus \zNew{}^\ell \rangle) - {V}^*_{t+1} (\langle \sHist{t+1}, \zHist{t} \oplus \zNew{}^\ell \rangle) |\\
		\displaystyle = |\max_{\sNew{t+2} \in \Adom(\sNew{t+1})} \mathds{Q}_{t+1} (\sNew{t+2}, \langle \sHist{t+1}, \zHist{t} \oplus \zNew{}^\ell \rangle) - \max_{\sNew{t+2} \in \Adom(\sNew{t+1}))} {Q}^*_{t+1} (\sNew{t+2}, \langle \sHist{t+1}, \zHist{t} \oplus \zNew{}^\ell \rangle) |\\
		\displaystyle \le \max_{\sNew{t+2} \in \Adom(\sNew{t+1})} | \mathds{Q}_{t+1} (\sNew{t+2}, \langle \sHist{t+1}, \zHist{t} \oplus \zNew{}^\ell \rangle) - {Q}^*_{t+1} (\sNew{t+2}, \langle \sHist{t+1}, \zHist{t} \oplus \zNew{}^\ell \rangle) | \\
		\le \theta_{t+1}
		\end{array}
		\label{initbound}
		\end{equation}		
		where the equality is due to~\eqref{eq:OptimalValFunDef} and~\eqref{ml-policy}, $\theta_{t+1}$ is previously defined in~\eqref{grabtaxi}, and the last inequality follows from~\eqref{bobo} in the proof of Theorem~\ref{th-mle_bound};
		\item Recursively expand/construct a ``minimal'' sub-tree rooted at the child node $\langle \sHist{t+1}, \zHist{t} \oplus \zNew{}^{\overline{\ell}} \rangle$ using the most likely sample $\zNew{}^{\overline{\ell}}$ (lines $9$-$10$); 
		\item Use the tightened upper heuristic bound $\overline{V}^{*}_{t+1}(\langle \sHist{t+1}, \zHist{t} \oplus \zNew{}^{\overline{\ell}} \rangle)$ and lower heuristic bound $\underline{V}^{*}_{t+1}(\langle \sHist{t+1}, \zHist{t} \oplus \zNew{}^{\overline{\ell}} \rangle)$ of ${V}^{*}_{t+1}(\langle \sHist{t+1}, \zHist{t} \oplus \zNew{}^{\overline{\ell}} \rangle)$ at node $\langle \sHist{t+1}, \zHist{t} \oplus \zNew{}^{\overline{\ell}} \rangle$ to refine the heuristic bounds at its unexplored siblings (see Corollary~\ref{th:anytime-2}) by exploiting the Lipschitz continuity of $V^*_{t+1}$ (Theorem~\ref{th:lip-optimal}) (line $11$); and				
	\end{enumerate}
	\item Backpropagate the tightened/refined heuristic bounds at node $\langle \sHist{t+1}, \zHist{t} \oplus \zNew{}^{\overline{\ell}} \rangle$ and its siblings to that at their parent node $d_t$
	(lines $12$-$16$).
\end{enumerate}
\begin{algorithm}		
	\caption{Anytime $\epsilon$-Macro-GPO}
	\label{alg:Anytime}
	\begin{algorithmic}[1]
		\FUNCTION{$\text{ExpandTree}(t, d_t, \lambda)$}
		\IF{$t = H$}
		\STATE {\bf return} $\langle 0, 0 \rangle$  
		\ENDIF
		
		\FORALL{$\sNew{t+1} \in \Adom(\sNew{t})$} 		
		
		\STATE $\{\zNew{}^\ell \}_{\ell = 1 ,\ldots, N} \leftarrow$ Draw $N$ i.i.d. multivariate Gaussian vectors from GP posterior belief $p(\zNew{t+1} | \sNew{t+1}, d_t )$~\eqref{gp-posteriors}					
		\FORALL{$\zNew{}^\ell$} 
		\STATE $\underline{V}^{*}_{t+1}(\langle \sHist{t+1}, \zHist{t} \oplus \zNew{}^\ell \rangle) \leftarrow \mathds{V}_{t+1} (\langle \sHist{t+1}, \zHist{t} \oplus \zNew{}^\ell \rangle) - \theta_{t+1}$~\eqref{initbound}
		\STATE $\overline{V}^{*}_{t+1}(\langle \sHist{t+1}, \zHist{t} \oplus \zNew{}^\ell \rangle) \leftarrow \mathds{V}_{t+1} (\langle \sHist{t+1}, \zHist{t} \oplus \zNew{}^\ell \rangle) + \theta_{t+1}$~\eqref{initbound}
		\ENDFOR
		
		\STATE $\overline{\ell} \leftarrow \argmin_{\ell \in \{1, \ldots, N\}} {\lVert \zNew{}^\ell  - \mu_{\sNew{t+1} | d_t} \rVert}$
		
		\STATE $\langle \overline{V}^{*}_{t+1}(\langle \sHist{t+1}, \zHist{t} \oplus \zNew{}^{\overline{\ell}} \rangle),   \, \underline{V}^{*}_{t+1}(\langle \sHist{t+1}, \zHist{t} \oplus \zNew{}^{\overline{\ell}} \rangle)  \rangle \leftarrow \text{ExpandTree}(t+1,\langle \sHist{t+1}, \zHist{t} \oplus \zNew{}^{\overline{\ell}} \rangle, \lambda)$
		
		\STATE $\text{RefineBounds}(t, d_t,\sNew{t+1}, \overline{\ell})$ 
		
		\STATE $R(\sNew{t+1}, d_t) \leftarrow  \mathbf{1}^{\top} \mu_{\sNew{t+1} | d_t} +  0.5{\beta}    \log | I + \sigma_n^{-2} \Sigma_{\sNew{t+1} | \sHist{t}} |$
		
		\STATE $\underline{Q}^{*}_t(\sNew{t+1}, d_t) \leftarrow R(\sNew{t+1}, d_t) + {N}^{-1}\sum_{\ell = 1}^{N} \underline{V}^{*}_{t+1} (\langle \sHist{t+1}, \zHist{t} \oplus \zNew{}^\ell \rangle ) - \lambda$
		
		\STATE $\overline{Q}^{*}_t(\sNew{t+1}, d_t) \leftarrow R(\sNew{t+1}, d_t) + {N}^{-1} \sum_{\ell = 1}^{N} \overline{V}^{*}_{t+1} (\langle \sHist{t+1}, \zHist{t} \oplus \zNew{}^\ell \rangle ) + \lambda$
		
		\ENDFOR
		
		\STATE $\underline{V}^{*}_{t} (d_t) \leftarrow \max_{\sNew{t+1} \in \Adom(\sNew{t})} \underline{Q}^{*}_t(\sNew{t+1}, d_t) $
		
		\STATE $\overline{V}^{*}_{t} (d_t) \leftarrow \max_{\sNew{t+1} \in \Adom(\sNew{t})} \overline{Q}^{*}_t(\sNew{t+1}, d_t) $
		
		\STATE {\bf return} $\langle \overline{V}^{*}_{t} (d_t), \underline{V}^{*}_{t} (d_t) \rangle$		
		\ENDFUNCTION

		\FUNCTION{$\text{RefineBounds}(t, d_t, \sNew{t+1}, j)$}	
		\STATE $\{\zNew{}^\ell \}_{\ell = 1, \ldots, N} \leftarrow \text{RetrieveSamples}(t, d_t, \sNew{t+1})$
		\FORALL{$i \neq j$}
		\STATE $b \leftarrow L_{t+1}(\sHist{t+1})\lVert \zNew{}^i - \zNew{}^j \rVert$
		\STATE $\underline{V}^{*}_{t+1}(\langle \sHist{t+1}, \zHist{t} \oplus \zNew{}^i \rangle) \leftarrow \max (\underline{V}^{*}_{t+1}(\langle \sHist{t+1}, \zHist{t} \oplus \zNew{}^i \rangle), \underline{V}^{*}_{t+1}(\langle \sHist{t+1}, \zHist{t} \oplus \zNew{}^j \rangle) - b )$
		
		\STATE $\overline{V}^{*}_{t+1}(\langle \sHist{t+1}, \zHist{t} \oplus \zNew{}^i \rangle) \leftarrow \min (\overline{V}^{*}_{t+1}(\langle \sHist{t+1}, \zHist{t} \oplus \zNew{}^i \rangle),  \overline{V}^{*}_{t+1}(\langle \sHist{t+1}, \zHist{t} \oplus \zNew{}^j \rangle) + b )$								
		\ENDFOR 
		\ENDFUNCTION

		\FUNCTION{$\text{ConstructTree}(t, d_t, \lambda)$} 
		\IF {$d_t$ has been explored}
		
		\STATE $\sNew{t+1} \leftarrow \argmax_{\sNew{t+1}' \in \Adom(\sNew{t})} \underline{Q}^{*}_t(\sNew{t+1}', d_t)$
		
		\STATE $\{\zNew{}^\ell \}_{\ell = 1, \ldots, N} \leftarrow \text{RetrieveSamples}(t, d_t, \sNew{t+1})$
		
		\STATE $ \ell^* \leftarrow \argmax_{\ell \in \{1, \ldots, N\}} \overline{V}^{*}_{t+1}(\langle \sHist{t+1}, \zHist{t} \oplus \zNew{}^\ell \rangle)  - \underline{V}^{*}_{t+1}(\langle \sHist{t+1}, \zHist{t} \oplus \zNew{}^\ell \rangle)$

		\STATE $\langle \overline{V}^{*}_{t+1}(\langle \sHist{t+1}, \zHist{t} \oplus \zNew{}^{\ell^*} \rangle),   \, \underline{V}^{*}_{t+1}(\langle \sHist{t+1}, \zHist{t} \oplus \zNew{}^{\ell^*} \rangle) \rangle \leftarrow \text{ConstructTree}(t+1, \langle \sHist{t+1}, \zHist{t} \oplus \zNew{}^{\ell^*} \rangle, \lambda)$
		
		\STATE $\text{RefineBounds}(t, d_t,\sNew{t+1}, \ell^*)$ 		
		
		\STATE $R(\sNew{t+1}, d_t) \leftarrow  \mathbf{1}^{\top} \mu_{\sNew{t+1} | d_t} +  0.5{\beta} \log | I + \sigma_n^{-2} \Sigma_{\sNew{t+1} | \sHist{t}} |$
		
		\STATE $\underline{Q}^{*}_t(\sNew{t+1}, d_t) \leftarrow R(\sNew{t+1}, d_t) + {N}^{-1} \sum_{\ell = 1}^{N}  \underline{V}^{*}_{t+1} (\langle \sHist{t+1}, \zHist{t} \oplus \zNew{}^\ell \rangle ) - \lambda$
		
		\STATE $\overline{Q}^{*}_t(\sNew{t+1}, d_t) \leftarrow R(\sNew{t+1}, d_t) + {N}^{-1} \sum_{\ell = 1}^{N} \overline{V}^{*}_{t+1} (\langle \sHist{t+1}, \zHist{t} \oplus \zNew{}^\ell \rangle ) + \lambda$

		\STATE $\underline{V}^{*}_{t} (d_t) \leftarrow \max_{\sNew{t+1} \in \Adom(\sNew{t})}
		\underline{Q}^{*}_t(\sNew{t+1}, d_t) $
		
		\STATE $\overline{V}^{*}_{t} (d_t) \leftarrow \max_{\sNew{t+1} \in \Adom(\sNew{t})} \overline{Q}^{*}_t(\sNew{t+1}, d_t)$

		\STATE {\bf return} $\langle \overline{V}^{*}_{t} (d_t), \underline{V}^{*}_{t} (d_t) \rangle$		
		\ELSE
		\STATE {\bf return} $\text{ExpandTree}(t, d_t, \lambda)$
		\ENDIF
		\ENDFUNCTION

		\FUNCTION{Anytime-$\epsilon$-Macro-GPO$(d_0, \epsilon, H)$}
		
		\FORALL{$\sHist{t+1}$ reachable from $\mathbf{s}_{0}$ and $t=0,\ldots,H-1$}
		\STATE Compute $\Sigma_{\sNew{t+1} | \sHist{t}}$~\eqref{gp-posteriors}, $L_{t+1}(\sHist{t+1})$ (Definition~\ref{definition-L}), and $\mathds{Q}_t (\sNew{t+1}, d_t)$~\eqref{ml-policy}
		\ENDFOR
		
		\STATE Set $\theta$ according to Theorem~\ref{th-mle_bound} (Appendix~\ref{sec:th-mle_bound_proof})
		\STATE  $\lambda \leftarrow 1/({4 H/\epsilon+1/(2\theta)}), \quad \delta \leftarrow \epsilon/(8 \theta H)$
		\WHILE {resources permit}
		\STATE $\langle \overline{V}^{*}_{0}(d_0), \underline{V}^{*}_{0}(d_0) \rangle \leftarrow \text{ConstructTree}(0, d_0, \lambda)$
		\ENDWHILE
		
		\STATE $\omega \leftarrow \overline{V}^{*}_0(d_0) - \underline{V}^{*}_0(d_0)$
		\FORALL {$\sNew{1} \in \Adom(\sNew{0})$}
		\STATE $Q^{\omega\epsilon}_0 (\sNew{1}, d_0) \leftarrow \underline{Q}^{*}_0 (\sNew{1}, d_0)$
		\IF{$|Q^{\omega\epsilon}_0 (\sNew{1}, d_0) - \mathds{Q}_0(\sNew{1}, d_0)| >  2 \lambda + \omega + \theta$}  
		\STATE { $Q^{\omega\epsilon}_0 (\sNew{1}, d_0) \leftarrow \mathds{Q}_0(\sNew{1}, d_0)$}
		\ENDIF
		\ENDFOR
		\STATE  {\bf return} $\pi^{\omega\epsilon}(d_0) \leftarrow \argmax_{\sNew{1} \in \Adom(\sNew{0})} {Q}^{\omega\epsilon}_0(\sNew{1}, d_0)$~\eqref{eq:anytime_policy}
		\ENDFUNCTION
	\end{algorithmic}
\end{algorithm}
\subsection{Theoretical Analysis}
Our result below proves that $\overline{V}^{*}_t(d_t)$ and $\underline{V}^{*}_t(d_t)$, which are previously defined in lines $15$-$16$ and $34$-$35$ in Algorithm~\ref{alg:Anytime}, are upper and lower heuristic bounds of $V^{*}_t(d_t)$, respectively:
%
\begin{theorem}
	\label{th:anytime-1}
	Suppose that the observations $d_{t'}$, $H\in\mathbb{Z}^+$, a budget of $\kappa(H-t')$ input locations for $t'=0,\ldots, H-1$, $\delta\in(0, 1)$, and $\lambda > 0$ are given.
	Then, the probability of 
	\begin{equation}
	\label{eq_5_1}
	\underline{V}^{*}_{t}(d_{t}) \le V^{*}_{t}(d_{t})  \le \overline{V}^{*}_{t}(d_{t})
	\end{equation}
	for all tuples $\langle t,  d_{t}\rangle$ generated at stage $t = t', \ldots, H$ by Algorithm~\ref{alg:Anytime} 
	is at least  $1 - \delta$ by setting $N$ according to Theorem~\ref{th:1_new}.  
\end{theorem}
\begin{proof}
	We will give a proof by induction on $t$ that the probability of~\eqref{eq_5_1} for all tuples $\langle t, d_t\rangle$ generated at stage $t = t', \ldots, H$ by Algorithm~\ref{alg:Anytime} is at least $1-\delta$. The base case of $t = H$ is true since $\underline{V}^{*}_{H}(d_{H}) = V^{*}_{H}(d_{H})  = \overline{V}^{*}_{H}(d_{H}) = 0$. 
	Supposing~\eqref{eq_5_1} holds for $t+1$ (i.e. induction hypothesis), we will prove that it holds for $t=t',\ldots,H-1$. 
	
	Similar to Lemma~\ref{lemma:3} and the main proof of Theorem~\ref{th:1_new}, the probability of 
	\begin{equation}
	\label{eq_5_2}
	\mathcal{U}_t (\sNew{t+1}, d_t) - \lambda \le Q^{*}_t (\sNew{t+1}, d_t) \le    \mathcal{U}_t (\sNew{t+1}, d_t)  + \lambda.
	\end{equation}
	for all tuples $\langle t, \sNew{t+1}, d_t\rangle$ generated at stage $t = t', \ldots, H-1$ by Algorithm~\ref{alg:Anytime} is at least $1 - \delta$.
	%
	%
	
	So, the probability of  
	$$
	\begin{array}{l}
	\displaystyle Q^{*}_t (\sNew{t+1}, d_t) \\
	\displaystyle  \le   \mathcal{U}_t (\sNew{t+1}, d_t) + \lambda \\
	\displaystyle =  R(\sNew{t+1}, d_t) + 
	\frac{1}{N}  
	\sum_{\ell= 1}^{N} 
	V^{*}_{t+1} (\langle \sHist{t+1}, \zHist{t} \oplus \zNew{}^\ell \rangle ) + \lambda  \\
	\displaystyle  \le  R(\sNew{t+1}, d_t) + 
	\frac{1}{N}  
	\sum_{\ell = 1}^{N} 
	\overline{V}^{*}_{t+1} (\langle \sHist{t+1}, \zHist{t} \oplus \zNew{}^\ell\rangle ) + \lambda \\
	\displaystyle = \overline{Q}^{*}_t(\sNew{t+1}, d_t)
	\end{array}
	$$			
	for all tuples $\langle t, \sNew{t+1}, d_t\rangle$ generated at stage $t = t', \ldots, H-1$ by Algorithm~\ref{alg:Anytime} is at least $1 - \delta$
	where the first inequality follows from~\eqref{eq_5_2}, 
	the first equality is due to definition of $\mathcal{U}_t (\sNew{t+1}, d_t)$~\eqref{eq:u-function}, the last inequality is due  to the induction hypothesis, and the last equality is due to definition of $\overline{Q}_t^*$ (see lines $14$ and $33$ in Algorithm~\ref{alg:Anytime}). It follows that the probability of $V^{*}_t(d_t)  \le \overline{V}^{*}_t(d_t)$ for all tuples $\langle t, d_t\rangle$ generated at stage $t = t', \ldots, H-1$ by Algorithm~\ref{alg:Anytime} is at least $1-\delta$. 
	
	Similarly, the probability of 
	$$
	\begin{array}{l}
	\displaystyle  Q^{*}_t (\sNew{t+1}, d_t) \\
	\displaystyle  \ge   \mathcal{U}_t (\sNew{t+1}, d_t) - \lambda \\ 
	\displaystyle =  R(\sNew{t+1}, d_t) + 
	\frac{1}{N}  
	\sum_{\ell = 1}^{N} 
	V^{*}_{t+1} (\langle \sHist{t+1}, \zHist{t} \oplus \zNew{}^\ell \rangle) - \lambda  \\
	\displaystyle  \ge  R(\sNew{t+1}, d_t) + 
	\frac{1}{N}  
	\sum_{\ell= 1}^{N} 
	\underline{V}^{*}_{t+1} (\langle \sHist{t+1}, \zHist{t} \oplus \zNew{}^\ell \rangle ) - \lambda \\
	\displaystyle = \underline{Q}^{*}_t(\sNew{t+1}, d_t)
	\end{array}
	$$
	for all tuples $\langle t, \sNew{t+1}, d_t\rangle$ generated at stage $t = t', \ldots, H-1$ by Algorithm~\ref{alg:Anytime} is at least $1 - \delta$ 		
	where the first inequality is due to~\eqref{eq_5_2}, the first equality is due to definition $\mathcal{U}_t (\sNew{t+1}, d_t)$~\eqref{eq:u-function},  the last inequality is due to the induction hypothesis, and the last equality is due to definition of $\underline{Q}_t^*$ (see lines $13$ and $32$ in Algorithm~\ref{alg:Anytime}). 
	It follows that the probability of $V^{*}_t(d_t)\ge \underline{V}^{*}_t(d_t)$ for all tuples $\langle t, d_t\rangle$ generated at stage $t = t', \ldots, H-1$ by Algorithm~\ref{alg:Anytime} is at least $1-\delta$. 
	%
\end{proof}
Our next result justifies why the function RefineBounds (lines $18$-$23$) in Algorithm~\ref{alg:Anytime} can use the tightened heuristic bounds at nodes $\langle \sHist{t+1}, \zHist{t} \oplus \zNew{}^{\overline{\ell}} \rangle$ and $\langle \sHist{t+1}, \zHist{t} \oplus \zNew{}^{\ell^*} \rangle$ to refine the heuristic bounds at their siblings (lines $11$ and $30$) by exploiting the Lipschitz continuity of $V^{*}_{t+1}$ (Theorem~\ref{th:lip-optimal}), as explained previously in Appendix~\ref{sec:pseudo}:
\begin{corollary}
	\label{th:anytime-2}
	Suppose that the observations $d_{t'}$, $H\in\mathbb{Z}^+$, a budget of $\kappa(H-t')$ input locations for $t'=0,\ldots, H-1$, $\delta\in(0, 1)$  and $\lambda > 0$ are given.
	Then, the probability of
	%
	$$					
	\underline{V}^{*}_t(\langle \sHist{t}, \zHist{t-1} \oplus \zNew{}^j \rangle) - L_t(\sHist{t})  \lVert \zNew{}^i - \zNew{}^j \rVert  \le V^{*}_t(\langle \sHist{t}, \zHist{t-1} \oplus \zNew{}^i \rangle) \le \overline{V}^{*}_t(\langle \sHist{t}, \zHist{t-1} \oplus \zNew{}^j \rangle) +  L_t(\sHist{t}) \lVert \zNew{}^i - \zNew{}^j \rVert 
	$$
	%
	between any pair of tuples $\langle t, \langle \sHist{t}, \zHist{t-1} \oplus \zNew{}^i \rangle \rangle$ and $\langle t, \langle \sHist{t}, \zHist{t-1} \oplus \zNew{}^j \rangle \rangle$ for $i,j = 1, \ldots, N$
	generated at stage $t = t'+1, \ldots, H$ by Algorithm~\ref{alg:Anytime}
	is at least  $1 - \delta$ by setting $N$ according to Theorem~\ref{th:1_new}.
\end{corollary}
\begin{proof}
	$$
	\begin{array}{l}
	\displaystyle  V^{*}_t(\langle \sHist{t}, \zHist{t-1} \oplus \zNew{}^i \rangle)  \\
	\displaystyle  \le V^{*}_t(\langle \sHist{t}, \zHist{t-1} \oplus \zNew{}^j \rangle) +  L_t( \sHist{t})  \lVert \zNew{}^i - \zNew{}^j \rVert \\
	\displaystyle \le \overline{V}^{*}_t(\langle \sHist{t}, \zHist{t-1} \oplus \zNew{}^j \rangle) +  L_t( \sHist{t}) \lVert \zNew{}^i - \zNew{}^j \rVert
	\end{array}
	$$
	where the first inequality is a direct consequence of Theorem~\ref{th:lip-optimal} in Appendix~\ref{lip-optimal-value} and the second inequality is due to Theorem~\ref{th:anytime-1}.
	%
	$$
	\begin{array}{l}
	\displaystyle   V^{*}_t(\langle \sHist{t}, \zHist{t-1} \oplus \zNew{}^i \rangle) \\
	\displaystyle    \ge V^{*}_t(\langle \sHist{t}, \zHist{t-1} \oplus \zNew{}^j \rangle) -  L_t( \sHist{t}) \lVert \zNew{}^i - \zNew{}^j \rVert \\ 
	\displaystyle \ge \underline{V}^{*}_t(\langle \sHist{t}, \zHist{t-1} \oplus \zNew{}^j \rangle) -  L_t( \sHist{t})  \lVert \zNew{}^i - \zNew{}^j \rVert.
	\end{array}
	$$
	where the first inequality is a direct consequence of Theorem~\ref{th:lip-optimal} in Appendix~\ref{lip-optimal-value} and the second inequality is due to Theorem~\ref{th:anytime-1}.
\end{proof}
Similar to Theorem~\ref{th:1_new}, our result below derives a probabilistic guarantee on the approximation quality of $\underline{Q}^*_t (\sNew{t+1}, d_t)$:
\begin{theorem}
	Suppose that the observations $d_{t}$, $H\in\mathbb{Z}^+$, a budget of $\kappa(H-t)$ input locations for $t=0,\ldots, H-1$, $\delta\in(0, 1)$, and $\lambda > 0$ are given and Algorithm~\ref{alg:Anytime} terminates at $\omega \triangleq \overline{V}^{*}_0(d_0) - \underline{V}^{*}_0(d_0)$ (see line $46$ in Algorithm~\ref{alg:Anytime}).
	Then, the probability of 
	$| \underline{Q}_t^*(\sNew{t+1} , d_t) - Q^*_t(\sNew{t+1} , d_t)|\le 2 \lambda + \omega$ for all 
	$\sNew{t+1} \in \Adom(\sNew{t})$ 
	is at least $1 - \delta$ by setting $N$ according to Theorem~\ref{th:1_new}.
	\label{crapmount}
\end{theorem}	
\begin{proof}
	It follows directly from Theorem~\ref{th:anytime-1} that the probability of 
	\begin{equation}
	\label{lalaland}
	| V^*_0(d_0) - \underline{V}^{*}_0(d_0) | \le \omega
	\end{equation} 
	is at least $1- \delta$. 	
	In general, supposing the planning horizon is reduced to $H-t$ stages for $ t = 0, \ldots,  H-1$,~\eqref{lalaland} is equivalent to  
	\begin{equation}
	\label{eq_5_7}
	| V^*_t(d_t) - \underline{V}^{*}_t(d_t) | \le \omega
	\end{equation}
	by shifting the indices of $V^*_0(d_0)$ and $\underline{V}^{*}_0(d_0)$ in~\eqref{lalaland} from $0$ to $t$ so that they start at stage $t$ instead. Then, the probability of	
	$$
	\begin{array}{l}
	\displaystyle | \underline{Q}_t^*(\sNew{t+1} , d_t)-Q^*_t(\sNew{t+1} , d_t)| \\
	\displaystyle \le |\underline{Q}_t^*(\sNew{t+1} , d_t) - \mathcal{U}_t (\sNew{t+1}, d_t) | +
	|\mathcal{U}_t (\sNew{t+1}, d_t) - Q^*_t(\sNew{t+1} , d_t)|
	\\
	\displaystyle \le   \lambda + \left| \left(\frac{1}{N} \sum_{\ell = 1}^{N}  
	V^{*}_{t+1} (\langle \sHist{t+1}, \zHist{t} \oplus \zNew{}^\ell \rangle ) -  \sum_{\ell = 1}^{N}  \underline{V}^{*}_{t+1} (\langle \sHist{t+1}, \zHist{t} \oplus \zNew{}^\ell \rangle ) \right) + \lambda\right| \\
	\displaystyle \le  2 \lambda + 
	\frac{1}{N} \sum_{\ell = 1}^{N}  
	| V^{*}_{t+1} (\langle \sHist{t+1}, \zHist{t} \oplus \zNew{}^\ell \rangle ) - 
	\underline{V}^{*}_{t+1} (\langle \sHist{t+1}, \zHist{t} \oplus \zNew{}^\ell \rangle)| \\
	\displaystyle \le 2 \lambda + \omega 
	\end{array}
	$$
	for all $\sNew{t+1} \in \Adom(\sNew{t})$ 
	is at least $1-\delta$
	where the first and the third inequalities are due to triangle inequality, the second inequality follows from~\eqref{eq_5_2}, definition of $\mathcal{U}_t (\sNew{t+1}, d_t)$~\eqref{eq:u-function}, and definition of $\underline{Q}_t^*$ (see lines $13$ and $32$ in Algorithm~\ref{alg:Anytime}), and the last inequality is due to \eqref{eq_5_7}. 
\end{proof}	
We will now give an anytime analogue/variant of our nonmyopic adaptive $\epsilon$-Macro-GPO policy $\pi^{\epsilon}$~\eqref{eq_4_8}, which we call the $\langle\omega,\epsilon\rangle$-Macro-GPO policy $\pi^{\omega\epsilon}$:
%
\begin{equation}
\label{eq:anytime_policy}
\begin{array}{rl}
\pi^{\omega\epsilon}(d_t)  \triangleq &\hspace{-2.4mm}  \argmax_{\sNew{t+1} \in \Adom(\sNew{t})} Q^{\omega\epsilon}_t (\sNew{t+1}, d_{t})\vspace{0.5mm}\\
Q^{\omega\epsilon}_t (\sNew{t+1}, d_t) \triangleq &\hspace{-2.4mm}  \displaystyle
\begin{cases}
\underline{Q}^*_t (\sNew{t+1}, d_t) & \text{if} \left|\underline{Q}^*_t (\sNew{t+1}, d_t) - \mathds{Q}_t(\sNew{t+1}, d_t)\right| \le 2 \lambda + \omega + \theta , \\
\mathds{Q}_t (\sNew{t+1}, d_t)     & \text{otherwise};
\end{cases}
\end{array}   
\end{equation}	
for stages $t =0, \ldots, H-1$ where $\mathds{Q}_t(\sNew{t+1}, d_t)$ and $\theta$ are previously defined in~\eqref{ml-policy} and Theorem~\ref{th-mle_bound}, respectively. 
The implications of the tractable choice of the if condition in~\eqref{eq:anytime_policy} for theoretically guaranteeing the performance of our $\langle\omega,\epsilon\rangle$-Macro-GPO policy $\pi^{\omega\epsilon}$ are similar to that of our $\epsilon$-Macro-GPO policy $\pi^{\epsilon}$~\eqref{eq_4_8}:

{\bf I.} In the likely event (with an arbitrarily high probability of at least $1-\delta$) that $| \underline{Q}_t^*(\sNew{t+1} , d_t) - Q^*_t(\sNew{t+1} , d_t)|\le 2 \lambda + \omega$ for all 
$\sNew{t+1} \in \Adom(\sNew{t})$ (Theorem~\ref{crapmount}),
$|\underline{Q}^*_t (\sNew{t+1}, d_t) - \mathds{Q}_t(\sNew{t+1}, d_t)|
\le | \underline{Q}^*_t (\sNew{t+1}, d_t) - {Q}^*_t (\sNew{t+1}, d_t) | + |{Q}^*_t (\sNew{t+1}, d_t) - \mathds{Q}_t (\sNew{t+1}, d_t)|\leq 2\lambda  + \omega + \theta$ for all $\sNew{t+1} \in \Adom(\sNew{t})$, by triangle inequality and Theorems~\ref{th-mle_bound} and~\ref{crapmount}. 
Consequently, according to~\eqref{eq:anytime_policy}, $Q^{\omega\epsilon}_t (\sNew{t+1}, d_{t}) = \underline{Q}^*_t (\sNew{t+1}, d_t)$ for all $\sNew{t+1} \in \Adom(\sNew{t})$ and $\pi^{\omega\epsilon}(d_{t})$ thus selects the same macro-action as the policy induced by $\underline{Q}_t^*(\sNew{t+1} , d_t)$ (see lines $13$ and $32$ in Algorithm~\ref{alg:Anytime}).

{\bf II.} In the unlikely event (with an arbitrarily small probability of at most $\delta$) that 
$\underline{Q}_t^*(\sNew{t+1} , d_t)$ (see lines $13$ and $32$ in Algorithm~\ref{alg:Anytime}) is unboundedly far from ${Q}^*_t (\sNew{t+1}, d_t)$~\eqref{eq:OptimalValFunDef}  
(i.e., $| \underline{Q}_t^*(\sNew{t+1} , d_t) - {Q}^*_t (\sNew{t+1}, d_t) | > 2\lambda + \omega$) for some $\sNew{t+1} \in \Adom(\sNew{t})$, 
$\pi^{\omega\epsilon}(d_{t})$~\eqref{eq:anytime_policy} guarantees that, for any selected macro-action $\sNew{t+1} \in \Adom(\sNew{t})$,
\begin{equation}
\hspace{-1.9mm}
\begin{array}{l}
| {Q}^{\omega\epsilon}_t (\sNew{t+1}, d_t) - {Q}^*_t (\sNew{t+1}, d_t) |\\
=\hspace{-1mm}
\begin{cases}
|\underline{Q}_t^* \hspace{-0.3mm}(\sNew{t+1}, d_t)\hspace{-0.8mm} -\hspace{-0.6mm} {Q}^*_t \hspace{-0.3mm}(\sNew{t+1}, d_t) | &\hspace{-1.05mm}  
\begin{array}{l}
\text{if } |\underline{Q}_t^*\hspace{-0.3mm} (\sNew{t+1}, d_t)\hspace{-0.8mm} -\hspace{-0.6mm} \mathds{Q}_t\hspace{-0.3mm}(\sNew{t+1}, d_t)| \\
\quad\le 2\lambda +\omega  + \theta , 
\end{array}
\\
|\mathds{Q}_t \hspace{-0.3mm}(\sNew{t+1}, d_t)\hspace{-0.8mm} -\hspace{-0.6mm} {Q}^*_t \hspace{-0.3mm}(\sNew{t+1}, d_t) |     & \hspace{0.65mm}\text{otherwise};
\end{cases}
\\
\le\hspace{-1mm}
\begin{cases}
\hspace{-1.78mm}
\begin{array}{l}
|\underline{Q}_t^* \hspace{-0.3mm}(\sNew{t+1}, d_t)\hspace{-0.8mm} -\hspace{-0.6mm} \mathds{Q}_t \hspace{-0.3mm}(\sNew{t+1}, d_t) | \\
+ |\mathds{Q}_t \hspace{-0.3mm}(\sNew{t+1}, d_t)\hspace{-0.8mm} -\hspace{-0.6mm} {Q}^*_t \hspace{-0.3mm}(\sNew{t+1}, d_t) |
\end{array}
&\hspace{-5.5mm}  
\begin{array}{l}
\text{if } |\underline{Q}_t^*\hspace{-0.3mm} (\sNew{t+1}, d_t)\hspace{-0.8mm} -\hspace{-0.6mm} \mathds{Q}_t\hspace{-0.3mm}(\sNew{t+1}, d_t)| \\
\quad\le 2\lambda +\omega  + \theta , 
\end{array}
\\
\theta     & \hspace{-3.8mm}\text{otherwise};
\end{cases}
\\
\le 2\lambda +\omega  +2\theta\ ,
\end{array}
\label{flop2}
\end{equation} 
by triangle inequality and Theorem~\ref{th-mle_bound}.

The above implications are central to proving our next result bounding the \emph{expected} performance loss of $\pi^{\omega\epsilon}$ relative to that of Bayes-optimal Macro-GPO policy $\pi^*$, that is, policy $\pi^{\omega\epsilon}$ is $\langle\omega, \epsilon\rangle$-Bayes-optimal:
\begin{lemma}
	Suppose that the observations $d_{t}$, $H\in\mathbb{Z}^+$, a budget of $\kappa(H-t)$ input locations for $t=0,\ldots, H-1$, $\delta\in(0, 1)$, and $\lambda > 0$ are given.
	Then, the probability of
	$$| {Q}^*_t (\pi^*(d_t), d_t) - {Q}^*_t (\pi^{\omega\epsilon}(d_{t}), d_t) | \le 2 \lambda + 2\omega$$ 
	is at least $1-\delta$ by setting $N$ according to that in Theorem~\ref{th:1_new}.
	\label{smith3}
\end{lemma}
\begin{proof}
	$$
	\begin{array}{l}
	\displaystyle {Q}^*_t (\pi^*(d_t), d_t) - {Q}^*_t (\pi^{\omega\epsilon}(d_{t}), d_t) \\
	\displaystyle\le {Q}^*_t (\pi^*(d_t), d_t) - \underline{Q}^*_t (\pi^{\omega\epsilon}(d_{t}), d_t) + 2\lambda +\omega\\
	\displaystyle\le |{Q}^*_t (\pi^*(d_t), d_t) - \underline{Q}^*_t (\pi^{\omega\epsilon}(d_{t}), d_t)| + 2\lambda +\omega\\
	\displaystyle=  |  Q^{*}_t (\pi^{*}(d_t), d_{t}) -  \max_{\sNew{t+1} \in \Adom(\sNew{t})} \underline{Q}_t^*(\sNew{t+1}, d_t) | + 2\lambda +\omega\\
	\displaystyle= | V^*_t(d_t) - \underline{V}^{*}_t(d_t) | + 2\lambda +\omega\\
	\displaystyle  \le\omega + 2\lambda +\omega\\
	\displaystyle = 2 \lambda + 2\omega
	\end{array}
	$$
	where the first inequality is due to Theorem~\ref{crapmount}, the first equality is further due to implication I discussed just after~\eqref{eq:anytime_policy}, the second equality is due to the definitions of $V^*_t$~\eqref{eq:OptimalValFunDef} and $\underline{V}^{*}_t$ (see lines $15$ and $34$ in Algorithm~\ref{alg:Anytime}), and the last inequality is due to~\eqref{eq_5_7}.
\end{proof}
\begin{lemma}
	Suppose that the observations $d_{t}$, $H\in\mathbb{Z}^+$, a budget of $\kappa(H-t)$ input locations for $t=0,\ldots, H-1$, $\delta\in(0, 1)$, and $\lambda > 0$ are given.
	Then, 
	$$Q^*_t (\pi^*(d_t), d_t) - \mathbb{E}_{\pi^{\omega\epsilon}(d_t)}[Q^*_t (\pi^{\omega\epsilon}(d_t), d_t)]\le 2\lambda + 2\delta\lambda + 2\omega + 4\delta\theta$$ 
	where $\theta$ is previously defined in Theorem~\ref{th-mle_bound}.
	\label{lem:fixed_policy_beta2}	
\end{lemma}	
\begin{proof}
	By Lemma~\ref{smith3}, the probability of $| {Q}^*_t (\pi^*(d_t), d_t) - {Q}^*_t (\pi^{\omega\epsilon}(d_{t}), d_t) | \le 2 \lambda +2\omega$ is at least $1-\delta$. 
	Otherwise, the probability of $| {Q}^*_t (\pi^*(d_t), d_t) - {Q}^*_t (\pi^{\omega\epsilon}(d_{t}), d_t) | > 2 \lambda +2\omega$ is at most $\delta$. In the latter case,
	\begin{equation}
	\begin{array}{l}
	\displaystyle | {Q}^*_t (\pi^*(d_t), d_t) - {Q}^*_t (\pi^{\omega\epsilon}(d_{t}), d_t) |\\
	\displaystyle\le | {Q}^*_t (\pi^*(d_t), d_t) -{Q}^{\omega\epsilon}_t (\pi^{\omega\epsilon}(d_{t}), d_t)| + |{Q}^{\omega\epsilon}_t (\pi^{\omega\epsilon}(d_{t}), d_t) -{Q}^*_t (\pi^{\omega\epsilon}(d_{t}), d_t) |\\
	\displaystyle\le \max_{\sNew{t+1} \in \Adom(\sNew{t})}  | {Q}^{\omega\epsilon}_t (\sNew{t+1}, d_t) - {Q}^*_t (\sNew{t+1}, d_t) |  + 2\lambda + \omega +2\theta\\
	\displaystyle\le 2\lambda + \omega +2\theta + 2\lambda + \omega +2\theta\\
	\displaystyle = 4\lambda + 2\omega +4\theta
	\end{array}
	\label{smith4}
	\end{equation}
	where the first inequality is due to triangle inequality and the last two inequalities are due to~\eqref{flop2} (i.e., implication II). 
	Recall that $\pi^{\omega\epsilon}$ is a stochastic policy due to its use of stochastic sampling in $\underline{Q}^*_t$(see lines $13$ and $32$ in Algorithm~\ref{alg:Anytime}), which implies that 
	$\pi^{\omega\epsilon}(d_t)$ is a random variable. 
	Then,	
	$$
	\begin{array}{l}
	\displaystyle Q^*_t (\pi^*(d_t), d_t) - \mathbb{E}_{\pi^{\omega\epsilon}(d_t)}[Q^*_t (\pi^{\omega\epsilon}(d_t), d_t)]\\
	\displaystyle = \mathbb{E}_{\pi^{\omega\epsilon}(d_t)}[ Q^*_t (\pi^*(d_t), d_t) - Q^*_t (\pi^{\omega\epsilon}(d_t), d_t)]\\
	\le (1-\delta) (2\lambda +2\omega) + \delta (4\lambda + 2\omega +4\theta)\\
	= 2\lambda + 2\delta\lambda + 2\omega + 4\delta\theta
	\end{array}
	$$
	where the expectation is with respect to random variable $\pi^{\omega\epsilon}(d_t)$ and the inequality follows from Lemma~\ref{smith3} and~\eqref{smith4}.
\end{proof}
\begin{theorem}
	\label{th:expected2}
	Suppose that the observations $d_{0}$, $H\in\mathbb{Z}^+$, a budget of $\kappa H$ input locations, and an arbitrarily user-specified loss bound $\epsilon > 0$ are given and Algorithm~\ref{alg:Anytime} terminates at $\omega \triangleq \overline{V}^{*}_0(d_0) - \underline{V}^{*}_0(d_0)$ (see line $46$ in Algorithm~\ref{alg:Anytime}). 		
	Then, $V^*_0(d_0) - \mathbb{E}_{\pi^{\omega\epsilon}}[V^{\pi^{\omega\epsilon}}_0(d_0)] \le 2\omega H + \epsilon$ 
	by setting $\theta$ according to Theorem~\ref{th-mle_bound}, 
	$\delta = {\epsilon}/(8 \theta H)$ and $\lambda = 1/({4 H/\epsilon+1/(2\theta)})$ in Theorem~\ref{th:1_new} 
	to yield
	$$
	N  = \mathcal{O} \left( \frac{\kappa^{2H}}{\epsilon^2}
	\log{\frac{\kappa A}{\epsilon}} \right) .
	$$
\end{theorem}
\begin{proof}
	We will give a proof by induction on $t$ that 
	\begin{equation}
	V_t^*(d_t) - \mathbb{E}_{\pi^{\omega\epsilon}}[ V_t^{\pi^{\omega\epsilon}}(d_t)] \le (2\lambda + 2\delta\lambda + 2\omega + 4\delta\theta) (H - t)\ .
	\label{beta_V2}
	\end{equation}
	When $t = H-1$ (i.e., base case), 
	$$
	\begin{array}{l}
	\displaystyle  V_{H-1}^*(d_{H-1}) - \mathbb{E}_{\pi^{\omega\epsilon}} [ V_{H-1}^{\pi^{\omega\epsilon}}(d_{H-1}) ] \\ 
	\displaystyle = Q_{H-1}^*(\pi^*(d_{H-1}), d_{H-1}) -  \mathbb{E}_{\pi^{\omega\epsilon}} [ Q_t^{\pi^{\omega\epsilon}} (\pi^{\omega\epsilon}(d_{H-1}), d_{H-1})] \\ 
	\displaystyle = Q_{H-1}^*(\pi^*(d_{H-1}), d_{H-1}) -  \mathbb{E}_{\pi^{\omega\epsilon}(d_{H-1})} [ R (\pi^{\omega\epsilon}(d_{H-1}), d_{H-1})] \\ 
	\displaystyle =   
	Q_{H-1}^*(\pi^*(d_{H-1}), d_{H-1}) -  \mathbb{E}_{\pi^{\omega\epsilon}(d_{H-1})} [ Q_t^{*} (\pi^{\omega\epsilon}(d_{H-1}), d_{H-1})] \\ 		
	\displaystyle \le 2\lambda + 2\delta\lambda + 2\omega + 4\delta\theta
	\end{array}
	$$
	where the first equality is due to~\eqref{eq:general-policy} and~\eqref{eq:OptimalValFunDef}, the second equality is due to~\eqref{eq:general-policy}, the third equality is due to~\eqref{eq:OptimalValFunDef}, and the inequality is due to Lemma~\ref{lem:fixed_policy_beta2}. 
	So,~\eqref{beta_V2} holds for the base case.
	Supposing~\eqref{beta_V2} holds for $t+1$ (i.e., induction hypothesis), we will prove that it holds for $t = 0, \ldots, H-2$:			
	\begin{equation}
	\hspace{-1.9mm}
	\begin{array}{l}
	\displaystyle  V_t^*(d_t) - \mathbb{E}_{\pi^{\omega\epsilon}}[ V_t^{\pi^{\omega\epsilon}}(d_t) ]\\ 
	\displaystyle = Q_t^*(\pi^*(d_t), d_t) -  \mathbb{E}_{\pi^{\omega\epsilon}} [ Q_t^{\pi^{\omega\epsilon}} (\pi^{\omega\epsilon}(d_t), d_t)]\\ 
	\displaystyle = Q_t^*(\pi^*(d_t), d_t) -  \mathbb{E}_{\pi^{\omega\epsilon}} [Q_t^* (\pi^{\omega\epsilon}(d_t), d_t)]  + \mathbb{E}_{\pi^{\omega\epsilon}}[ Q_t^* (\pi^{\omega\epsilon}(d_t), d_t) ] -  \mathbb{E}_{\pi^{\omega\epsilon}}[ Q_t^{\pi^{\omega\epsilon}} (\pi^{\omega\epsilon}(d_t), d_t)]\\ 
	\displaystyle = Q_t^*(\pi^*(d_t), d_t) -  \mathbb{E}_{\pi^{\omega\epsilon}(d_t)} [Q_t^* (\pi^{\omega\epsilon}(d_t), d_t)]  +   \mathbb{E}_{\pi^{\omega\epsilon}} [ Q_t^* (\pi^{\omega\epsilon}(d_t), d_t)   - Q_t^{\pi^{\omega\epsilon}} (\pi^{\omega\epsilon}(d_t), d_t)] \\ 
	\displaystyle \le  2\lambda + 2\delta\lambda + 2\omega + 4\delta\theta +  \mathbb{E}_{\pi^{\omega\epsilon}} [ Q_t^* (\pi^{\omega\epsilon}(d_t), d_t)   - Q_t^{\pi^{\omega\epsilon}} (\pi^{\omega\epsilon}(d_t), d_t)] \\ 
	\displaystyle  = 2\lambda + 2\delta\lambda + 2\omega + 4\delta\theta  \\
	\displaystyle\quad +  \mathbb{E}_{\pi^{\omega\epsilon}} [ \mathbb{E}_{\zNew{t+1} | \pi^{\omega\epsilon}(d_t), d_t} [V^*_{t+1} ( \langle \sHist{t} \oplus \pi^{\omega\epsilon}(d_t),  \zHist{t} \oplus \zNew{t+1} \rangle) -  V^{\pi^{\omega\epsilon}}_{t+1} (\langle \sHist{t} \oplus \pi^{\omega\epsilon}(d_t),  \zHist{t} \oplus \zNew{t+1} \rangle)] ]\\ 
	\displaystyle  = 2\lambda + 2\delta\lambda + 2\omega + 4\delta\theta\\
	\displaystyle\quad +  \mathbb{E}_{\pi^{\omega\epsilon}(d_t)} [ \mathbb{E}_{\zNew{t+1} | \pi^{\omega\epsilon}(d_t), d_t} [V^*_{t+1} ( \langle \sHist{t} \oplus \pi^{\omega\epsilon}(d_t),  \zHist{t} \oplus \zNew{t+1} \rangle) -  \mathbb{E}_{\pi^{\omega\epsilon}}[V^{\pi^{\omega\epsilon}}_{t+1} (\langle \sHist{t} \oplus \pi^{\omega\epsilon}(d_t),  \zHist{t} \oplus \zNew{t+1} \rangle)]] ]\\ 
	\displaystyle \le  2\lambda + 2\delta\lambda + 2\omega + 4\delta\theta + \mathbb{E}_{\pi^{\omega\epsilon}(d_t)} [ \mathbb{E}_{\zNew{t+1} | \pi^{\omega\epsilon}(d_t), d_t} [(2\lambda + 2\delta\lambda + 2\omega + 4\delta\theta)(H - t - 1)]]\\
	\displaystyle  = (2\lambda + 2\delta\lambda + 2\omega + 4\delta\theta) (H- t)
	\end{array}
	\label{smith5}		
	\end{equation}
	where the first and fourth equalities are due to~\eqref{eq:general-policy} and~\eqref{eq:OptimalValFunDef}, 
	the first inequality is due to Lemma~\ref{lem:fixed_policy_beta2}, 
	and the last inequality is due to the induction hypothesis.
	
	From~\eqref{smith5}, when $t = 0$,
	$$
	V_0^*(d_0) - \mathbb{E}_{\pi^{\omega\epsilon}}  [V_0^{\pi^{\omega\epsilon}}(d_0)] \le 2H (\lambda + \delta\lambda + \omega + 2\delta\theta) = 2\omega H + 2H (\lambda + \delta\lambda + 2\delta\theta)\ .
	$$
	Let $\epsilon = 2H (\lambda +\delta\lambda  + 2 \delta  \theta)$ by setting 
	%
	%
	$\lambda=1/({4 H/\epsilon+1/(2\theta)})$
	%
	and $\delta = {\epsilon}/({8 \theta H})$. Consequently, using Lemma~\ref{smith3} and $\theta = \mathcal{O}( \kappa^{H + 1/2})$ previously defined in Theorem~\ref{th-mle_bound}, 	 
	$$
	N  = \mathcal{O} \left( \frac{\kappa^{2H}}{\epsilon^2}
	\log{\frac{\kappa A}{\epsilon}} \right) .
	$$
\end{proof}

\section{Auxiliary Results}
\label{sec:auxiliary}
\begin{lemma}
	\label{lem:L-bound}
	$L_t(\sHist{t}) = \mathcal{O}(\kappa^{H - t +{1}/{2}}\sqrt{H!/t!} (1+\sigma^2_y/\sigma^2_n)^{H-t})$ for  $t  = 0, \ldots, H-1$.
\end{lemma}
\begin{proof}
	Using Definition~\ref{definition-L} followed by Lemma~\ref{app-a-lemma-alpha},
	\begin{equation}
	\label{gosh}
	\begin{array}{l}
	\displaystyle  L_t(\sHist{t}) \\
	\displaystyle  = \max_{{\sNew{t+1} \in \Adom(\sNew{t})}} 
	\sqrt{\kappa}\ \alpha(\sHist{t+1}) + L_{t+1}(\sHist{t+1}) \sqrt{1 + \alpha(\sHist{t+1})^2}  \\
	\displaystyle = (\sqrt{\kappa}+L_{t+1}(\sHist{t} \oplus \sNew{t+1}^*))\ \mathcal{O}(\kappa \sqrt{t+1}(1+\sigma^2_y/\sigma^2_n))			 
	\end{array}
	\end{equation}
	for $t=0,\ldots,H-1$ where $\sNew{t+1}^* \triangleq \argmax_{{\sNew{t+1} \in \Adom(\sNew{t})}}  L_{t+1}(\sHist{t} \oplus \sNew{t+1})$. 
	
	We will now give a proof by induction on $t$.
	When $t = H-1$ (i.e., base case), since $L_H({\sHist{H}}) = 0$ (Definition~\ref{definition-L}), 
	it follows from~\eqref{gosh} that $L_{H-1}({\sHist{H-1}}) = \mathcal{O}(\kappa^{3/2} \sqrt{H}(1+\sigma^2_y/\sigma^2_n))$.
	Supposing Lemma~\ref{lem:L-bound} holds for $t+1$ (i.e., induction hypothesis), 
	we will prove that it holds for $0\leq t<H-1$:
	$$
	\begin{array}{l}
	\displaystyle L_t(\sHist{t}) \\
	\displaystyle = (\sqrt{\kappa}+\mathcal{O}(\kappa^{H - t -{1}/{2}}\sqrt{H!/(t+1)!}\ (1+\sigma^2_y/\sigma^2_n)^{H-t-1}))\ \mathcal{O}(\kappa \sqrt{t+1}(1+\sigma^2_y/\sigma^2_n)) \\		
	\displaystyle = \mathcal{O}(\kappa^{H - t +{1}/{2}}\sqrt{H!/t!}\ (1+\sigma^2_y/\sigma^2_n)^{H-t})		
	\end{array}				
	$$
	where the first equality follows from~\eqref{gosh} and the induction hypothesis. 
	%
\end{proof}	
\begin{lemma}
	\label{app-a-lemma-alpha}
	$\alpha(\sHist{t+1})  = \mathcal{O}(\kappa \sqrt{t+1}(1+\sigma^2_y/\sigma^2_n))$ for $t = 0,\ldots,H-1$ where the function $\alpha$ is previously defined in Lemma~\ref{lemma:reward}.
\end{lemma}
\begin{proof}
	Let $\Xi \Lambda \Xi^\top$ be an eigendecomposition of the symmetric and positive definite $\Sigma_{\sHist{t} \sHist{t}}$ where $\Xi$ is a matrix whose columns comprise an orthonormal basis of eigenvectors of $\Sigma_{\sHist{t} \sHist{t}}$ and $\Lambda$ is a diagonal matrix with positive eigenvalues of $\Sigma_{\sHist{t} \sHist{t}}$.
	From the definition of the function $\alpha$ in Lemma~\ref{lemma:reward},		
	%
	%
	\begin{equation}
	\label{app-a-09}
	\begin{array}{l}
	\displaystyle  \alpha(\sHist{t+1})^2  \\
	\displaystyle = \lVert \Sigma_{\sNew{t+1}\sHist{t}}\Sigma_{\sHist{t} \sHist{t}}^{-1} \lVert_F^2 \\
	\displaystyle   = \lVert \Sigma_{\sNew{t+1} \sHist{t}} \Xi \Lambda^{-1} \Xi^\top\lVert_F^2 \\
	\displaystyle  =  \mathrm{Tr}(\Sigma_{\sNew{t+1} \sHist{t}} \Xi \Lambda^{-1} \Xi^\top \Xi \Lambda^{-1}  \Xi^\top  \Sigma_{ \sHist{t} \sNew{t+1}} )  \\ 
	\displaystyle = \mathrm{Tr}(\Sigma_{\sNew{t+1} \sHist{t}} \Xi \Lambda^{-2}  \Xi^\top  \Sigma_{  \sHist{t} \sNew{t+1} } ) \\
	\displaystyle = \mathrm{Tr}(\Sigma_{\sNew{t+1} \sHist{t}} \Xi (\xi^{-2}I)   \Xi^\top  \Sigma_{  \sHist{t} \sNew{t+1} } ) - \mathrm{Tr}(\Sigma_{\sNew{t+1} \sHist{t}} \Xi (\xi^{-2} I - \Lambda^{-2})   \Xi^\top  \Sigma_{  \sHist{t} \sNew{t+1} } )   \\
	\displaystyle  \le \mathrm{Tr}(\Sigma_{\sNew{t+1} \sHist{t}} \Xi (\xi^{-2} I)   \Xi^\top  \Sigma_{ \sHist{t} \sNew{t+1} } )  \\
	\displaystyle = \xi^{-2}  \mathrm{Tr}(\Sigma_{\sNew{t+1} \sHist{t}}   \Sigma_{  \sHist{t} \sNew{t+1} }) \\
	\displaystyle  = \xi^{-2} \lVert \Sigma_{\sNew{t+1} \sHist{t}} \lVert_F^2 \\
	\displaystyle  = \mathcal{O}(\kappa^2 (t+1)(1+\sigma^2_y/\sigma^2_n)^2) 
	\end{array}
	\end{equation}
	where $\xi$ is the smallest eigenvalue in $\Lambda$, 
	the second equality is due to $\Sigma_{\sHist{t} \sHist{t}}^{-1} = \Xi \Lambda^{-1} \Xi^\top$, the third and seventh equalities are due to $\lVert \Phi\lVert^2_F = \mathrm{Tr}(\Phi\Phi^\top)$ for any matrix $\Phi$ (see Section~$10.4.3$ in~\cite{cookbook}), the fourth equality follows from the orthonormality of $\Xi$, the fifth equality is due to linearity of trace, the inequality is due to the positive semidefinite $(\xi^{-2}I - \Lambda^{-2})$ since $\xi^{-2}$ is the largest eigenvalue in $\Lambda^{-2}$, and
	the last equality follows from (a) $\lVert\Sigma_{\sNew{t+1} \sHist{t}} \lVert_F^2 = \mathcal{O}(\kappa^2 (t+1)(\sigma^2_y + \sigma^2_n)^2)$ since every prior covariance is not more than $\sigma^2_y + \sigma^2_n$ and the length of $\sHist{t}$ is $\mathcal{O}(\kappa (t+1))$ and (b) $\xi \ge \sigma_n^2$ since 
	$(\Sigma_{\sHist{t} \sHist{t}} - \sigma_n^2 I)$ is positive semidefinite and hence $\xi - \sigma_n^2$ is nonnegative.
\end{proof}
%

\ifCLASSOPTIONcaptionsoff
\newpage
\fi

\end{document}